\documentclass{article}
\usepackage[utf8]{inputenc}
\usepackage{amsthm}
\usepackage{amssymb}
\usepackage{amsmath,empheq, amsfonts}
\usepackage{url}
\usepackage{bm}
\usepackage{subcaption}
\usepackage{algorithmicx}
\usepackage[utf8]{inputenc} 
\usepackage[T1]{fontenc}    
\usepackage{hyperref}       
\usepackage{url}            
\usepackage{booktabs}       
\usepackage{amsfonts}       
\usepackage{microtype}      

\usepackage{graphics}  
\usepackage{bm}        
\usepackage{epstopdf}
\usepackage{placeins}
\usepackage{float}
\usepackage{amsthm}
\usepackage{dcolumn}   
\usepackage{bm}        
\usepackage{amssymb}   
\usepackage{amsmath}
\usepackage{comment}
\usepackage{tikz,ifthen}
\usepackage{color}
\usepackage{hyperref}
\usepackage{placeins}
\usepackage{tabularx}
\usepackage{float}
\usepackage[]{algorithm2e}
\usepackage{comment}
\usepackage{subcaption} 
\usepackage{algpseudocode}

\usepackage{amsopn}

\DeclareMathOperator{\Tr}{Tr}
\DeclareMathOperator{\diag}{diag}
\DeclareMathOperator{\Diag}{Diag}

\newtheorem{lemma}{Lemma}
\newtheorem{theorem}{Theorem}
\newtheorem{corollary}{Corollary}

\author{Micha\"el Fanuel\thanks{KU Leuven, Department of Electrical Engineering (ESAT),
STADIUS Center for Dynamical Systems, Signal Processing and Data Analytics,
Kasteelpark Arenberg 10, B-3001 Leuven, Belgium. email: michael.fanuel@kuleuven.be}
\and Joachim Schreurs\footnotemark[1]
\and Johan A.K. Suykens\footnotemark[1]}

\title{Diversity sampling is an implicit regularization for kernel methods}

\begin{document}

\maketitle

\begin{abstract}
Kernel methods have achieved very good performance on large scale regression and classification problems, by using the Nystr\"om method and preconditioning techniques. The Nystr\"om approximation -- based on a subset of landmarks --   gives a low rank approximation of the kernel matrix, and is known to provide a form of implicit regularization. We further elaborate on the impact of sampling \emph{diverse} landmarks for constructing the Nystr\"om approximation in supervised as well as unsupervised kernel methods. By using Determinantal Point Processes for sampling, we obtain additional theoretical results concerning the interplay between diversity and regularization. Empirically, we demonstrate the advantages of training kernel methods based on subsets made of diverse points. In particular, if the dataset has a dense bulk and a sparser tail, we show that Nystr\"om kernel regression with diverse landmarks increases the accuracy of the regression in sparser regions of the dataset, with respect to a uniform landmark sampling.  A greedy heuristic is also proposed to select diverse samples of significant size within large datasets when exact DPP sampling is not practically feasible.
\end{abstract}



\section{Introduction}
Kernel methods often rely on low rank  approximations to deal with large scale datasets. This paper addresses the special case of Nystr\"om approximation that is defined hereafter. Namely, let $k(x,y)>0$ be a continuous and strictly positive definite kernel. Given data $\{x_i\in \mathbb{R}^d \}_{i\in [n]}$, kernel methods rely on the entries of the Gram matrix $K = [k(x_i,x_j)]_{i,j}$.
To deal with large scale problems, one often samples a subset of landmarks $\mathcal{C}\subseteq [n]$ and defines a $n\times |\mathcal{C}|$ sampling matrix $C$ obtained by selecting the columns of the identity matrix indexed by $\mathcal{C}$. This is useful to select rectangular and squared submatrices as follows:  $K_{\mathcal{C}} = KC$ and $K_{\mathcal{C}\mathcal{C}} = C^\top K C$.
Then, the $n\times n$ kernel matrix $K$ is approximated by a low rank Nystr\"om approximation 
\begin{equation}L(K,\mathcal{C}) = K_{\mathcal{C}}K_{\mathcal{C}\mathcal{C}}^{-1} K_{\mathcal{C}}^\top,\label{eq:Nystr}
\end{equation}
which involves inverting $K_{\mathcal{C}\mathcal{C}}$.
However, this submatrix can be  ill-conditioned. In practice, this happens especially when $\mathcal{C}$ is sampled uniformly at random and for the Gaussian kernel $k(x,y)= \exp(-\|x-y\|_2^2/\sigma^2)$, that is used in this paper.
We argue here that a sampling of $\mathcal{C}$ which yields a good performance is closely related to the \emph{diversity} of the set of landmarks. In this work, the diversity of $\mathcal{C}$ is measured by the value of $\det(K_{\mathcal{C}\mathcal{C}})$. This is intuitively understood thanks to the connection between determinant and volume \cite{KuleszaT12}. We claim that selecting a diverse sample \emph{implicitly regularizes} the corresponding submatrix. This is illustrated on the \texttt{Housing}  dataset in Figure~\ref{fig:Nystrom} where the Nystr\"om approximation error and condition number of $K_{\mathcal{C}\mathcal{C}}$ are given for several $\mathcal{C}$ of identical cardinality, associated with different diversities. Additional technical details are given in Section~\ref{sec:experiments}. The empirical findings described in Section~\ref{sec:experiments} indicate that {Nystr\"om-based} kernel methods are improved if the landmarks are both diverse and yield an accurate kernel approximation. To illustrate this statement, we use a greedy swapping algorithm, namely Algorithm~\ref{AlgGreedySwap} (blue line in Figure~\ref{fig:Nystrom}), which allows for obtaining a sample of landmarks with a given diversity. 
It is worth mentioning that different other methods exist to sample diverse landmarks, such as volume sampling~\cite{deshpande2006matrix}, greedy methods~\cite{CIVRIL},  Determinantal Point Processes~(DPP)~\cite{KuleszaT12}, etc.
The example of Figure~\ref{fig:Nystrom} illustrates the connection between diversity, regularization and Nystr\"om approximation error. Namely, we sample repeatedly subsets of the same size by using Uniform, Ridge Leverage Score (RLS) and DPP sampling (defined hereafter), which yield samples with an increasing diversity. Figure~\ref{fig:Nystrom} shows that the corresponding kernel submatrices have an increasing least eigenvalue, a decreasing Nystr\"om approximation error and a decreasing condition number. This highlights the implicit regularization due to diversity.
	\begin{figure}[h]
		\centering
				\begin{subfigure}[b]{0.32\textwidth}
			\includegraphics[width=\textwidth, height= 0.9\textwidth]{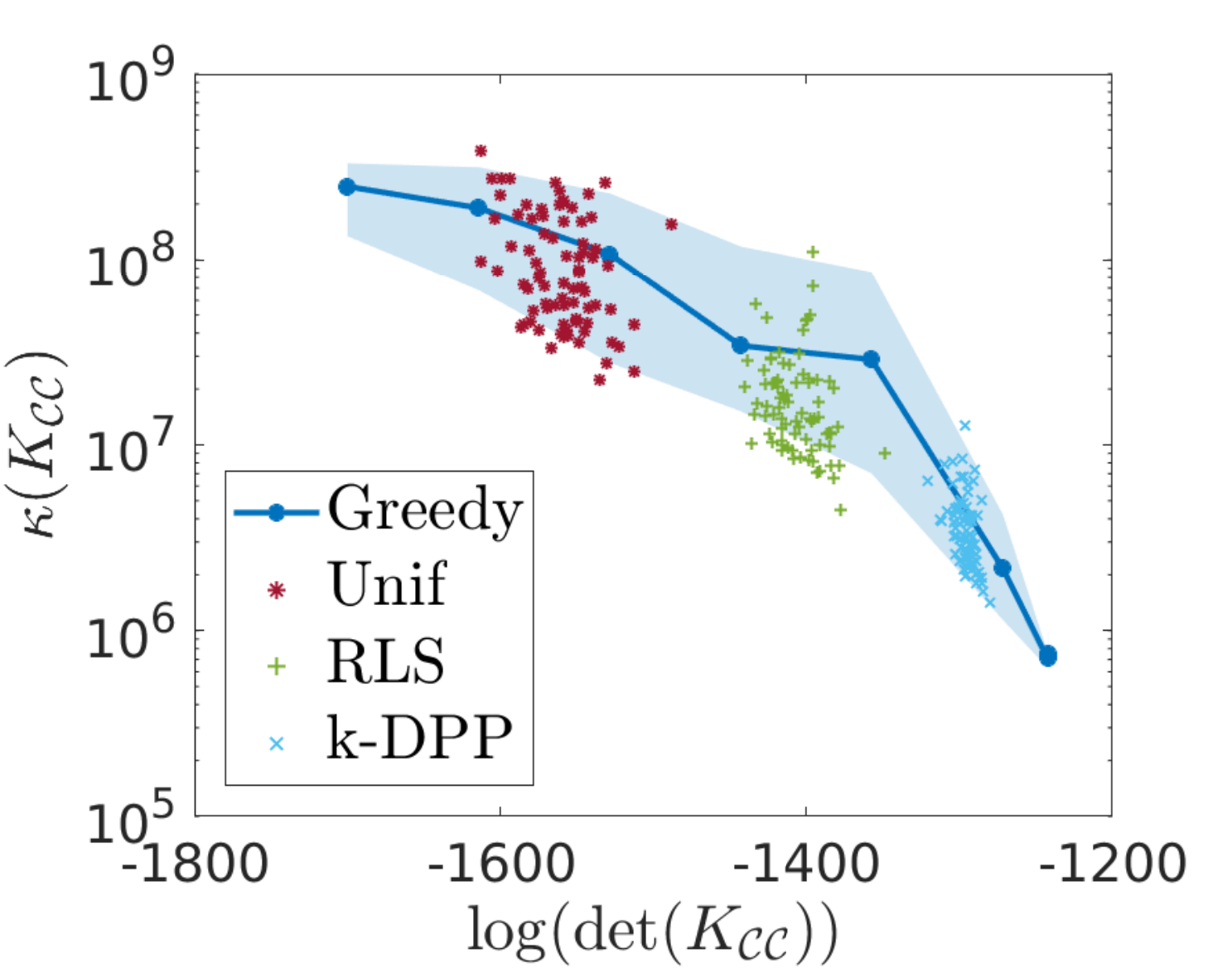}
			\caption{$\kappa(K_{\mathcal{C}\mathcal{C}})$}
		\end{subfigure}
		\begin{subfigure}[b]{0.32\textwidth}
			\includegraphics[width=\textwidth, height= 0.9\textwidth]{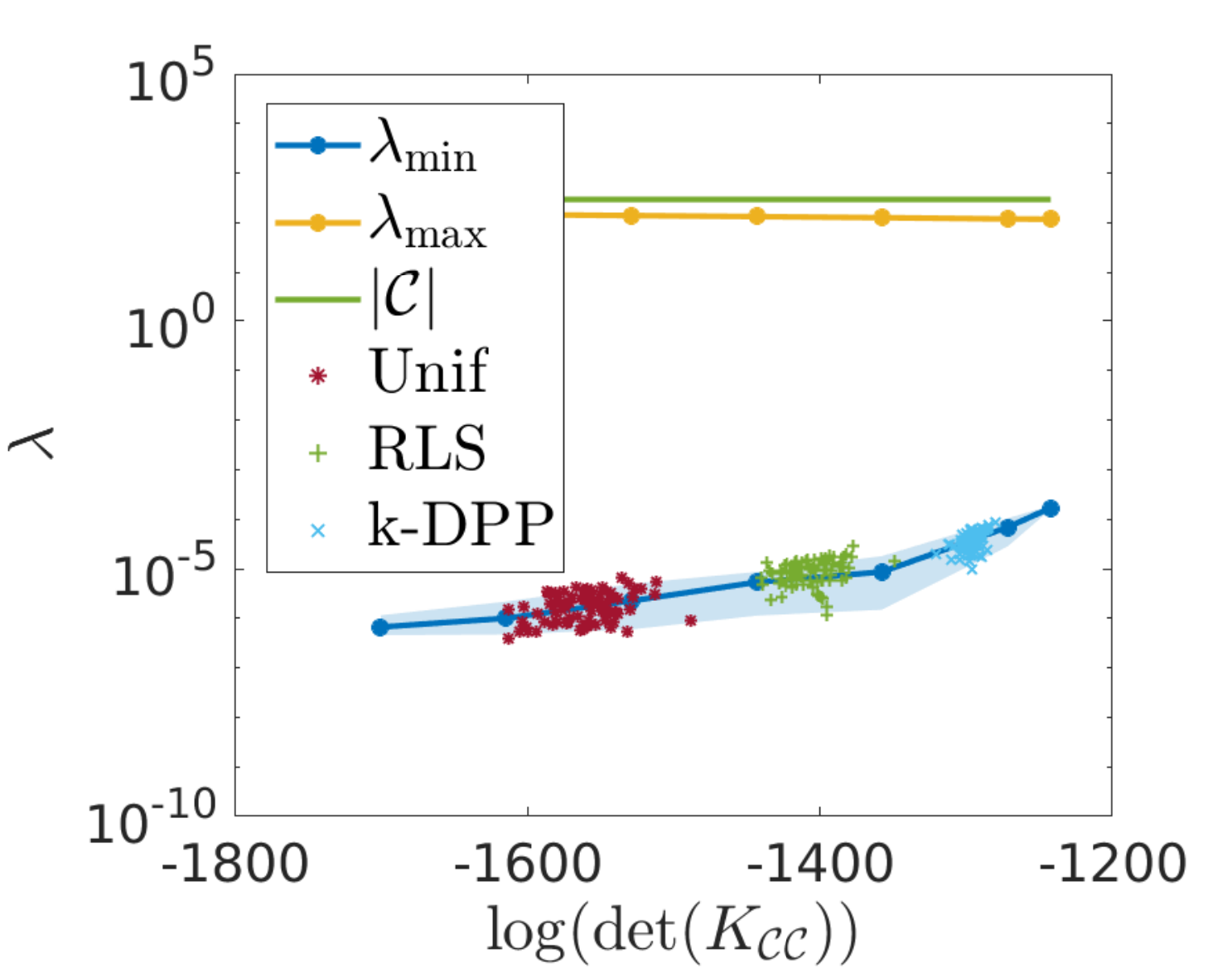}
			\caption{$\lambda_{\min}(K_{\mathcal{C}\mathcal{C}})$}
		\end{subfigure}
		\begin{subfigure}[b]{0.32\textwidth}
			\includegraphics[width=1\textwidth, height= 0.9\textwidth]{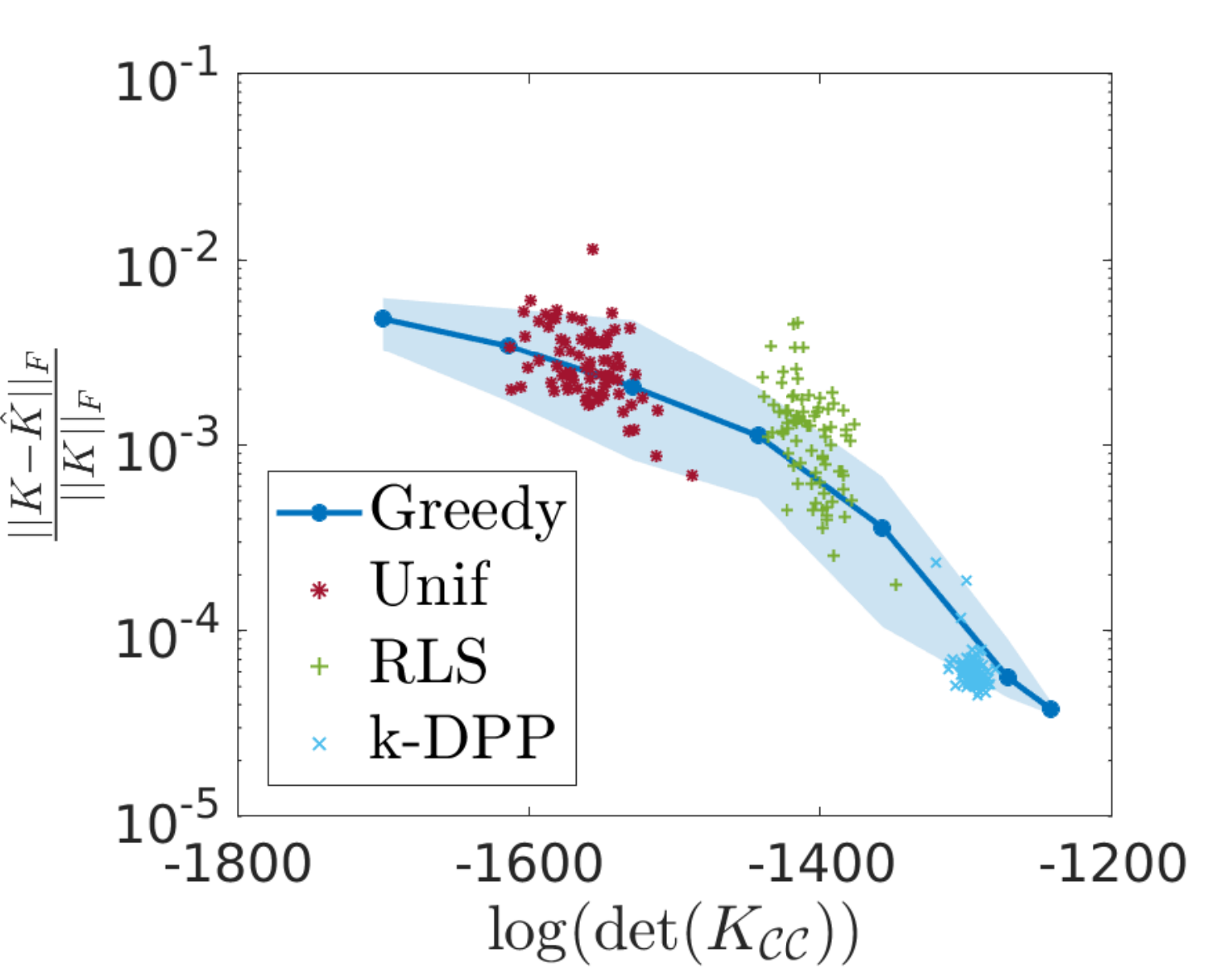}
			\caption{Approximation error}
		\end{subfigure}
	\caption{Nystr\"om approximation for the \texttt{Housing} dataset with parameters given in Table~\ref{Table:data}. The condition number, smallest and largest eigenvalues of $K_{\mathcal{C}\mathcal{C}}$, relative Frobenius norm of the approximation error of $\hat{K} = L(K,\mathcal{C})$ are plotted  versus the sample diversity. The larger $\mathrm{det}(K_{\mathcal{C}\mathcal{C}})$, the more diverse the subset. Error bars are standard deviations over 10 simulations.}\label{fig:Nystrom}
\end{figure}
Among diverse sampling methods, DPPs provide a natural probabilistic framework for diversity sampling. Their elegant definition allows to derive results formalizing our empirical observations about the interplay between regularization and diversity. Let us briefly define them in the simplest setting, while a more complete overview can be found in~\cite{KuleszaT12}.
\subsection{DPP sampling}

Let $L$ be a $n\times n$ positive definite symmetric matrix, called L-ensemble. Then, the probability that a subset $\mathcal{C}\subseteq [n]$ is sampled is defined as follows
\[
\Pr(Y = \mathcal{C}) = \det(L_{\mathcal{C}\mathcal{C}})/\det(\mathbb{I}+L).
\]
In this paper, we define $L = K/\alpha$ with $\alpha>0$ and denote the associated process  $DPP_L(K/\alpha)$.
Classically, an alternative viewpoint deals with the inclusion probabilities as given by
$
\Pr( \mathcal{C}\subseteq Y) = \det(P_{\mathcal{C}\mathcal{C}}),
$
where
\begin{equation}
P = K(K+\alpha\mathbb{I})^{-1},\label{eq:P}
\end{equation} is the marginal kernel  associated $L$-ensemble $L = K/\alpha$.
The diagonal of this soft projector matrix~\eqref{eq:P} yields the so-called Ridge Leverage Scores (RLS) of the data points:
\[\bm{\ell}_i = P_{ii} \text{ for } i\in [n],\] which have been used in order to sample landmarks points in various works~\cite{MuscoMusco,ElAlaouiMahoney,Bach2013} in the context of Nystr\"om approximations. RLS can be considered as a measure of importance or `outlierness' of a data point. The sum of the RLS yields the effective dimension $d_{\rm eff}(K/\alpha)$ which is also the expected size $d_{\rm eff}(K/\alpha) = \mathbb{E}_{\mathcal{C}}[|\mathcal{C}|]$ if $\mathcal{C}\sim DPP_L(K/\alpha)$.
Since the subset size $|\mathcal{C}|$ in itself also a random variable, it is also customary to use $k$-DPPs which are DPPs conditioned on a given subset size $k$. (see also \cite{Fastdpp}).

The following two sections motivate the impact of the regularity of $K_{\mathcal{C}\mathcal{C}}$ in two applications. Firstly, a better kernel approximation yields an improvement of the performance of unsupervised kernel methods such as Kernel Principal Component Analysis (KPCA)~\cite{scholkopf1998nonlinear, sterge2019gain} and Kernel $k$-means~\cite{wang2019scalable}. Secondly, the conditioning of $K_{\mathcal{C}\mathcal{C}}$ is also important for large-scale supervised learning methods -- based on Nystr\"om approximation -- as the convergence and accuracy of iterative solvers depends often of the condition number. Finally, sampling with a diverse method spread the points more over the full dataset. This is especially important for accuracy in less populated or `outlying' regions in the dataset. We now give a short overview of how the Nystr\"om approximation is used to speed up kernel PCA and  kernel ridge regression. 
\subsection{Kernel PCA}
The Nystr\"om method is used to develop a more computationally efficient approximate kernel PCA algorithm~\cite{sterge2019gain}. Let $\mathcal{H}$ be the Reproducing Kernel Hilbert Space associated to $k$ and assume that the data is sampled from a distribution $\mathbb{\rho}$ such that $\mathbb{E}_{X\sim\rho}[f(X)] = 0$ for all $f\in \mathcal{H}$. We recall that Kernel PCA is a principal component  analysis in a RKHS, i.e. it consists in finding the directions of maximum variance. Indeed, let $k_{x_i}(\cdot) = k(x_i,\cdot )$ and let the empirical covariance operator $C = \frac{1}{n}\sum_{i=1}^{n} k_{x_i}\otimes k_{x_i}$. Also, we define the subspace $\mathcal{H}_{\mathcal{C}} =  {\rm span}\{ k_{x_i}\text{ s.t. }  i\in \mathcal{C} \}$.
Then, the optimization problem 
\[
\sup_{f\in \mathcal{H}}\langle f, C f\rangle_{\mathcal{H}} \text{ s.t. } \|f\|_{\mathcal{H}} = 1 \text{ and } f\in \mathcal{H}_{\mathcal{C}},
\]
corresponds to a Nystr\"om approximation of KPCA if $\mathcal{C} \subset [n]$. The empirical estimation of KPCA involves the eigendecomposition of the matrix \[M =K_{\mathcal{C} \mathcal{C}}^{-1 / 2} K_{\mathcal{C}}^\top K_{\mathcal{C}} K_{\mathcal{C}\mathcal{C}}^{-1 / 2},\] sharing its non-zero eigenvalues with~\eqref{eq:Nystr}. Let $(\hat{\lambda}_{\ell, \mathcal{C}}, \bm{u}_{\ell})_{\ell=1}^{|\mathcal{C}|}$ be eigenpairs of $\frac{1}{n} M$ sorted in descending order. KPCA aims to construct the orthogonal projector $P_{\mathcal{C}} = \sum_{\ell=1}^{c} \bm{u}_{\ell}\bm{u}_{\ell}^{\top}$ on the subspace corresponding to the $c$ largest eigenvalues, with $c\leq |\mathcal{C}|$. Clearly, the conditioning of $K_{\mathcal{C}\mathcal{C}}$ is important for this task in view of the definition of $M$.
The reconstruction error for $c$ components assesses the quality of the approximation and is given by:
\begin{equation}
\hat{R}\left(P_{\mathcal{C}}\right)=\frac{1}{n} \operatorname{tr}(K)-\frac{1}{n} \sum_{\ell=1}^{c} \bm{u}_{\ell}^{\top} M \bm{u}_{\ell}=\sum_{\ell=1}^{n} \hat{\lambda}_{\ell}-\sum_{\ell=1}^{c} \hat{\lambda}_{\ell, \mathcal{C}},
\end{equation}
where $(\hat{\lambda}_{\ell}, \bm{v}_{\ell})_{\ell=1}^{n}$  are eigenpairs of $\frac{1}{n} K$. A small reconstruction error is then achieved thanks to an accurate Nystr\"om approximation as detailed in Section~\ref{sec:UnsupK}. 

\subsection{Regression}
In approximate Kernel Ridge Regression (KRR), the regressor is obtained from given input-output pairs $\{(x_i,y_i)\in \mathbb{R}^d \times \mathbb{R}\}_{i\in [n]}$ by solving
\begin{equation*}
\label{eq:ridgeRegression}
f^\star=  \arg\min_{f\in \mathcal{H}_{\mathcal{C}}}\frac{1}{n}\sum_{i=1}^{n}(y_i-f(x_i))^2 + \gamma \|f\|_{\mathcal{H}}^2, \text{ with } \gamma>0,
\end{equation*}
where $\mathcal{H}_{\mathcal{C}} = {\rm span} \{ k(x_i, \cdot)|i\in \mathcal{C}\}$ and $\mathcal{C}\subseteq [n]$. The regressor is $f^\star (\cdot) = \sum_{i\in \mathcal{C}}\alpha^\star_i k(x_i,\cdot)$  with  
\begin{equation}
 \bm{\alpha}^\star = (K_{\mathcal{C}}^\top K_{\mathcal{C}} + n \gamma K_{\mathcal{C}\mathcal{C}})^{-1}K_{\mathcal{C}}^\top  \bm{y}.\label{eq:SmallLS}
\end{equation}
The condition number of~\eqref{eq:SmallLS} crucially depends on the magnitude of the least eigenvalue of $K_{\mathcal{C}\mathcal{C}}$ which plays the role of regularization term. Notice that the full KRR is simply obtained by replacing $\mathcal{H}_{\mathcal{C}}$ by $\mathcal{H}$ in~\eqref{eq:ridgeRegression}.
\paragraph{Stability of the expected risk} The first consequence of an accurate Nystr\"om approximation is that the expected risk of approximate KRR is upper bounded by a controllable constant as it is explained in Section~\ref{sec:KRR}. This means that the training problem can not be dramatically affected by the Nystr\"om approximation.
\paragraph{Uniform test error} A second consequence is directly related to the diversity of the landmarks and is illustrated in Figure~\ref{fig:ToyExampleRegression}. Namely, Figure \ref{fig:ToyExampleRegression} shows the training dataset which consists  of 1000 sampled points generated as follows: $\bm{x}_i \sim \mathcal{N}(\bm{0},\mathbb{I})$ and $y_i= \bm{e}_1^\top \bm{x}_i + b + \epsilon_i$ with iid  $\epsilon_i\sim \mathcal{N}(0,\sigma^2)$, with $\sigma = 0.1$ and $b= 20$. The corresponding $y_i$ values are visualized by the color of the points. Landmarks are then sampled by using uniform and DPP sampling. Uniform sampling oversamples the dense parts, while a diverse sampling algorithm samples spreads the points over the full dataset. A kernel ridge regressor with Gaussian kernel is trained by using \eqref{eq:ridgeRegression}, where the optimal regularization parameter $\gamma>0$ is determined using cross-validation. Figure \ref{fig:binwisePrediction} shows the Mean Absolute Percentage Error (MAPE) in function of the ridge leverage scores of the test set, where each dot corresponds to the MAPE in the corresponding bin of the histogram. This stratification of the dataset allows to visualize how the regressor performs in dense (small RLS) and sparser (large RLS) groups of the dataset. 
 \begin{figure}[H]
 	\centering
 	\begin{subfigure}[t]{0.32\textwidth}
 		\includegraphics[width=\textwidth, height= 0.85\textwidth]{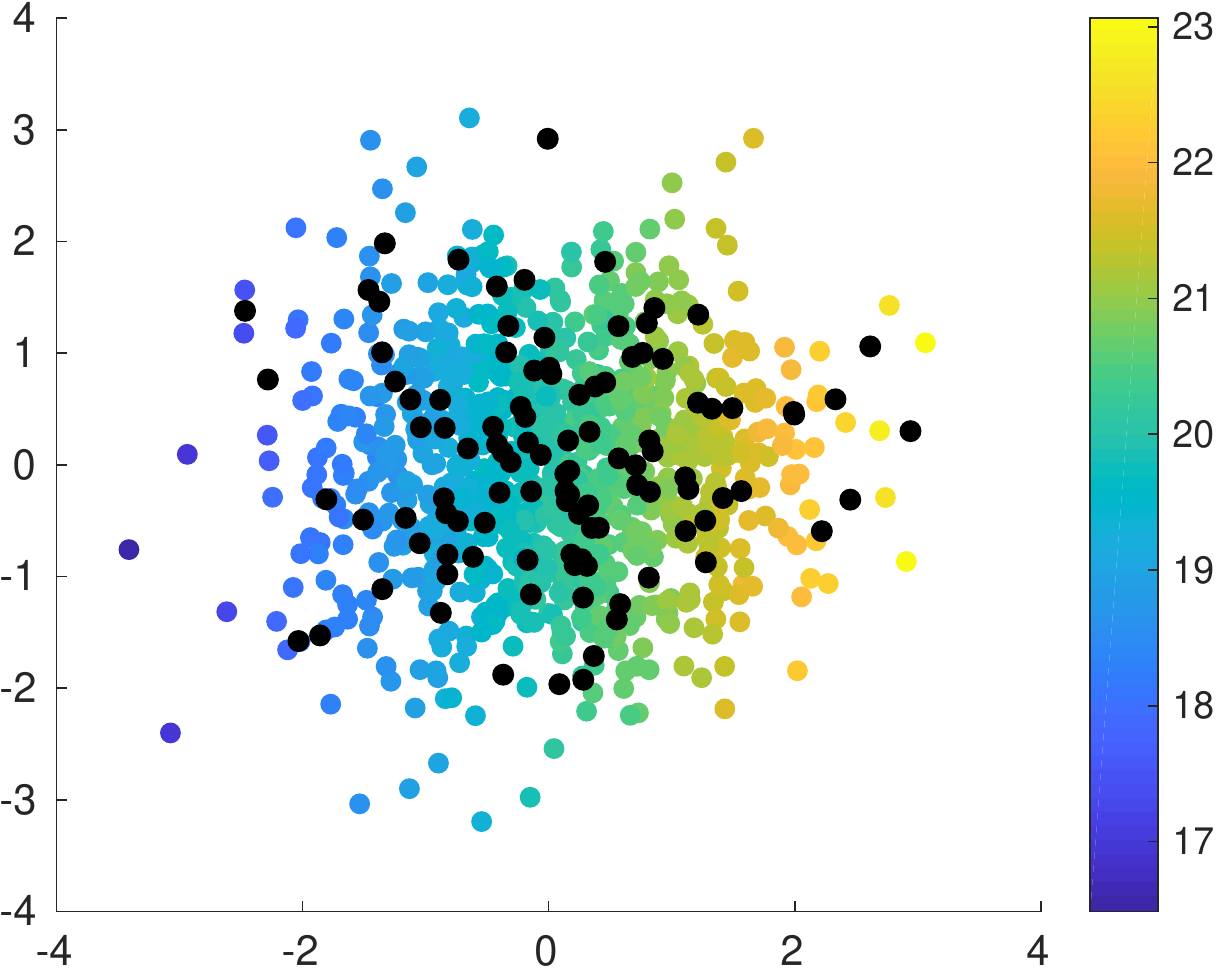}
 		\caption{Uniform sampling.}
 	\end{subfigure}
 	\begin{subfigure}[t]{0.32\textwidth}
 		\includegraphics[width=\textwidth, height= 0.85\textwidth]{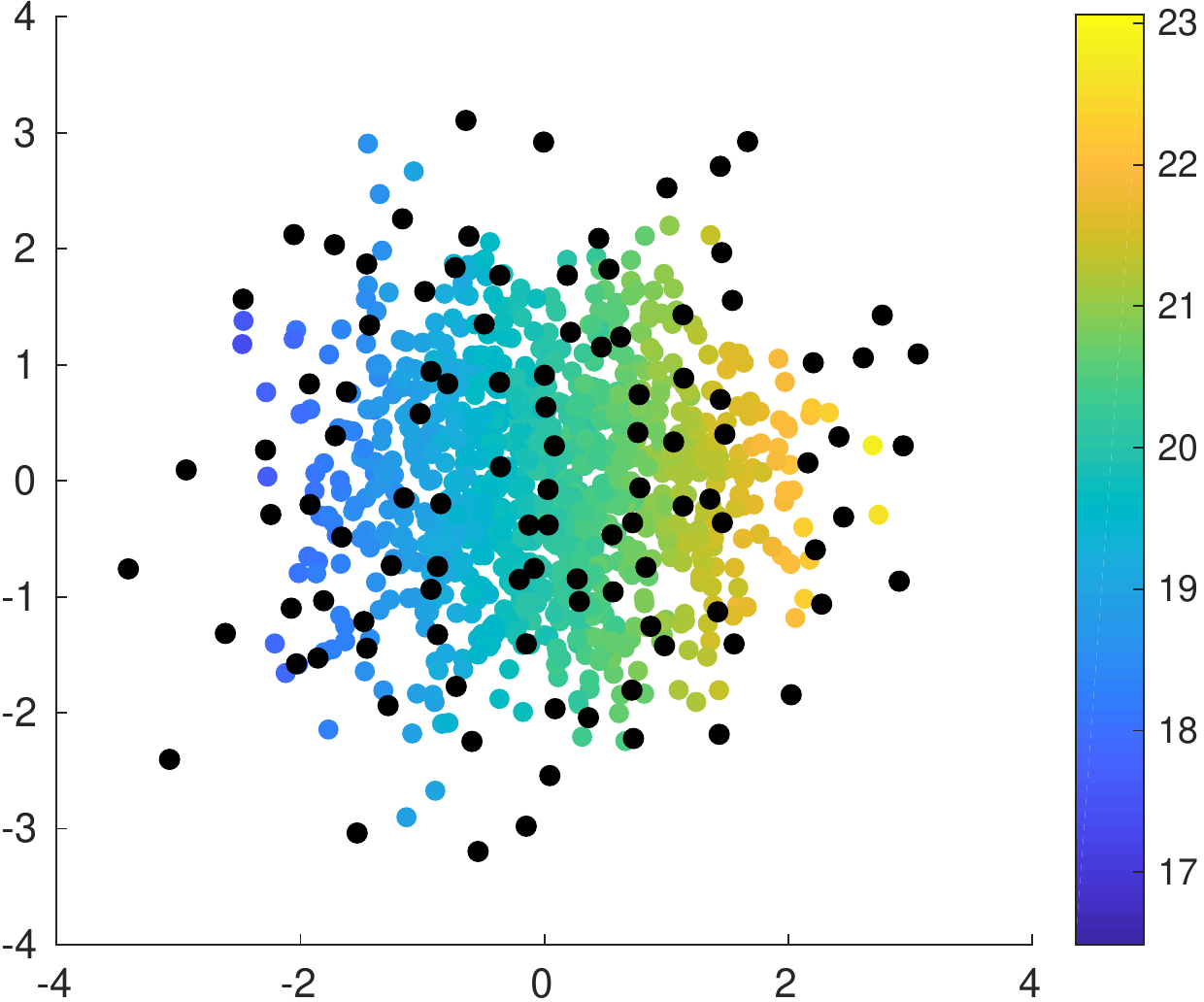}
 		\caption{k-DPP sampling.}
 	\end{subfigure}
 	\begin{subfigure}[t]{0.32\textwidth}
 		\includegraphics[width=\textwidth, height= 0.9\textwidth]{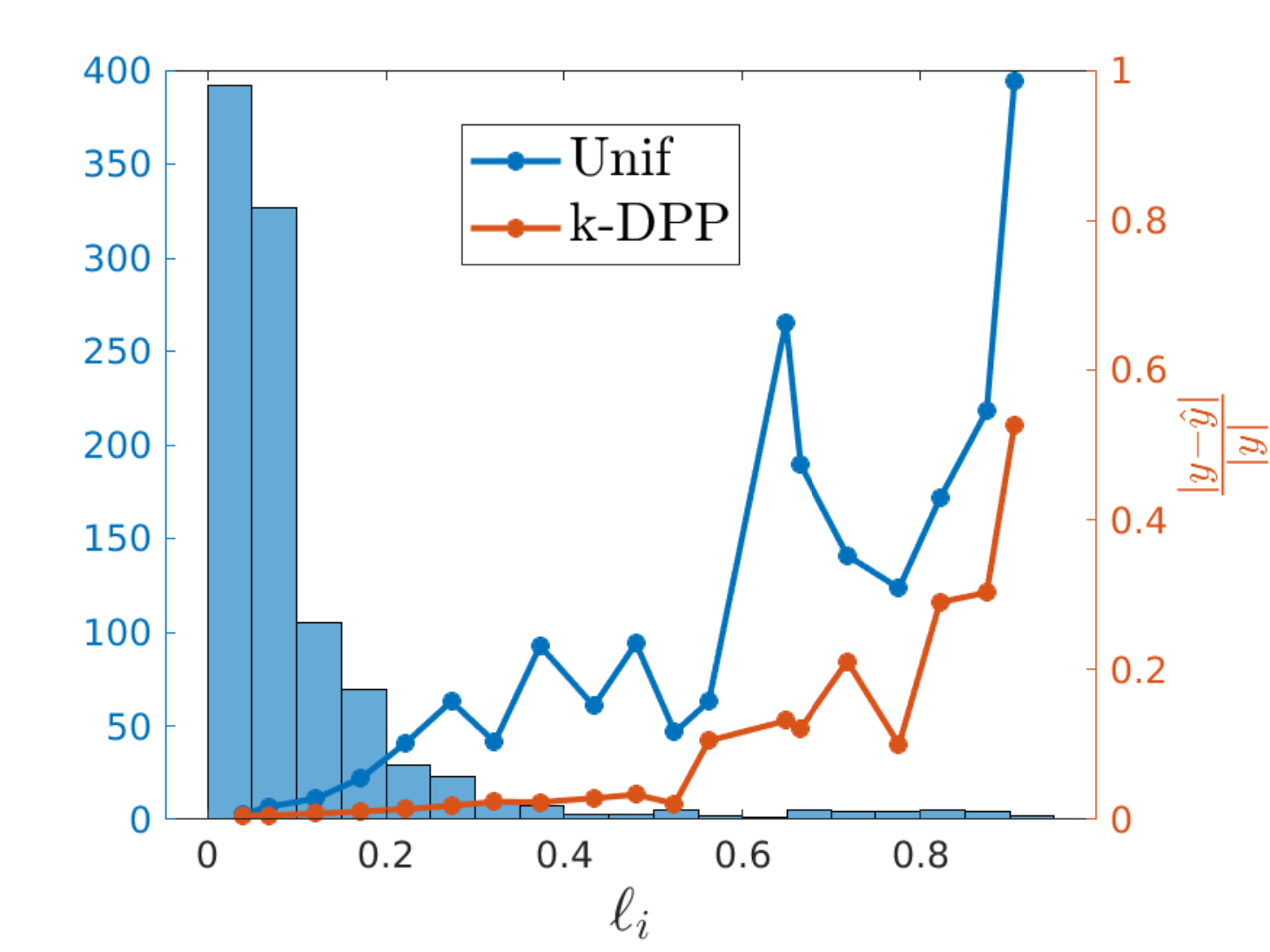}
 		\caption{RLS and binned MAPE.}
 		\label{fig:binwisePrediction}
 	\end{subfigure}	
 	\caption{Toy example of regression. In \ref{fig:binwisePrediction}, a histogram of RLS distribution and the MAPE test error in each bin are displayed.\label{fig:ToyExampleRegression}}
 \end{figure}
 Diverse sampling has a consistently better performance than uniform sampling, where the difference is more apparent for high leverage scores. This is especially important when sampling from datasets with long tail RLS distributions. Hence, in the case of diverse sampling, we emphasize that the percentage error is more uniform on the support of the dataset while the regressor makes a smaller error on points with larger leverage score compared to regressors obtained with uniform sampling, while the total MAPE shows only a minor difference.  Additional illustrations of this effect are given in Section~\ref{sec:experiments}. 
Naturally, diverse sampling is less important if there is no long tail of in the RLS distribution.\\

We now want to emphasize why diverse sampling is important, especially in stratified datasets. Recently, there has been a lot of interest in not only predicting well in the majority of the data, but also for specific outlying points~\cite{valverde2014100,oakden2019hidden,chen2019slice}. These outlying points can e.g. correspond to serious diseases in a medical dataset, being less common than mild diseases. Incorrectly classifying these outliers could lead to significant harm to patients. The performance in these subpopulations is often overlooked. This is because aggregate performance measures such as MSE or sensitivity can be dominated by larger subsets, obscuring the fact that there may be an unidentified subset of cases where the performance is poor. These stratifications often occur in datasets with a long tail, i.e. the data distribution of each class is viewed as a mixture of distinct subpopulations~\cite{feldman2019does}. For example, images of dogs include different species photographed from different perspectives and under different conditions (such as close-ups, in the woods and during the rain). A long-tailed mixture distribution will have some subpopulations from which just a few or only a single one example was observed. When using sampling algorithms, it is therefore necessary to select points out of each subpopulation to achieve close-to-optimal generalization error. One expects that, before seeing the dataset, the learning algorithm does not know the frequencies of subpopulations and may not be able to predict accurately on a subpopulation without observing any examples from it. By making sure the selected subset is \emph{diverse} enough, there is a higher chance of every subpolation being included in the sample. In~\cite{feldman2019does}, it is argued that datasets with long tails are a possible reason why interpolating models or models that achieve zero error rate on the training data, can still generalize~\cite{belkin2018understand,liang2018just,belkin2018overfitting}. These hidden stratifications motivate the search for better loss functions. We therefore propose an unsupervised approach, where the loss function is determined on two parts of the data: the bulk and tail of the data. The bulk and tail of the data correspond to points with low and high outlyingness respectively, where the outlyingness is measured by the ridge leverage scores. By splitting the loss function into two parts, one can identify if the model is not only focusing on the majority data but also performing well in `outlying' subpopulations.
 
The rest of the paper includes theoretical results in Section~\ref{sec:main}, and numerical experiments in Section~\ref{sec:experiments}. The proofs and dataset description are given in appendix. Another application, namely kernel k-means, and additional numerical experiments can be found in supplementary material.
\section{Main results: Implicit regularization}
\label{sec:main}
On expectation, the largest and smallest eigenvalues of several matrices obtained by DPP sampling can be bounded, showing indeed that the spectrum of those submatrices are likely to be under control. This is formalized in Theorem~\ref{Thm1}, where we denoted by $\circ$ the entry-wise product between matrices.
\begin{theorem}[Implicit regularization]\label{Thm1}
Let $\mathcal{C}\sim DPP_L(K/\alpha)$  and let $C$ be a sampling matrix associated to the set $\mathcal{C}$. Then, we have
\[
\mathbb{E}_{C}\left[C K_{\mathcal{C}\mathcal{C}}^{-1}  C^\top\right]= (K +\alpha\mathbb{I} )^{-1},
 \text{ and }
 \mathbb{E}_{C}\left[C K_{\mathcal{C}\mathcal{C}} C^\top\right]= P_{(2)}\circ K ,
\]
with $P_{(2)} = \Diag(\bm{\ell}) + \bm{\ell}\bm{\ell}^\top-P\circ P \succeq 0$.
\end{theorem}
Notice that $P_{(2)}$ is positive semi-definite. 
Furthermore, if $\diag(K) = \bm{1}$ as in the case of the Gaussian kernel,  the largest eigenvalue
\begin{equation}
\lambda_{\max}(\mathbb{E}_{C}\left[C K_{\mathcal{C}\mathcal{C}} C^\top\right])\leq \lambda_{\max}(P_{(2)}) \leq  \|\bm{\ell}\|_{\infty} + \|\bm{\ell}\|^2_{2},\label{eq:upperLambda}
\end{equation} is bounded in terms of the leverage scores.
This is a direct consequence of a Corollary~2 in \cite{BAPAT1985107}, namely the spectrum of $A\circ B$ is majorized by the spectrum of $A$ if $A$ and $B$ are symmetric and positive semidefinite with $\diag(B) = \bm{1}$.
It is noticeable that the largest eigenvalue of the expected kernel submatrix is under control for DPP sampling. Indeed, other sampling schemes are not known to yield similar guarantees. Again, if $\diag(K) = \bm{1}$, the trace of $\mathbb{E}_{\mathcal{C}}\left[C^\top K_{\mathcal{C}\mathcal{C}} C\right])$  is the expected size of the sample,  $d_{\rm eff}(K/\alpha)$, which gives then another an upper bound for $\lambda_{\max}(\mathbb{E}_{C}\left[C K_{\mathcal{C}\mathcal{C}}C^\top\right])$. We observe empirically that the latter yields a much larger upper bound compared to~\eqref{eq:upperLambda}.
Importantly, the scale parameter  $\alpha>0$ both controls the size of the sample and regularizes the subkernels matrix in the following sense:
\[
\lambda_{\max}(\mathbb{E}_{\mathcal{C}}\left[C K^{-1}_{\mathcal{C}\mathcal{C}} C^\top\right])\leq \alpha^{-1}.
\]
These results on expectation can be instructive since we expect concentration about the mean. Indeed, Permantle and Peres showed that  strong Rayleigh measures -- generalizing DPPs -- obey Gauss-Poisson concentration bounds~\cite{pemantle2014concentration}. Corollary~\ref{corol2} is then a direct consequence of that concentration result. For convenience, we write $\bm{w}_{\mathcal{C}} = C^\top \bm{w}$, where $C$ is the sampling matrix associated to $\mathcal{C}\subseteq [n]$.
\begin{corollary}[Regularization with high probability]\label{corol2}
Let $\mathcal{C}\sim DPP_L(K/\alpha)$ and $\bm{w}\in \mathbb{R}^{n}$ such that $\|\bm{w}\|_2 = 1$. Then, we have 
\[
|\bm{w}^{\top}_{\mathcal{C}} K_{\mathcal{C}\mathcal{C}}^{-1}  \bm{w}_{\mathcal{C}}-\bm{w}^{\top} (K +\alpha\mathbb{I} )^{-1}  \bm{w}|\leq \frac{\sqrt{48n\log(\frac{5}{\delta})}}{\lambda_{\min}(K)}.
\]
with probability at least $1-\delta$.
\end{corollary}
A drawback of Corollary~\ref{corol2} is that the bound hereabove depends of the inverse of $\lambda_{\min}(K)$ which may be a large number. The result may be improved by finding a better upper found on the Lipschitz constant of the function $f(\mathcal{C}) = \bm{w}^{\top}_{\mathcal{C}} K_{\mathcal{C}\mathcal{C}}^{-1}  \bm{w}_{\mathcal{C}}$. We refer to the proof of Corollary~\ref{corol2} for more details.
\subsection{Nystr\"om approximation}
The subset obtained thanks to a DPP sampling is not only yielding a regular kernel submatrix, Corollary~\ref{corol3} states that it produces a good Nystr\"om approximation.
It also gives a natural connection between the projector
\[\mathbb{P}_{{\rm range}(K^{1/2}C)}  = K^{1/2}C(K_{\mathcal{C}\mathcal{C}})^{-1}C^\top K^{1/2},
\]
and the marginal kernel~\eqref{eq:P}.
\begin{corollary}[Expected Nystr\"om approximation]\label{corol3}
Let $\mathcal{C}\sim DPP_L(K/\alpha)$. Then, we have an expression for the Nystr\"om error on expectation
\[
\mathbb{E}_{\mathcal{C}}\left[K-L(K,\mathcal{C}) \right] = \alpha P,
\text{ and }\mathbb{E}_{\mathcal{C}}\left[\mathbb{P}_{{\rm range}( K^{1/2}C)}\right]=P.\]
\end{corollary}
As a straightforward consequence of Corollary~\ref{corol3}, the nuclear norm of the approximation error is simply
$\mathbb{E}_{\mathcal{C}}[\Tr(K-L(K,\mathcal{C}))] = \alpha d_{\rm eff}(K/\alpha)$, since $K\succeq L(K,\mathcal{C})$.
To the best of our knowledge, only a weaker result about the accuracy of the Nystr\"om approximation~\cite{Fastdpp} related to $k$-DPPs exists in the literature.

\subsection{Unsupervised kernel methods \label{sec:UnsupK}}
Theorem~\ref{thm:projection-cost} states that the `distance' of $K$ to a $k$-dimensional subspace is well approximated by the `distance' of $L(K,\mathcal{C})$ to the same subspace, on expectation. An analogous results of RLS sampling can be found in~\cite{MuscoMusco}.

\begin{theorem}[Expected projection-cost preservation]\label{thm:projection-cost}
Let $\mathcal{C}\sim DPP_L(K/\alpha)$ and $X $ an orthogonal projector on a $k$-dimensional subspace. Denote $L =L(K,\mathcal{C})$. Then we have
\[
\Tr(K- X  KX  )\leq \mathbb{E}_{\mathcal{C}}[\Tr(L- X  LX  )] + c(\alpha)\leq \Tr(K- X  KX  ) + \min\{\alpha k,c(\alpha)\},
\]
where $c(\alpha) = \alpha d_{\rm eff}(K/\alpha)$.
\end{theorem}
A direct application of the above theorem is KPCA. Namely, the projector  $X^\star$ onto the leading $k$ components is obtained by
\begin{equation*}
\min_{X \in \Pi_k}\Tr(K - X K X),
\end{equation*}
where $\Pi_k$ is the  set of  $n\times n$ projectors of rank $k$. Then, the result Theorem~\ref{thm:projection-cost} is a stability result relating the objective functions of KPCA with and without Nystr\"om approximation.
Empirical experiments can be found in supplementary material.

\subsection{Kernel Ridge Regression \label{sec:KRR}}
A simple consequence of Corollary~\ref{corol3} is that the expected risk of KRR approximated by Nystr\"om method with DPP sampling cannot be arbitrary larger than the risk corresponding to the full KRR.
Namely, let the outputs be $y_i = z_i + \epsilon_i$ where $\epsilon_i$ are iid $\mathcal{N}(0,\sigma^2)$ and let the solution of KRR be  $\hat{\bm{z}}_K = K(K+n\gamma\mathbb{I})^{-1}\bm{y}$. The expected risk is then defined as $\mathcal{R}(\hat{\bm{z}}_{K})=\mathbb{E}_{\epsilon} \|\hat{\bm{z}}_K - \bm{z}\|_2^2.$
Then, we can give a bound on the risk of KRR associated to the Nystr\"om approximation.
\begin{theorem}[Expected risk bound]\label{thm:RiskKRR}
Let $\mathcal{C}\sim DPP_L(K/\alpha)$, then we have
\[
\mathbb{E}_{\mathcal{C}}\left[\sqrt{\frac{\mathcal{R}(\hat{\bm{z}}_{L(K,\mathcal{C})})}{\mathcal{R}(\hat{\bm{z}}_{K})}}\right]\leq 1 + \frac{\alpha}{n\gamma} d_{\rm eff}(K/\alpha).
\]
\end{theorem}
The upper bound in Theorem~\ref{thm:RiskKRR} tends to $1$ as $\alpha\to 0$ since $d_{\rm eff}(K/\alpha)\leq n$. This consistently shows that the larger is the number of landmarks, the closest is the risk of approximate KRR from the full KRR.
Notice that the increase in the risk is also mitigated by the regularization parameter.
\paragraph{Preconditioners}
Rudi et al.~\cite{Rudi:2015} propose a preconditioning of the linear system~\eqref{eq:SmallLS} of the form 
\begin{equation}
B^{\top}(K_{\mathcal{C}}^\top K_{\mathcal{C}} + n \gamma K_{\mathcal{C}\mathcal{C}})B \big(B^{-1}\bm{\alpha}\big) = B^\top K_{\mathcal{C}}^\top\bm{y}\label{eq:precondionerUniform}
\end{equation} where $B$ is obtained by solving 
$
B B^\top = \big(K_{\mathcal{C}\mathcal{C}}D_{\mathcal{C}\mathcal{C}}K_{\mathcal{C}\mathcal{C}}  + n \gamma K_{\mathcal{C}\mathcal{C}} \big)^{-1}
$, thanks to a Cholesky decomposition, where $D_{\mathcal{C}\mathcal{C}}$ is an appropriate diagonal matrix.
In the case of the uniform sampling of $\mathcal{C}$, the authors of~\cite{Rudi:2015} propose $D_{\mathcal{C}\mathcal{C}} = (n/|\mathcal{C}|) \mathbb{I}_{\mathcal{C}\mathcal{C}}$. For RLS sampling, they argue for $D_{\mathcal{C}\mathcal{C}} =\Diag(\bm{\ell}_{\mathcal{C}})^{-1}$, where $\bm{\ell}$ contains the so-called ridge leverage scores. We emphasize that the computation of $B$ indeed crucially depends on the magnitude of the least eigenvalue of $K_{\mathcal{C}\mathcal{C}}$. It is then interesting to sample diverse landmarks so that $K_{\mathcal{C}\mathcal{C}}$ is likely to be regular. 
A natural motivation for choosing $D_{\mathcal{C}\mathcal{C}} =\Diag(\bm{\ell}_{\mathcal{C}})^{-1}$ from the DPP viewpoint is given in Corollary~\ref{Corol:K2}.
This result naturally follows from Lemma~\ref{Lem:Estimator}, which can also be found in the context of Monte-Carlo integration~\cite{bardenet} with projective DPPs.
\begin{lemma}\label{Lem:Estimator}
Let $\mathcal{C}\sim DPP_L(K/\alpha)$ and  $\bm{v}$ and $\bm{w}\in \mathbb{R}^{n}$ 
Then, we have the identities 
$
\mathbb{E}_{\mathcal{C}}\left[\bm{v}^{\top}_{\mathcal{C}} \bm{w}_{\mathcal{C}}\right] = \bm{v}^\top \Diag(\bm{\ell})\bm{w},
$
and 
$
\mathbb{V}_{\mathcal{C}}\left[\bm{v}^{\top}_{\mathcal{C}} \bm{w}_{\mathcal{C}}\right] = (\bm{v}\circ \bm{w})^\top\Big( \Diag(\bm{\ell})-P\circ P\Big)(\bm{v}\circ \bm{w}).
$
\end{lemma}
Corollary~\ref{Corol:K2} then motivates the approximation of  $K_{\mathcal{C}}^\top K_{\mathcal{C}}$ in~\eqref{eq:SmallLS} by $K_{\mathcal{C}\mathcal{C}}D_{\mathcal{C}\mathcal{C}}K_{\mathcal{C}\mathcal{C}}$.
\begin{corollary}\label{Corol:K2}
Let $\mathcal{C}\sim DPP_L(K/\alpha)$ and $K_{\mathcal{C}} = KC$. Then, the following identity holds:
$
\mathbb{E}_{\mathcal{C}}\left[K_{\mathcal{C}} \Diag(\bm{\ell}_{\mathcal{C}})^{-1} K_{\mathcal{C}}^\top\right] = K^2.
$
\end{corollary}
A formula for the variance can also be obtained thanks to Lemma~\ref{Lem:Estimator}. Again, using DPP sampling with inverse leverage score preconditioning has the advantage that $K_{\mathcal{C}\mathcal{C}}$ is likely to be regular, in contrast with leverage score sampling.

\section{Experimental results}
\label{sec:experiments}

In this section, we illustrate the effect of sampling a subset $\mathcal{C}$ with small or large $\det(K_{\mathcal{C}\mathcal{C}})$ on a number of public datasets.  A swapping algorithm, described in Algorithm~\ref{AlgGreedySwap}, is used to determine subsets of size $|\mathcal{C}| = k$ with a given log-determinant $d_p$, such that $|\mathrm{log}(\mathrm{det}(K_{\mathcal{C}\mathcal{C}})) - d_p| \leq \epsilon$ and where $\epsilon>0$ is a numerical tolerance. The methods swaps points in and out of an initial subset, so that the swapped point is accepted if the determinant of the new submatrix is closer to the desired determinant $d_p$. If the determinant of the subset is too small, we sample a new candidate by using (approximate) leverage scores sampling. Otherwise, if the determinant is too large, we use inverse leverage scores sampling. The size of the subset is chosen to be the effective dimension $k = \sum_{i=1}^n \bm{\ell}_i$, where $\{\bm{\ell}_i\}_{i=1}^n$  are the ridge leverage scores with regularization parameter $\alpha = \lambda n$. The same ridge leverage scores are used in the greedy swapping algorithm. The algorithm stops if the desired precision is reached or the number of iterations exceeds 2000, whichever happens earlier. For large-scale problems, the ridge leverage scores are approximated using  Recursive Ridge Leverage Sampling (RRLS)~\cite{MuscoMusco} with $n_{RRLS}$ points. The size of the subset is chosen to be the effective dimension of the approximate RLS with regularization parameter $\alpha = \lambda n$. The same approximate RLS are used in the greedy swapping algorithm, where the maximum iterations is now equal to 5000.   

\paragraph{Settings}
In the sequel, a Gaussian kernel with bandwidth $\sigma$ is used after standardizing the data.
All the simulations are repeated 10 times, the averaged is displayed and the errorbars show the $0.05$ and $0.95$ quantile. The datasets and hyperparameters are given in Table~\ref{Table:data}. In the first case-studies, the following  exact algorithms are used to sample $k$ landmarks: Uniform sampling (Unif.), Ridge Leverage Score sampling (RLS)~\cite{ElAlaouiMahoney} and k-DPP \cite{kulesza2011k}.  For a fair comparison, we chose to use $k$-DPP rather than DPP so that the number of landmarks is always constant and equal to the expected subset size of the associated DPP. In the large-scale experiments: Unif., RRLS and the greedy swapping method are compared.

\begin{algorithm}[h]
	\centering
	\begin{algorithmic}[1]
		\Statex {\bf input}: Matrix $K\succ 0$, sample size $k$, ridge leverage scores $\{\bm{\ell}_i\}_{i=1}^n$, preferred log-determinant $d_{p}$ and precision $\epsilon>0$.
		\Statex {\bf initialization}: Sample an initial subset $|\mathcal{C}| = k$ uniformly at random.
		\Statex Determine the Cholesky decomposition $R$, with $K_{\mathcal{C}\mathcal{C}} = R^\top R$.
		\Statex {\bf repeat}:
		\Statex \quad Determine the log-determinant $d = 2\sum_{i=1}^k \mathrm{log}(R_{ii})$.	
		\Statex \quad {\bf if:} $|d - d_{p}| \leq \epsilon$	
		\Statex \qquad \quad {\bf break}
		\Statex \quad {\bf if:} $d < d_{p}$
		\Statex \qquad \quad Sample a new point $\tilde{c}$ out of the remaining subset with $p_i \sim  \bm{\ell}_i$
		\Statex \quad {\bf otherwise:}
		\Statex \qquad \quad Sample a new point $\tilde{c}$ out of the remaining subset with $p_i \sim 1 - \bm{\ell}_i$
		\Statex \quad Swap a uniform selected point out of $\mathcal{C}$ with the newly sampled $\tilde{c}$, which gives the new subset $\tilde{\mathcal{C}}$
		\Statex \quad Do a rank-1 update to the Cholesky decomposition, which gives $\tilde{R}$ and determine $\tilde{d} = 2\sum_{i=1}^k \mathrm{log}(\tilde{R}_{ii})$.
		\Statex \quad {\bf if:} $|\tilde{d} - d_{p}| \leq  |d - d_{p}|$	
		\Statex \qquad \quad Keep the swapped point and update $R = \tilde{R}$, $\mathcal{C} = \tilde{\mathcal{C}}$		
		\Statex {\bf return} $\mathcal{C}$.
	\end{algorithmic} 
	\caption{Greedy Swapping Algorithm based on the (approximated) ridge leverage scores. 
	\label{AlgGreedySwap}}
\end{algorithm}

\paragraph{Nystr\"om approximation}
The impact of diversity on the Nystr\"om approximation is illustrated on the \texttt{Housing}, \texttt{Abalone}, \texttt{codRNA} and \texttt{MiniBooNE} datasets\footnote{\url{https://www.cs.toronto.edu/~delve/data/datasets.html}, \url{https://archive.ics.uci.edu/ml/index.php}\label{footnote:datasets}}.
The condition number of $K_{\mathcal{C}\mathcal{C}}$, its largest/smallest eigenvalues and the accuracy of the Nystr\"om approximation are plotted as a function of the determinant in Figure~\ref{fig:Nystrom}. 
For completeness, the largest and smallest eigenvalues of $K_{\mathcal{C}\mathcal{C}}$ are also given in the appendix. The accuracy of the approximation is evaluated by calculating $\|K-\hat{K}\|_F/\|K\|_F$ with $\hat{K} = KC (K_{\mathcal{C}\mathcal{C}}+\varepsilon\mathbb{I}_{\mathcal{C}\mathcal{C}})^{-1}C^\top K$ with $\varepsilon = 10^{-12}$ for numerical stability. 
Afterwards, the following  algorithms are used to sample $k$ landmarks: Uniform sampling (Unif.), exact ridge leverage score sampling (RLS)~\cite{ElAlaouiMahoney} and k-DPP \cite{kulesza2011k}.  
The results in Figure~\ref{fig:Nystrom} show that the 3 sampling algorithms follow the general trend of the greedy swapping algorithm, namely, we have to following empirical observations:
\textbf{1)} Sampling a more diverse subset results in a smaller condition number $\kappa(K_{\mathcal{C}\mathcal{C}})$. This is mainly because a larger determinant corresponds to a larger $\lambda_{\min}(K_{\mathcal{C}\mathcal{C}})$.  Indeed,  diverse sampling is a computational regularization. 
\textbf{2)} Sampling a diverse subset gives a more accurate Nystr\"om approximation.  In practice, we observe that RLS sampling yield \emph{effectively} more diverse samples compared to uniform sampling.
Notice that in the presence of outliers, taking samples with an extremely large $\mathrm{det}(K_{\mathcal{C}\mathcal{C}})$ thanks to the Greedy Swapping Algorithm might increase the error on the Nystr\"om approximation as it explained in Supplementary Material. The results for the large-scale experiments are visualized on Figure~\ref{fig:NystromLS}. The accuracy of the approximation is now evaluated by averaging the Frobenius norm error $\|K-\hat{K}\|_F$ over 50 subsets of size 3000.

	\begin{figure}[h]
		\centering
		\begin{subfigure}[b]{0.45\textwidth}
			\includegraphics[width=1\textwidth, height= 0.95\textwidth]{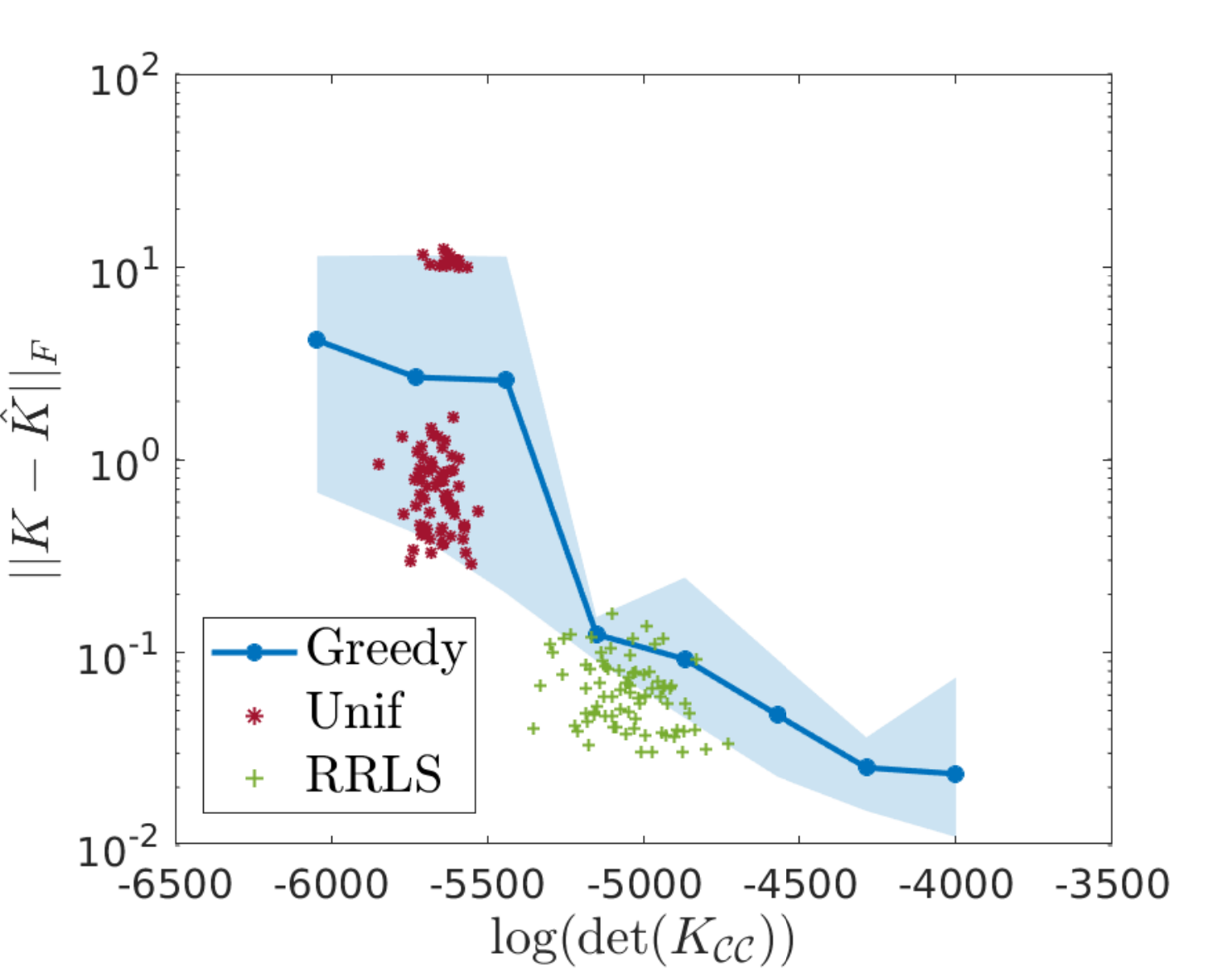}
			\caption{\texttt{MiniBooNE}: error}
		\end{subfigure}
		\begin{subfigure}[b]{0.45\textwidth}
			\includegraphics[width=1\textwidth, height= 0.95\textwidth]{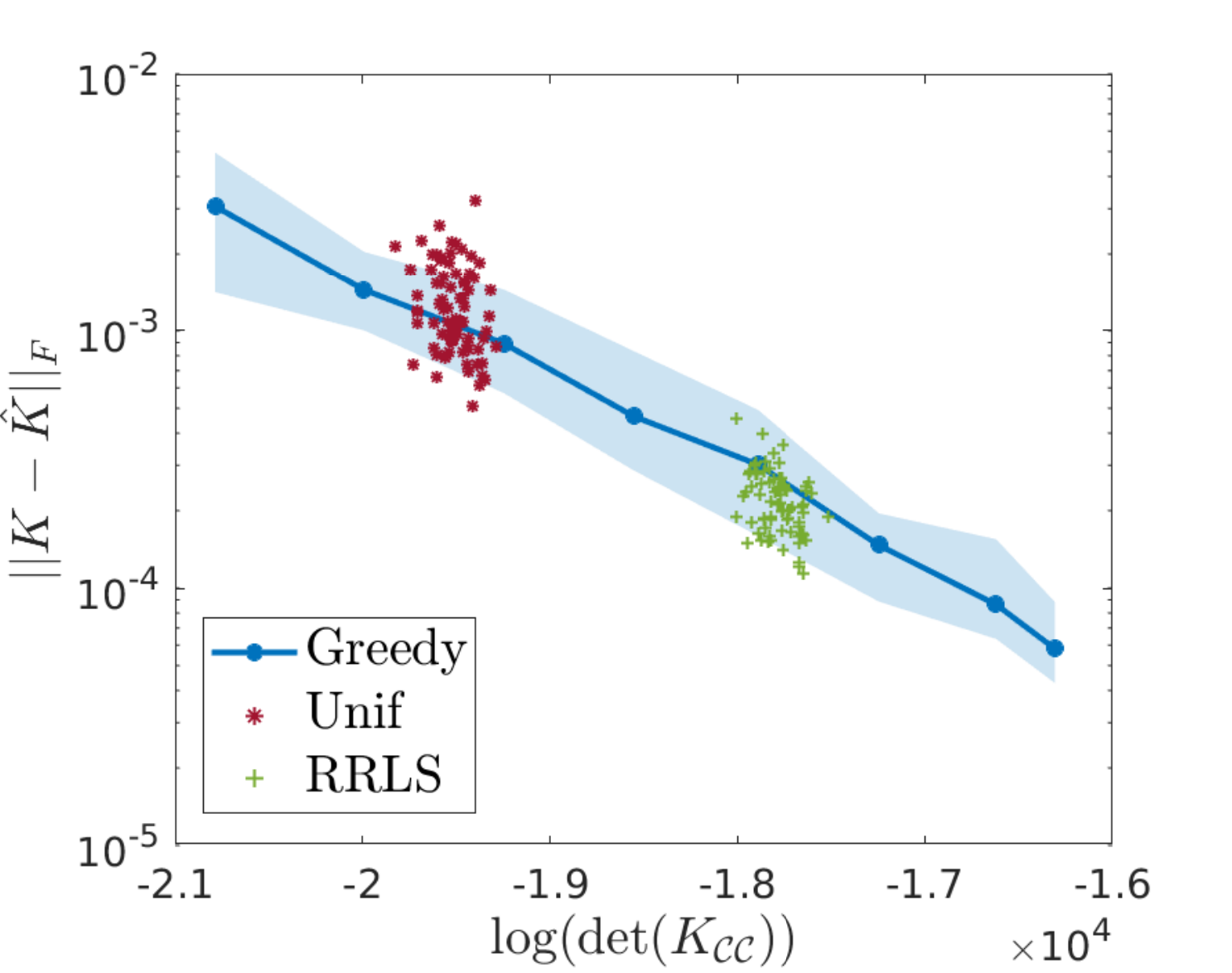}
			\caption{\texttt{codRNA}: error}
		\end{subfigure}		
		\caption{Large-scale Nystr\"om approximation results. The Frobenius norm of the approximation error of $\hat{K} = L(K,\mathcal{C})$ is plotted versus the sample diversity. The larger  $\mathrm{det}(K_{\mathcal{C}\mathcal{C}})$, the more diverse the subset.}\label{fig:NystromLS}
	\end{figure}

	\begin{figure}[h]
		\centering
		\begin{subfigure}[t]{0.24\textwidth}
			\includegraphics[width=\textwidth, height= \textwidth]{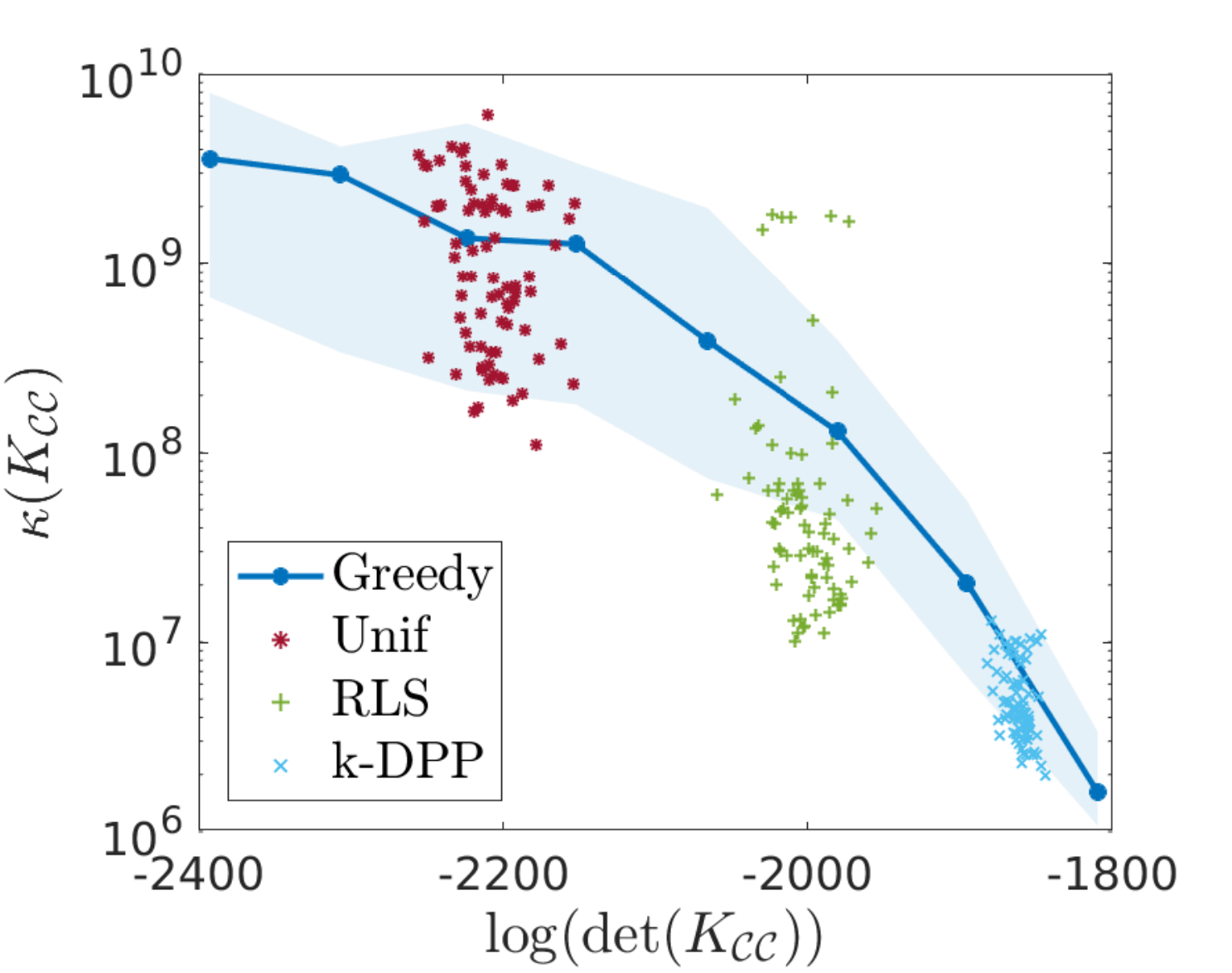}
			\caption{\texttt{A.Credit}: $\kappa(K_{\mathcal{C}\mathcal{C}})$}
		\end{subfigure}
		\begin{subfigure}[t]{0.24\textwidth}
			\includegraphics[width=\textwidth, height= \textwidth]{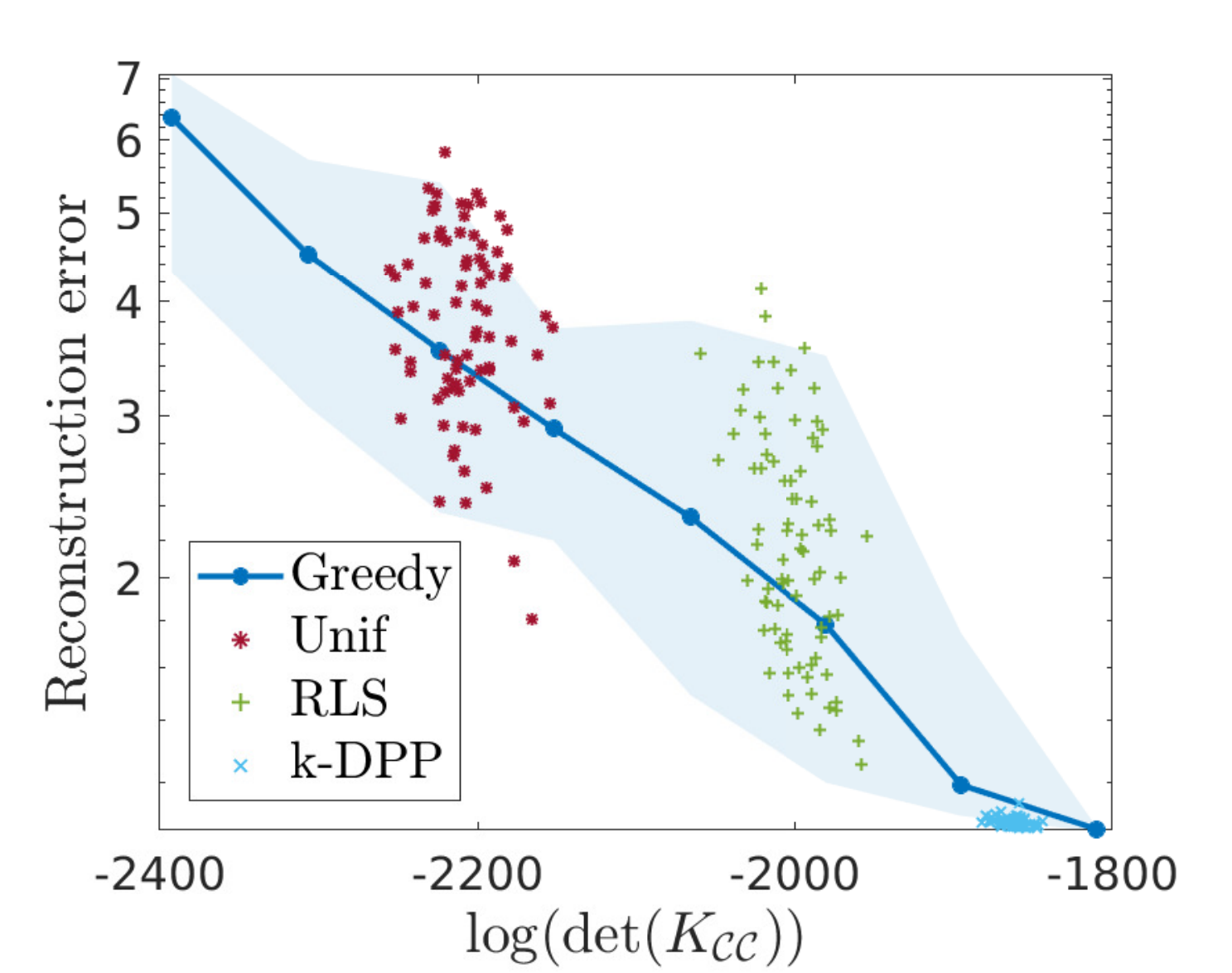}
			\caption{\texttt{A.Credit}: error}
		\end{subfigure}
		\begin{subfigure}[t]{0.24\textwidth}
			\includegraphics[width=\textwidth, height= \textwidth]{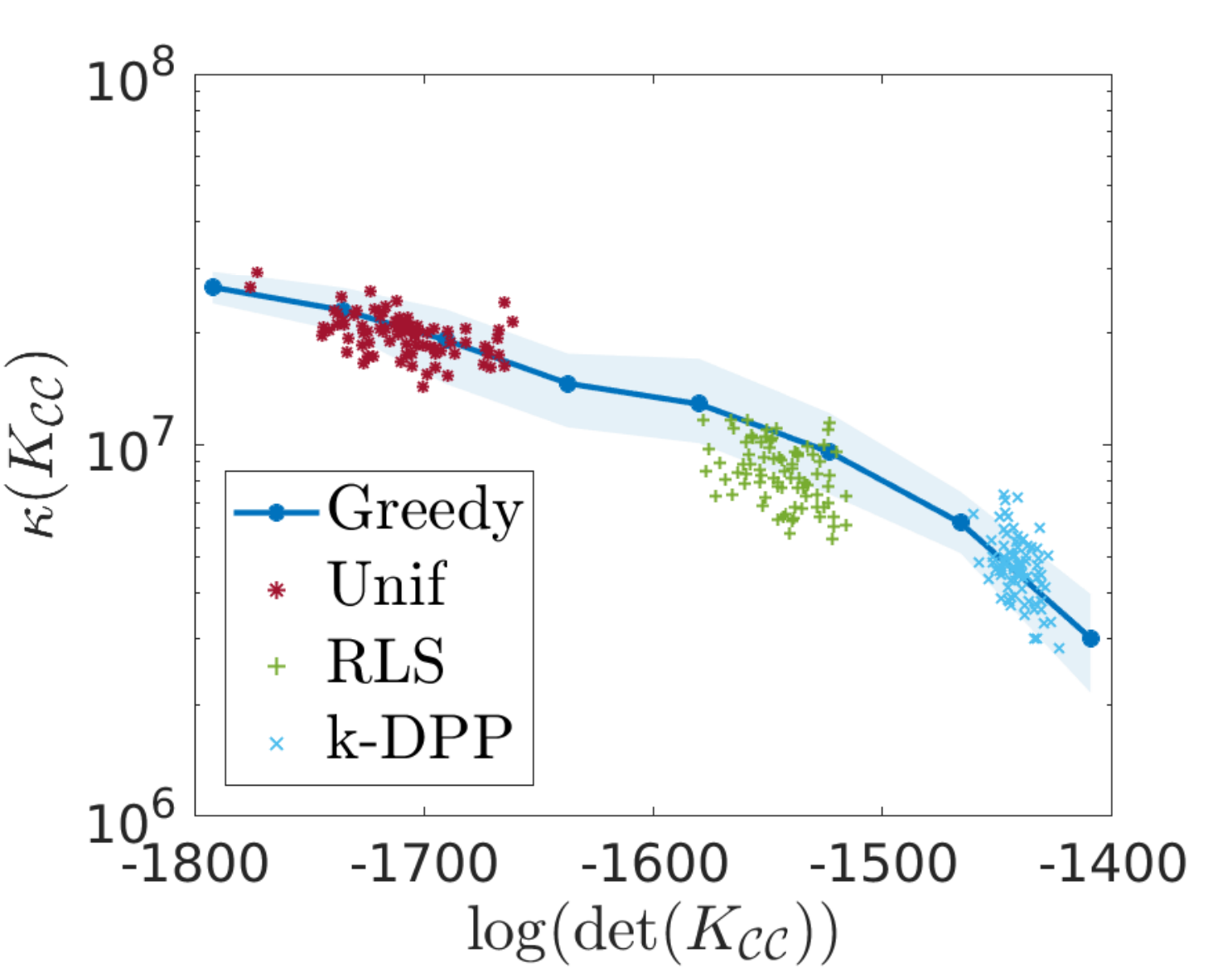}
			\caption{\texttt{B.Cancer}: $\kappa(K_{\mathcal{C}\mathcal{C}})$}
		\end{subfigure}
		\begin{subfigure}[t]{0.24\textwidth}
			\includegraphics[width=\textwidth, height= \textwidth]{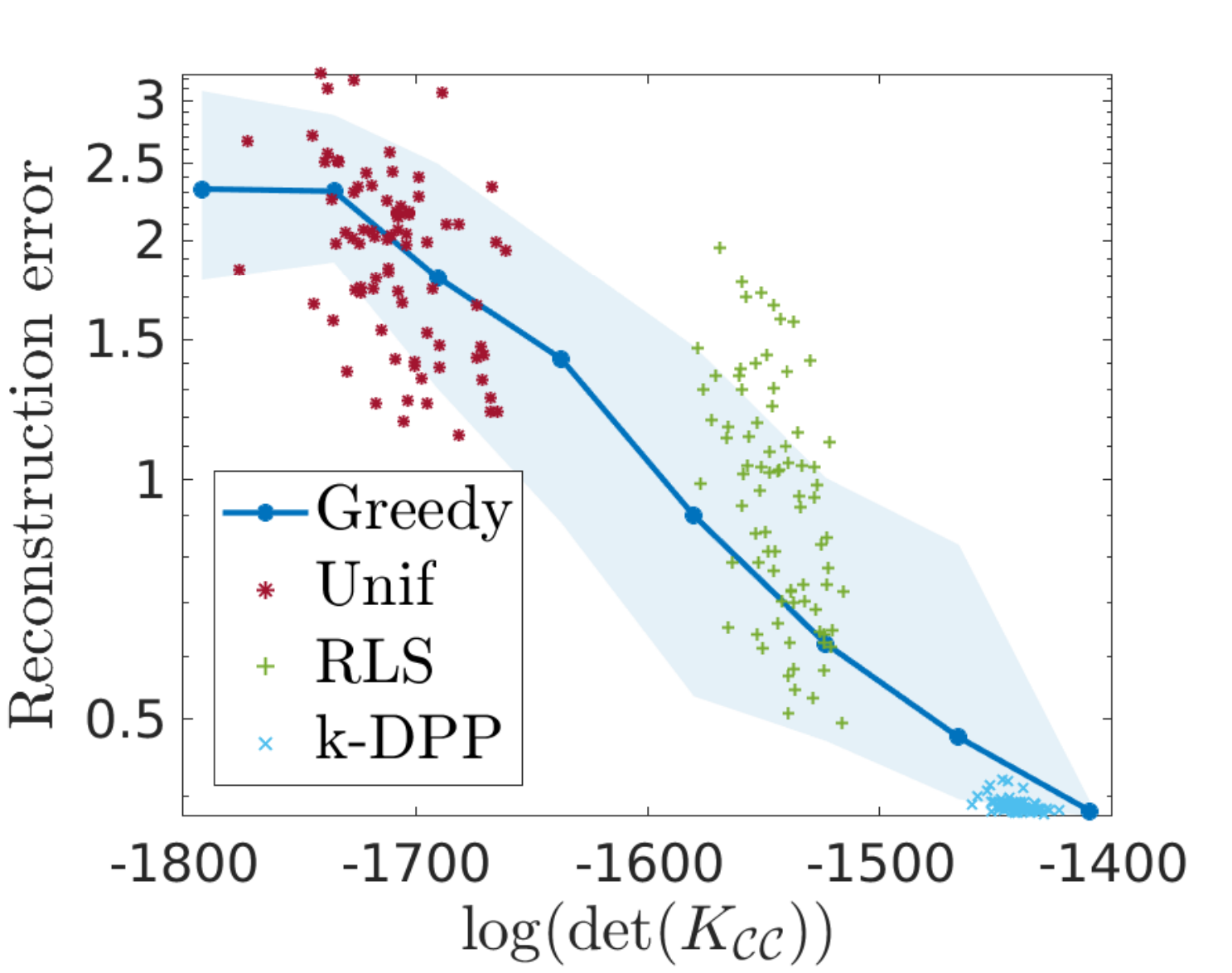}
			\caption{\texttt{B.Cancer}: error}
		\end{subfigure}					
		\caption{KPCA results. The condition number, minimum eigenvalues and reconstruction error using half of the components are plotted as a function of $\mathrm{det}(K_{\mathcal{C}\mathcal{C}})$.}\label{fig:KPCA}
	\end{figure}
	\begin{figure}[h]
		\centering
		\begin{subfigure}[t]{0.24\textwidth}
			\includegraphics[width=\textwidth, height= 0.95\textwidth]{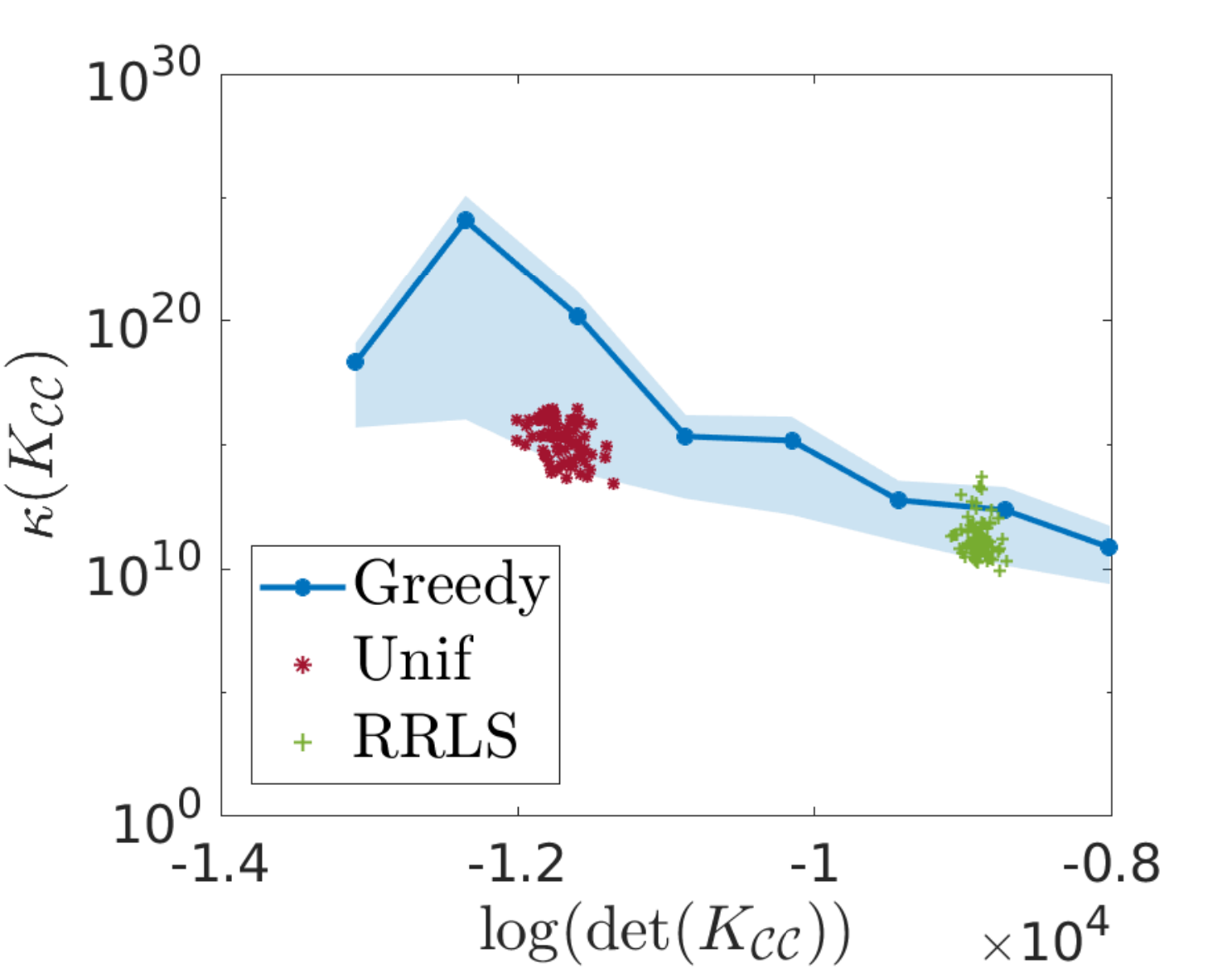}
			\caption{\texttt{Adult}: $\kappa(K_{\mathcal{C}\mathcal{C}})$}
		\end{subfigure}
		\begin{subfigure}[t]{0.24\textwidth}
			\includegraphics[width=\textwidth, height= 0.95\textwidth]{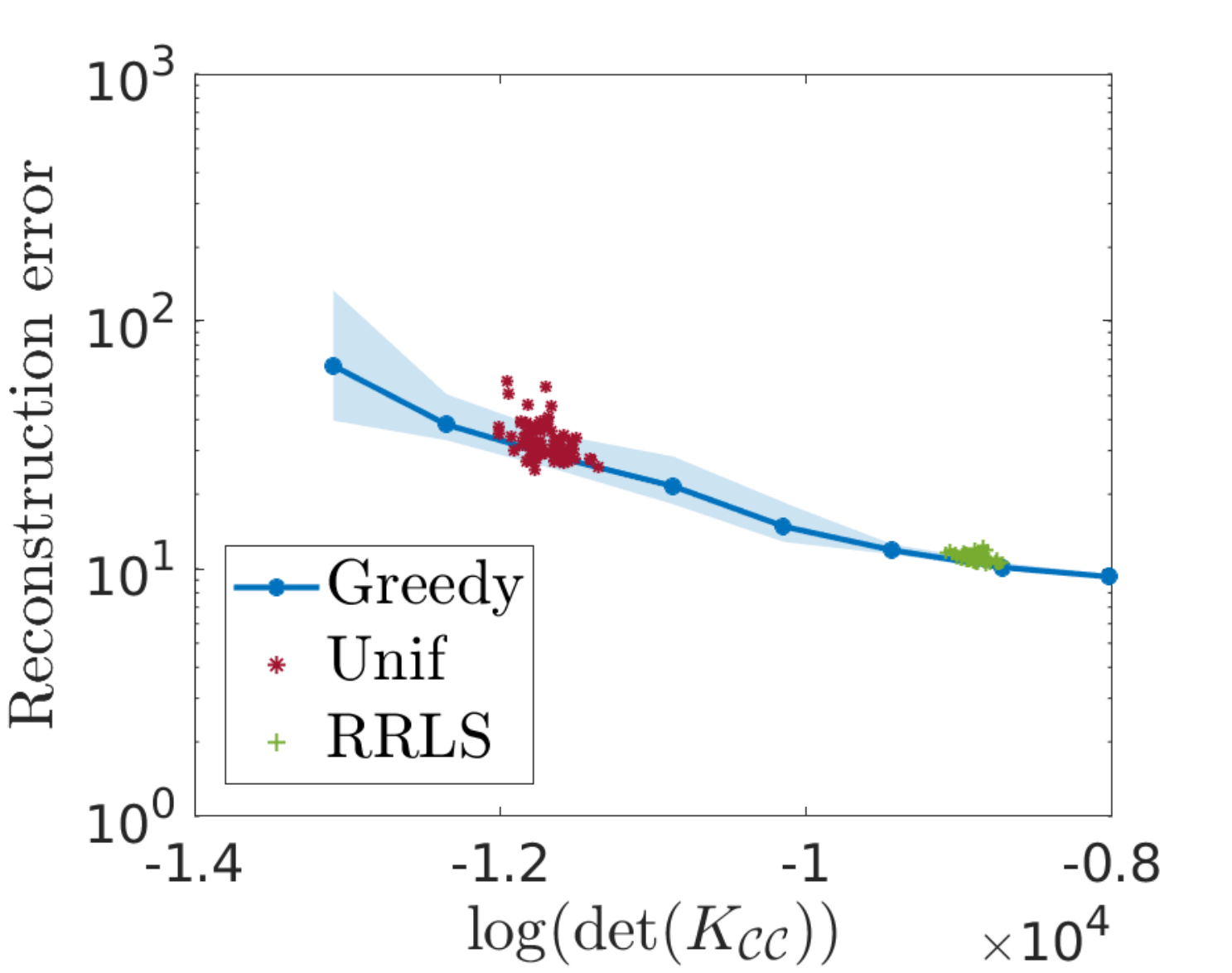}
			\caption{\texttt{Adult}: error}
		\end{subfigure}
		\begin{subfigure}[t]{0.24\textwidth}
			\includegraphics[width=\textwidth, height= 0.95\textwidth]{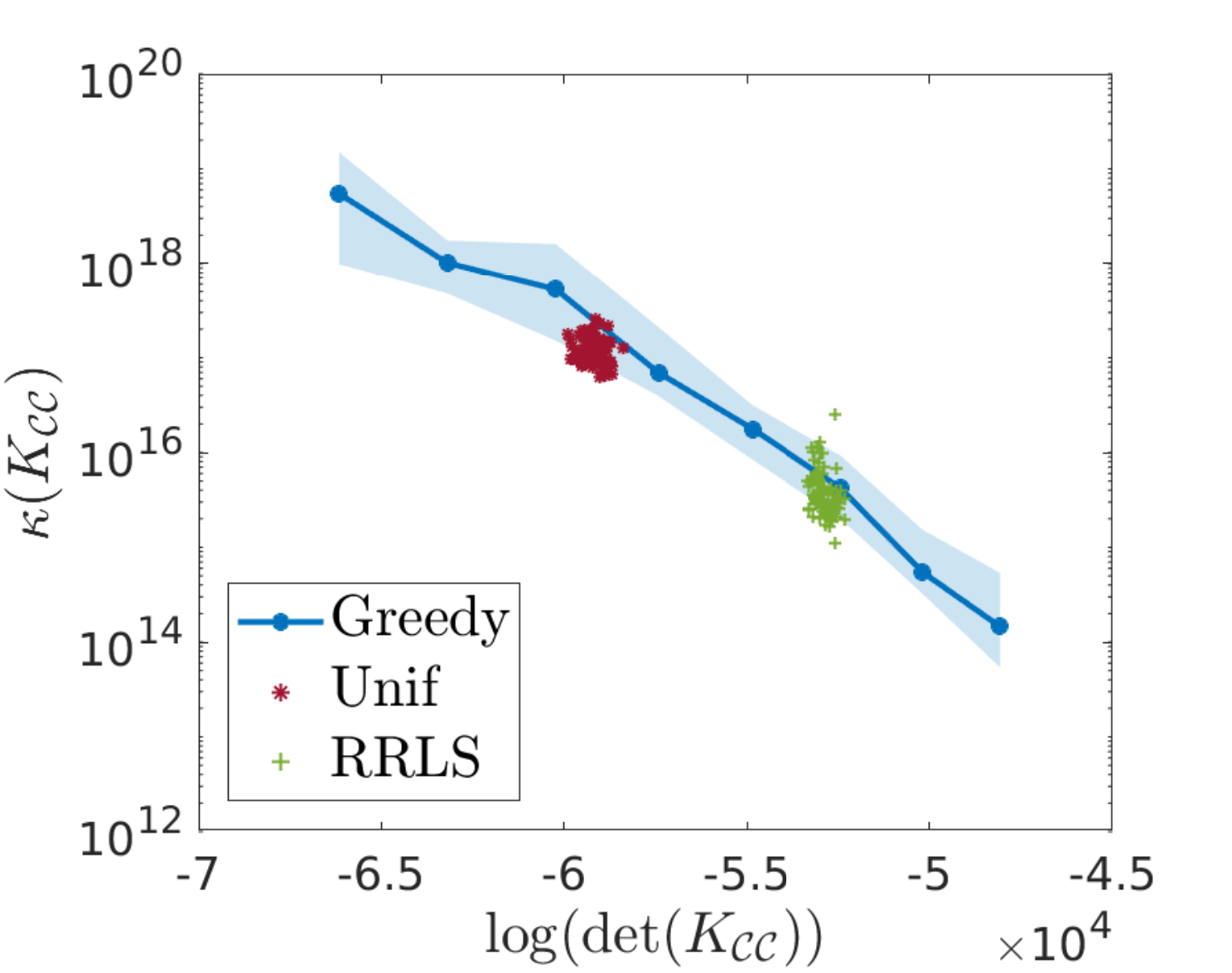}
			\caption{\texttt{Cov}: $\kappa(K_{\mathcal{C}\mathcal{C}})$}
		\end{subfigure}
		\begin{subfigure}[t]{0.24\textwidth}
			\includegraphics[width=\textwidth, height= 0.95\textwidth]{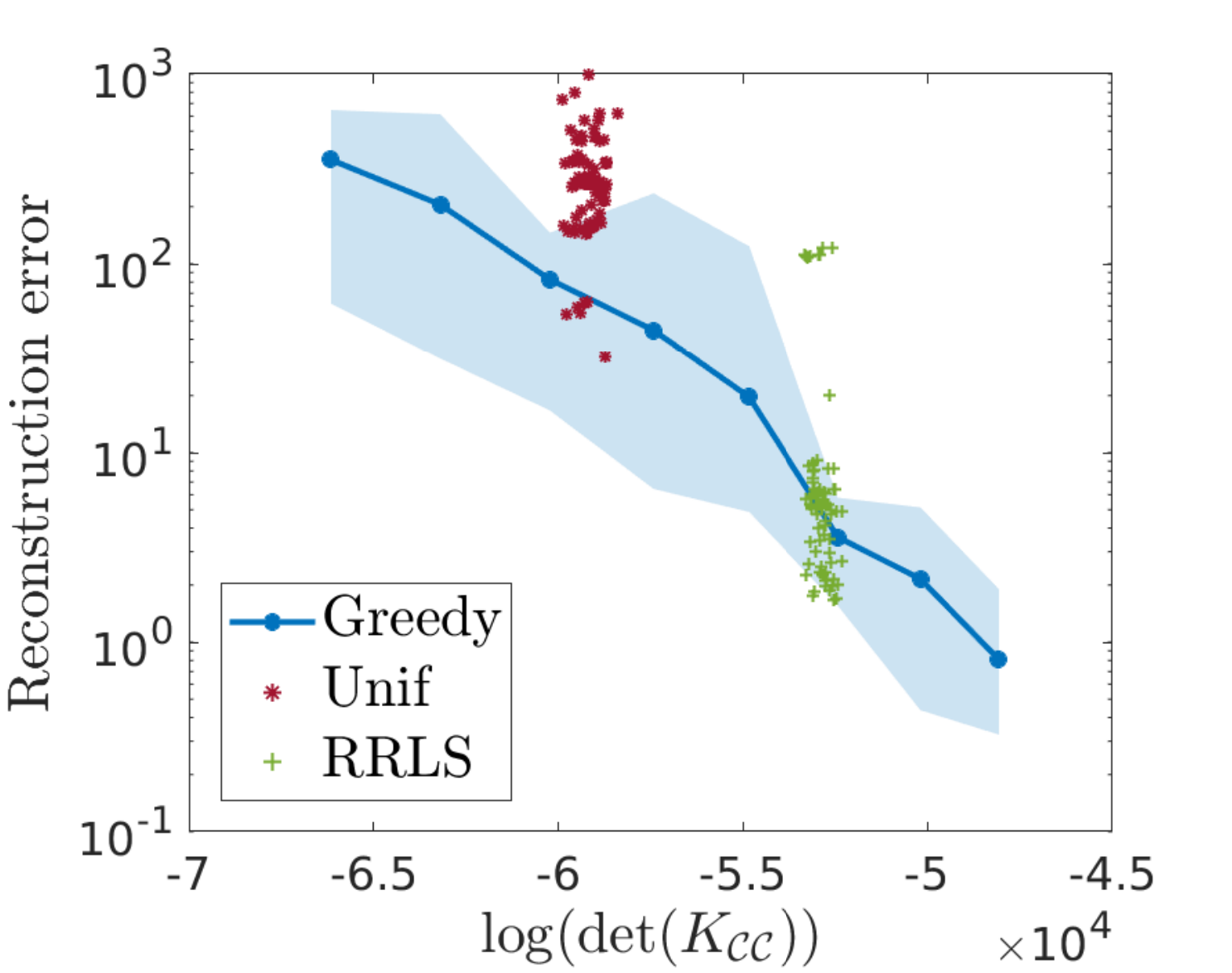}
			\caption{\texttt{Cov}: error}
		\end{subfigure}					
		\caption{Large-scale KPCA results. The condition number and reconstruction error using half of the components are plotted as a function of $\mathrm{det}(K_{\mathcal{C}\mathcal{C}})$.}\label{fig:LS_KPCA}
	\end{figure}	
\paragraph{Kernel PCA}
The numerical experiments are done on the \texttt{Breast Cancer} (\texttt{B. Cancer}), \texttt{Australian Credit} (\texttt{A. Credit}), \texttt{Adult} and \texttt{Covertype} datasets\footnotemark[1].
The condition number, smallest/largest eigenvalues of $K_{\mathcal{C}\mathcal{C}}$ and the reconstruction error are plotted as a function of the determinant. The averaged results are visualized in Figure~\ref{fig:KPCA}. Information about the datasets and hyperparameters used for the experiments is given in Table~\ref{Table:data} in appendix. Empirically, sampling a more diverse subset results in a smaller reconstruction error and a smaller condition number.  The results for the large-scale experiments are visualized on Figure~\ref{fig:LS_KPCA}.

\paragraph{Regression and Stratification of the error} To conclude, we verify the usefulness of diversity for a supervised learning task. The dataset is split in $50\%$ training data and $50\%$ test data, so to make sure the train and test set have similar RLS distributions. The test RLS distribution is visualized in Figure~\ref{fig:EIG_Regression}.
The regression experiment is repeated on the \texttt{Abalone}, \texttt{Wine Quality}, \texttt{Bike Sharing} (\texttt{Bike S.}) and \texttt{YearPredictionMSD} (\texttt{Year}) datasets\footnotemark[1] by using KRR. The MAPE of the kernel ridge regression  is calculated as a function of $\mathrm{det}(K_{\mathcal{C}\mathcal{C}})$. 
To evaluate the performance, the dataset is stratified, i.e., the test set is divided into `bulk' and `tail' as follows:
the bulk corresponds to test points where the RLS are smaller than or equal to the 70\% quantile, while the tail of the data corresponds  to test points where the ridge leverage score is larger than the 70\% quantile. 
	 \begin{figure}[h]
	\centering	
	\begin{subfigure}[b]{0.45\textwidth}
		\includegraphics[width=1.02\textwidth, height= 0.8\textwidth]{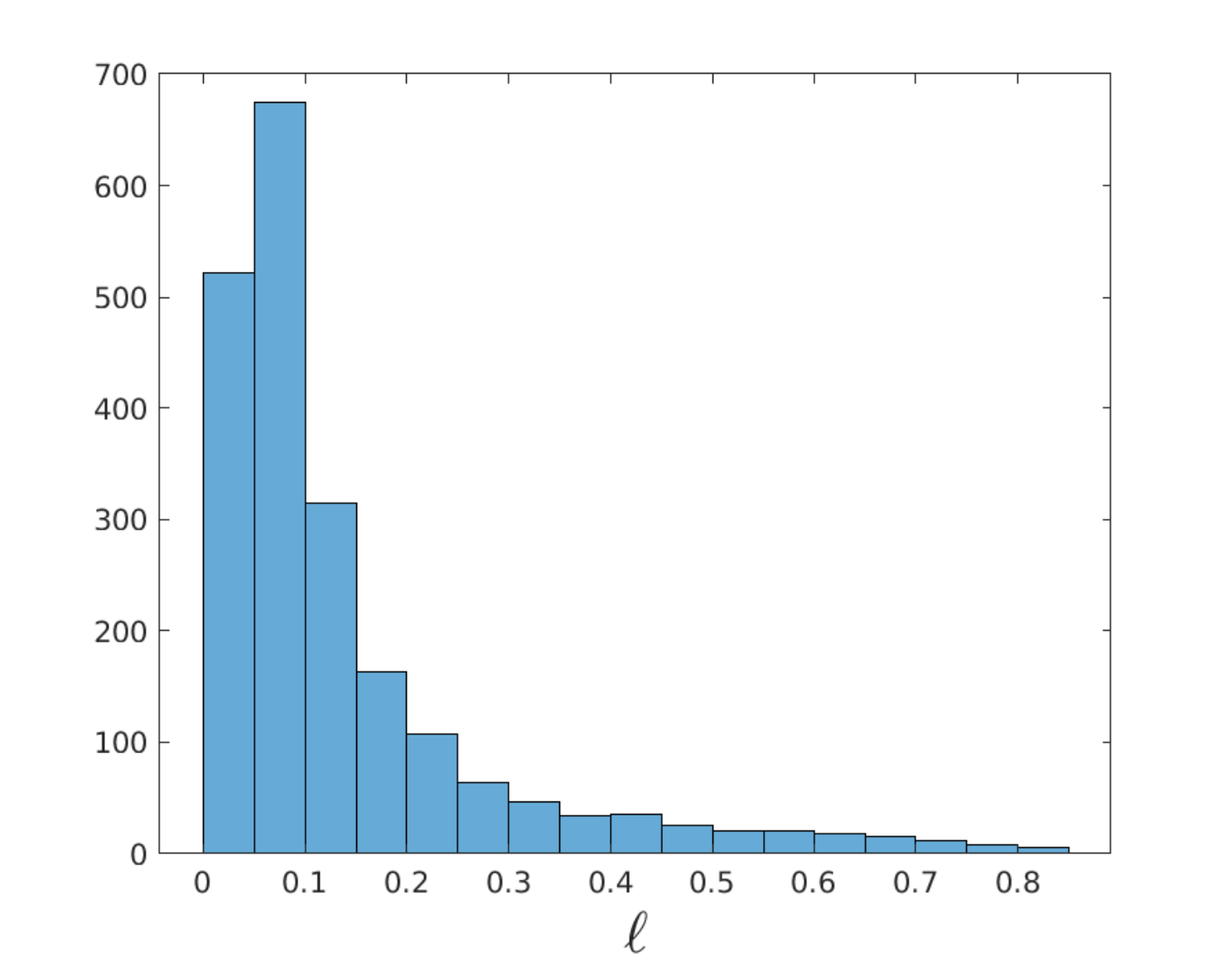}
		\caption{\texttt{Abalone}: RLS distribution}
	\end{subfigure}
	\begin{subfigure}[b]{0.45\textwidth}
		\includegraphics[width=\textwidth, height= 0.8\textwidth]{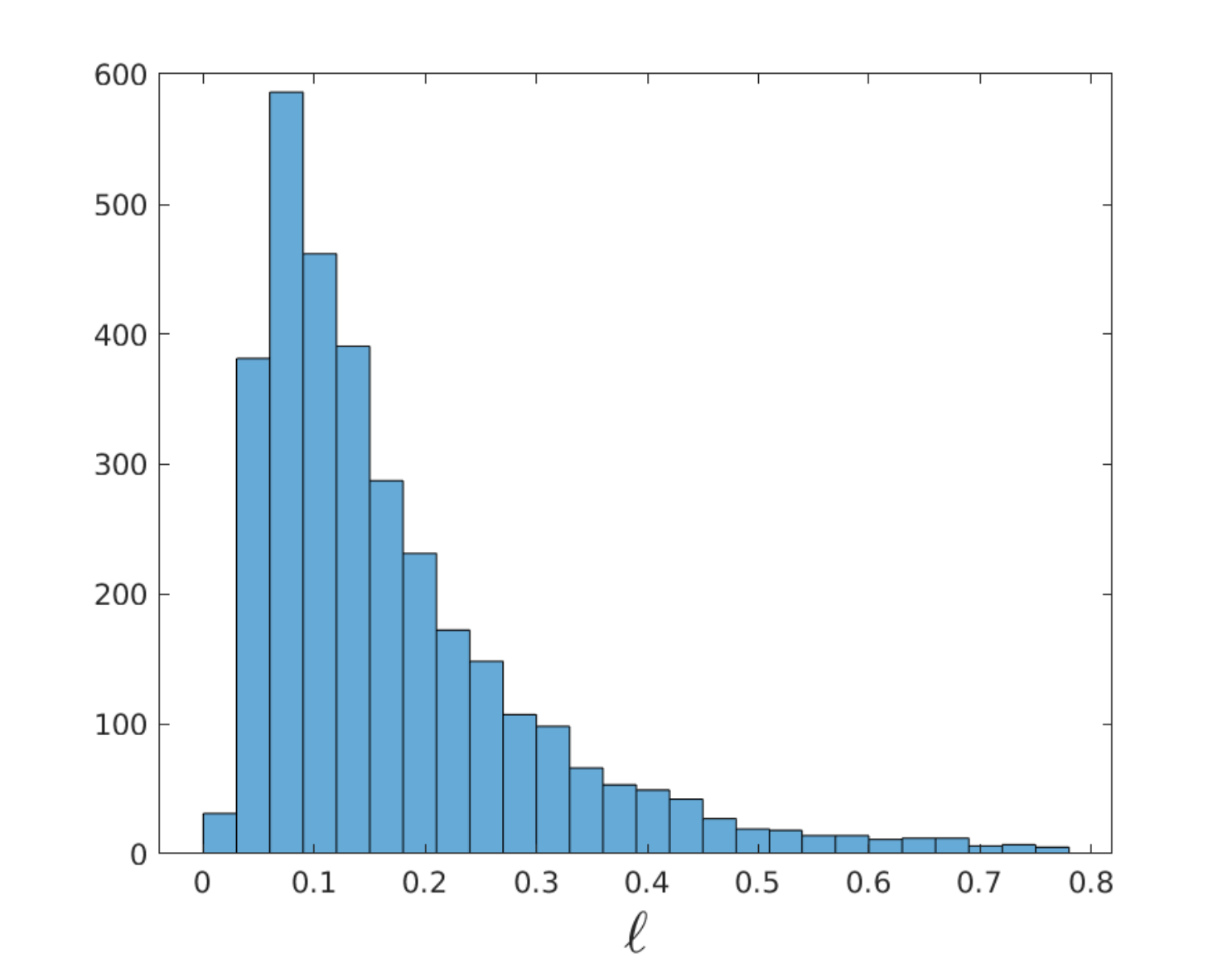}
		\caption{\texttt{Wine Q.}: RLS distribution}
	\end{subfigure}
	\begin{subfigure}[b]{0.45\textwidth}
		\includegraphics[width=1.02\textwidth, height= 0.8\textwidth]{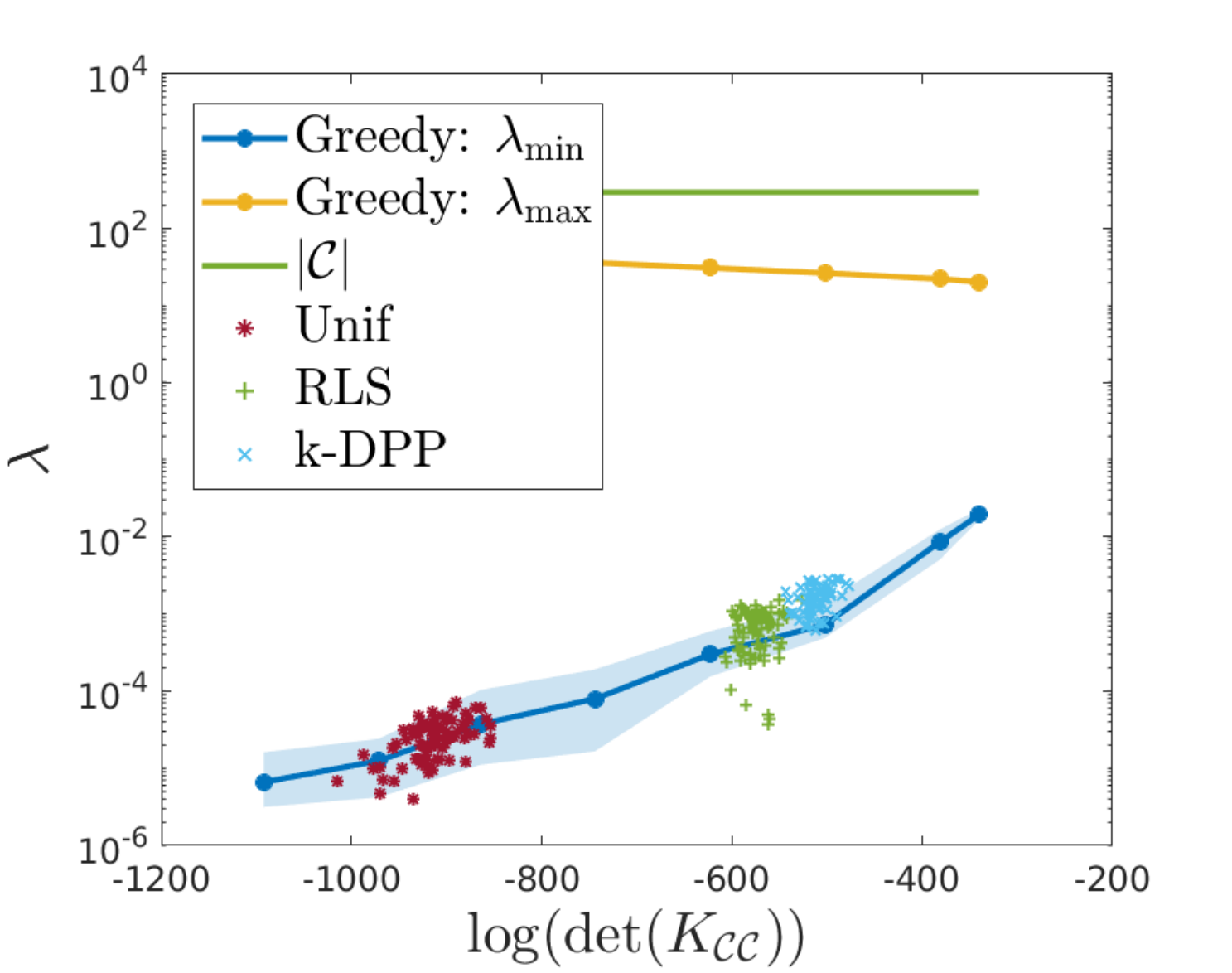}
		\caption{\texttt{Abalone}: eigenvalues}
	\end{subfigure}
	\begin{subfigure}[b]{0.45\textwidth}
		\includegraphics[width=\textwidth, height= 0.8\textwidth]{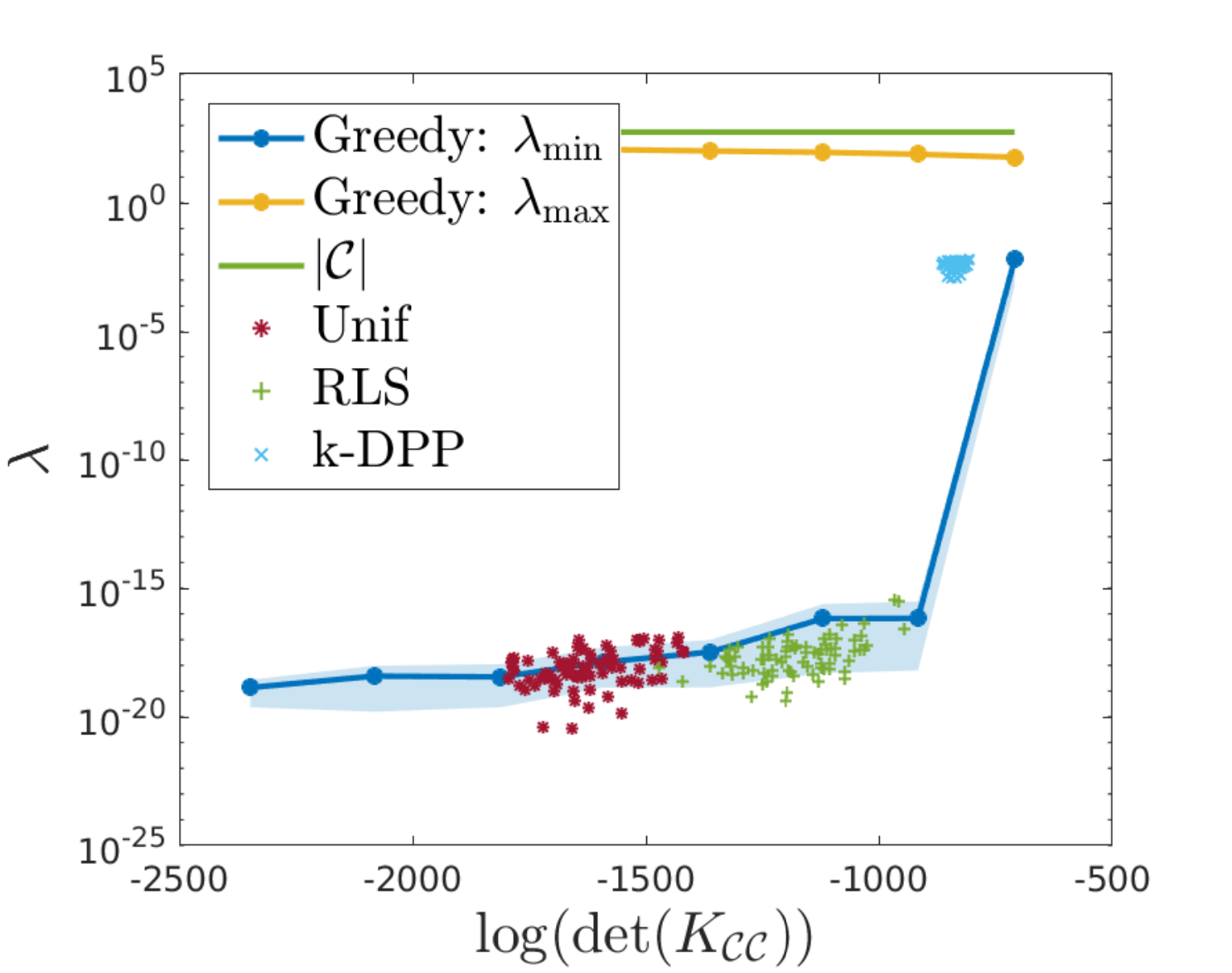}
		\caption{\texttt{Wine Q.}: eigenvalues}
	\end{subfigure}
	\caption{Regression results. The histogram of the RLS of the test set and minimum and maximum eigenvalues of $K_{\mathcal{C}\mathcal{C}}$ versus $\mathrm{det}(K_{\mathcal{C}\mathcal{C}})$, accompanying Figure \ref{fig:Regression}.}\label{fig:EIG_Regression}
\end{figure}
The regularization parameter $\gamma$ is determined by using cross-validation. The results in Figure~\ref{fig:Regression} show that the 3 sampling algorithms follow the general trend of the greedy swapping algorithm. Again here, sampling a more diverse subset, results in a better conditioning of the kernel sub-matrix. 
Diverse sampling has comparable performance for the bulk data, while performing much better in the tail of the data. This confirms the expectations from the regression toy example.  The results for the large-scale experiments are visualized on Figure~\ref{fig:RegressionLS}, where the Symmetric MAPE (SMAPE) is shown in the bulk and tail of the data.

\begin{figure}[h]
		\centering
		\begin{subfigure}[b]{0.45\textwidth}
			\includegraphics[width=\textwidth, height= 0.8\textwidth]{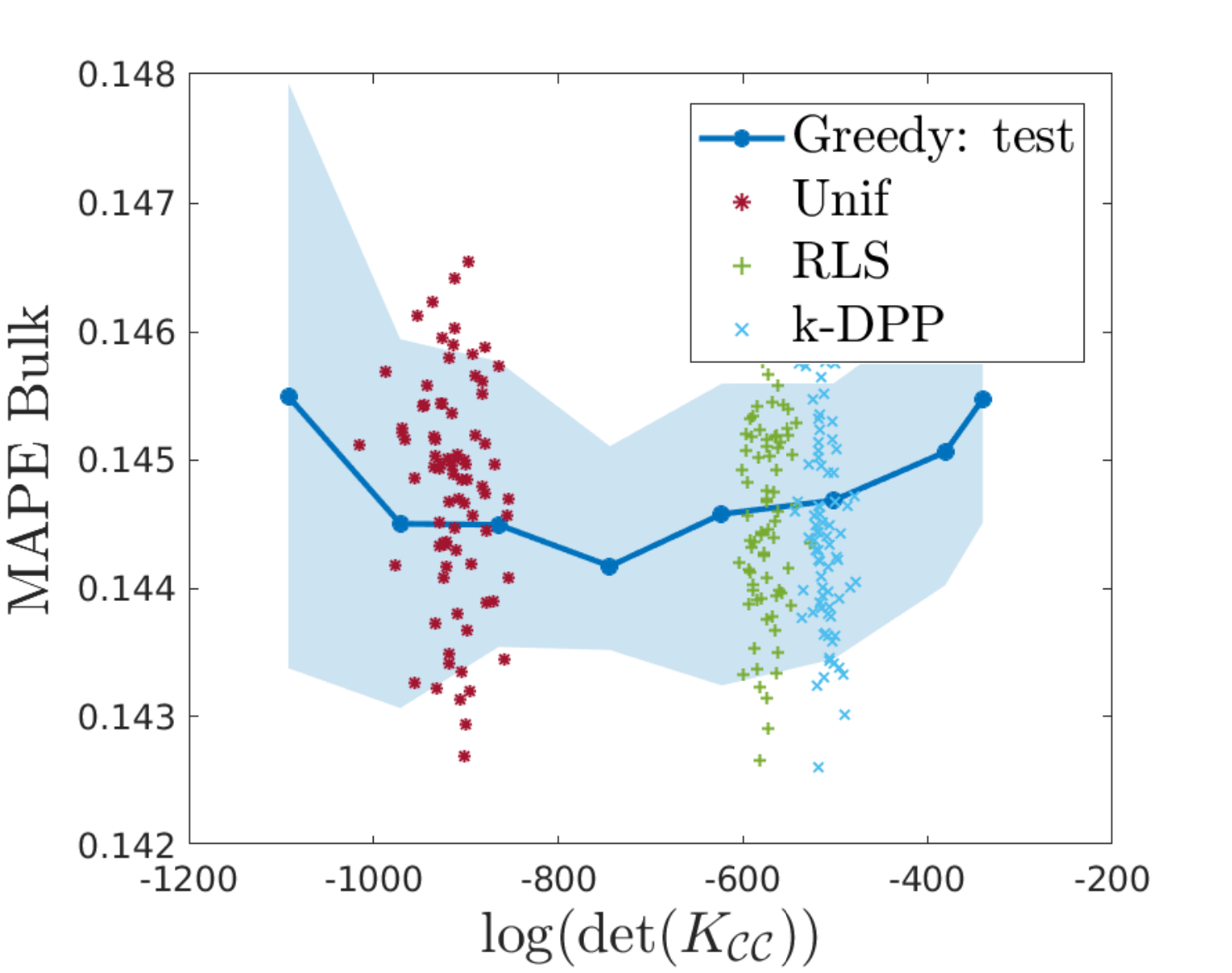}
			\caption{\texttt{Abalone}: MAPE Bulk}
		\end{subfigure}
		\begin{subfigure}[b]{0.45\textwidth}
			\includegraphics[width=\textwidth, height= 0.8\textwidth]{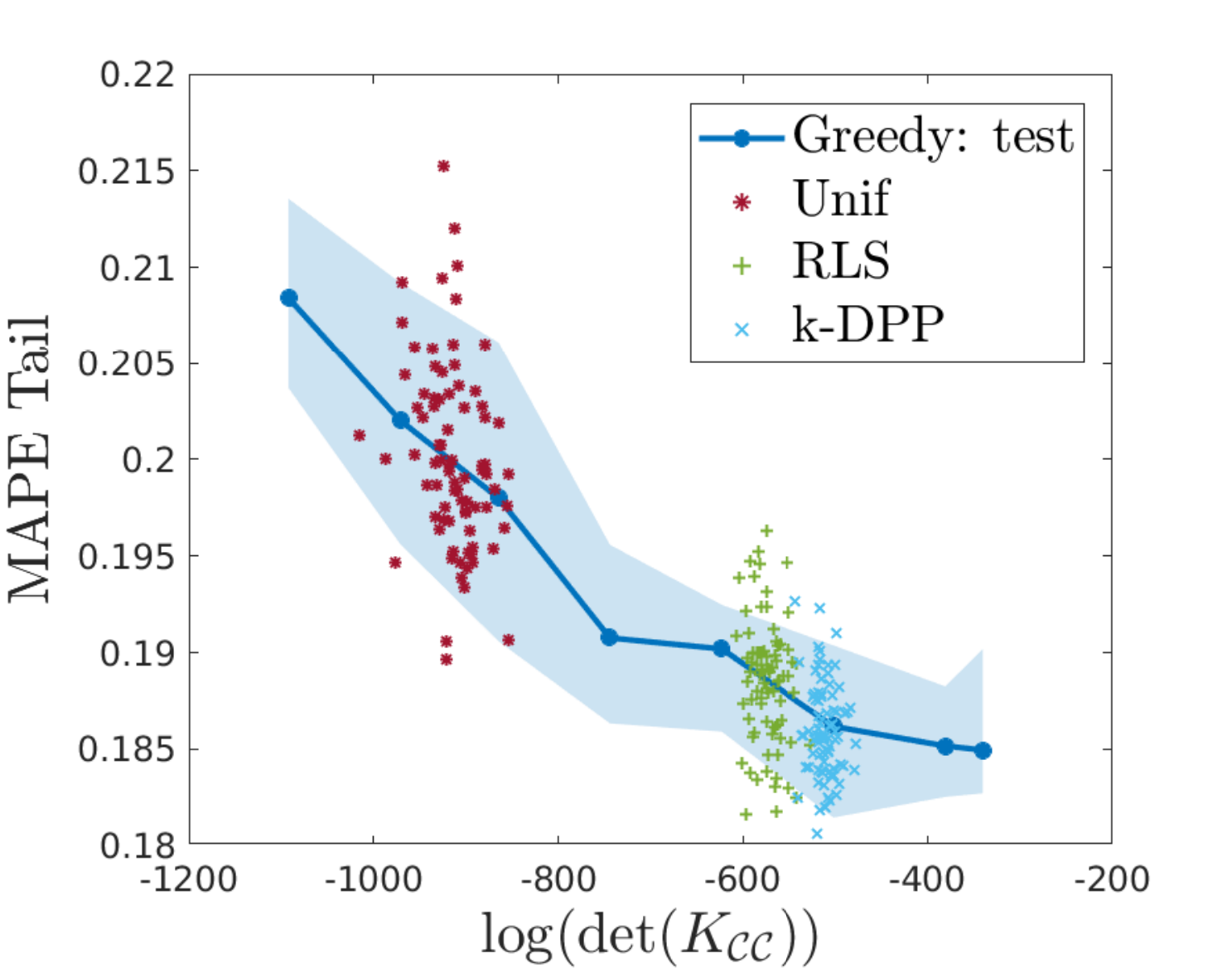}
			\caption{\texttt{Abalone}: MAPE Tail}
		\end{subfigure}
		\begin{subfigure}[b]{0.45\textwidth}
			\includegraphics[width=\textwidth, height= 0.8\textwidth]{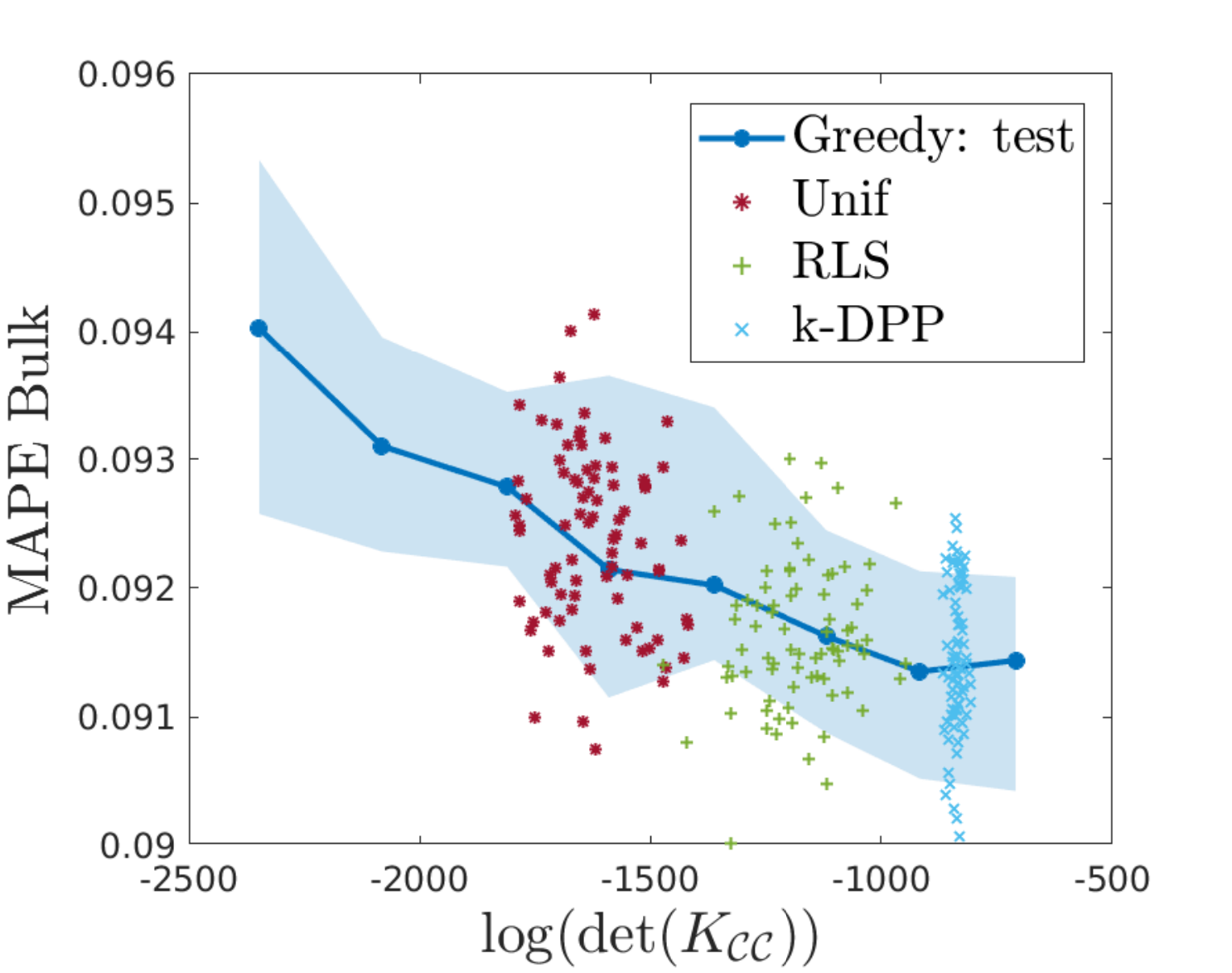}
			\caption{\texttt{Wine Q.}: MAPE Bulk}
		\end{subfigure}
		\begin{subfigure}[b]{0.45\textwidth}
			\includegraphics[width=\textwidth, height= 0.8\textwidth]{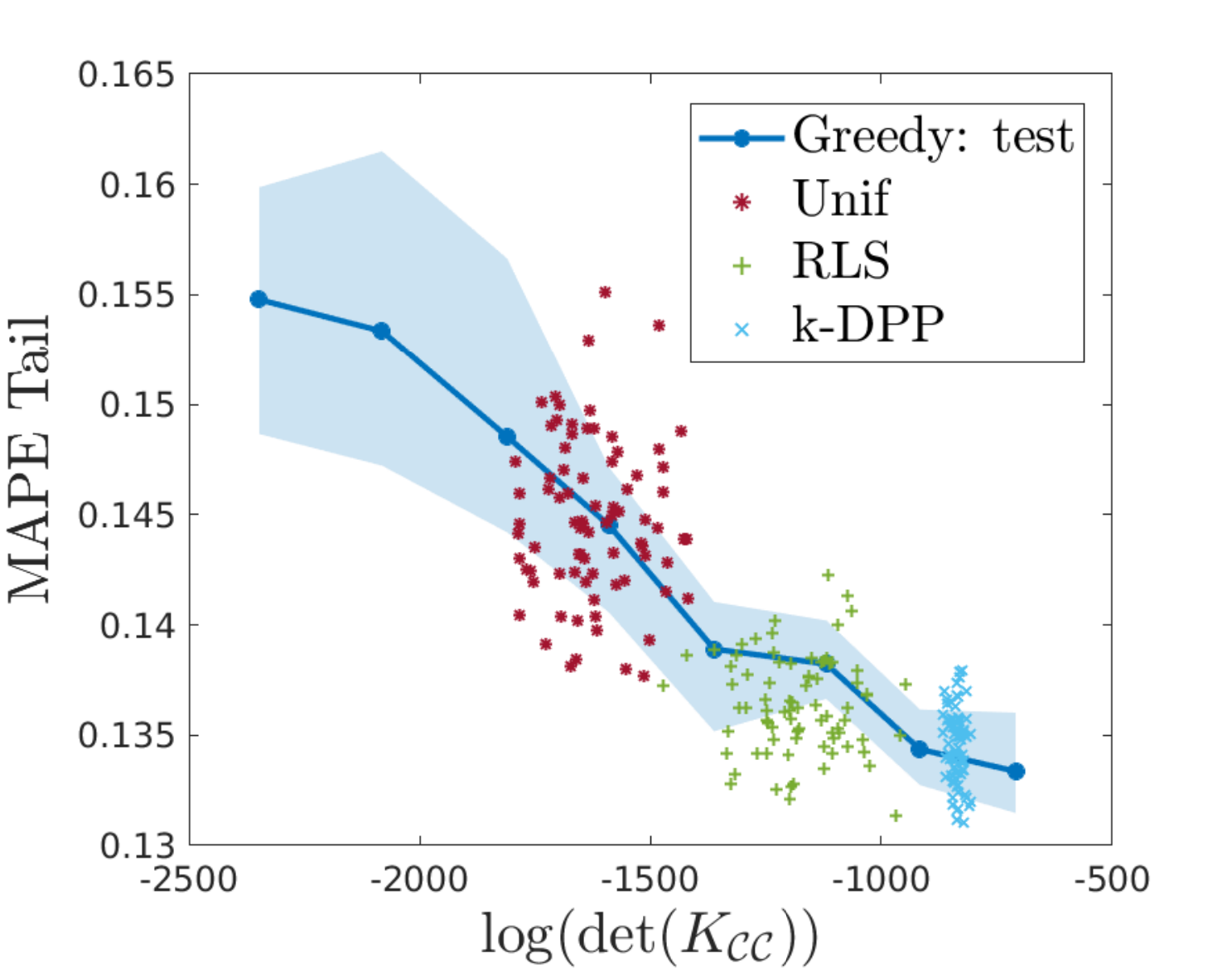}
			\caption{\texttt{Wine Q.}: MAPE Tail}
		\end{subfigure}				
		\caption{Regression results. The condition number and MAPE on the test set are plotted as a function of  $\mathrm{det}(K_{\mathcal{C}\mathcal{C}})$. A small MAPE corresponds to a large accuracy.}\label{fig:Regression}
	\end{figure}

	\begin{figure}[h]
		\centering
		\begin{subfigure}[b]{0.24\textwidth}
			\includegraphics[width=\textwidth, height= 0.95\textwidth]{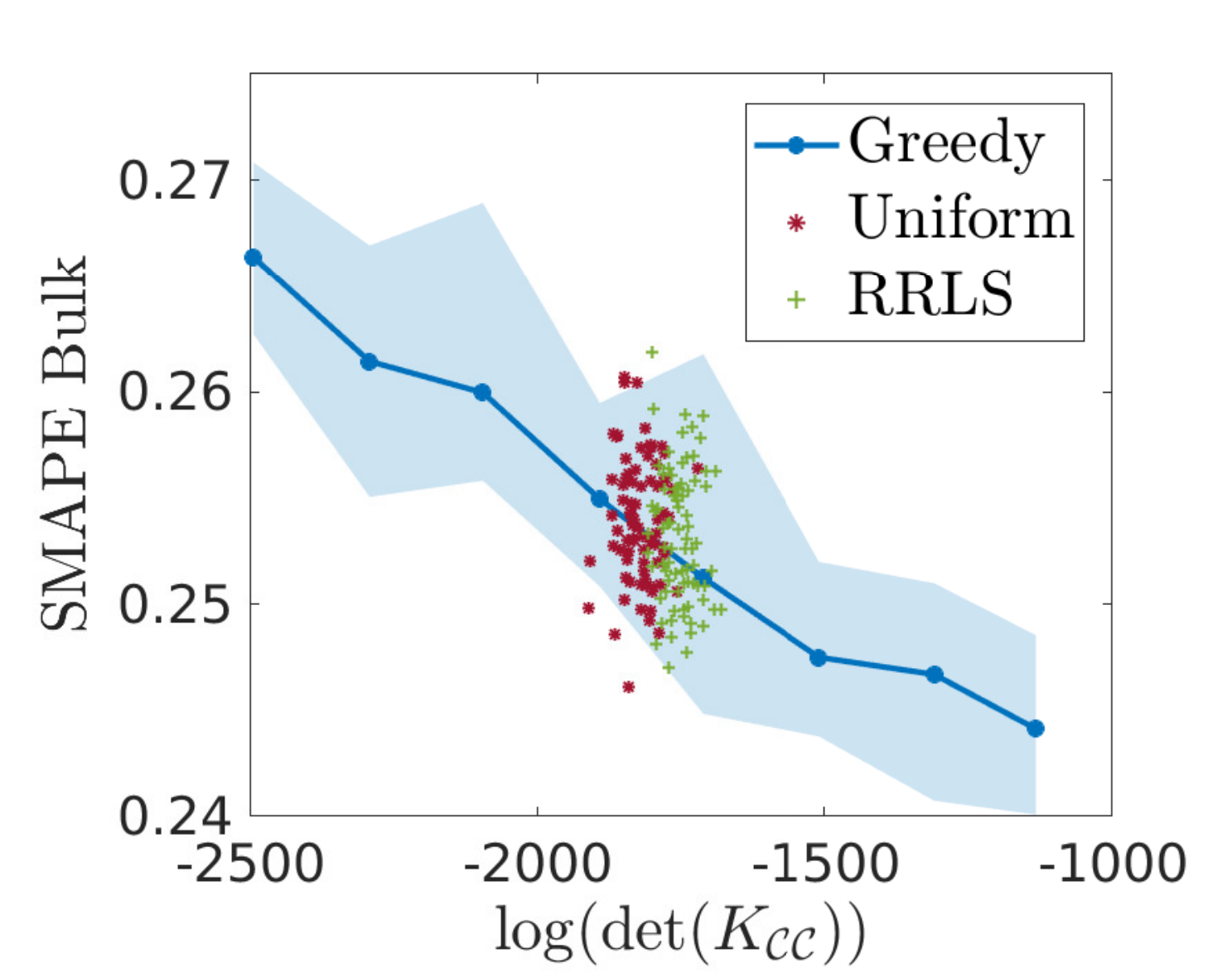}
			\caption{\texttt{Bike S.}: Bulk}
		\end{subfigure}
		\begin{subfigure}[b]{0.24\textwidth}
			\includegraphics[width=1\textwidth, height= 0.95\textwidth]{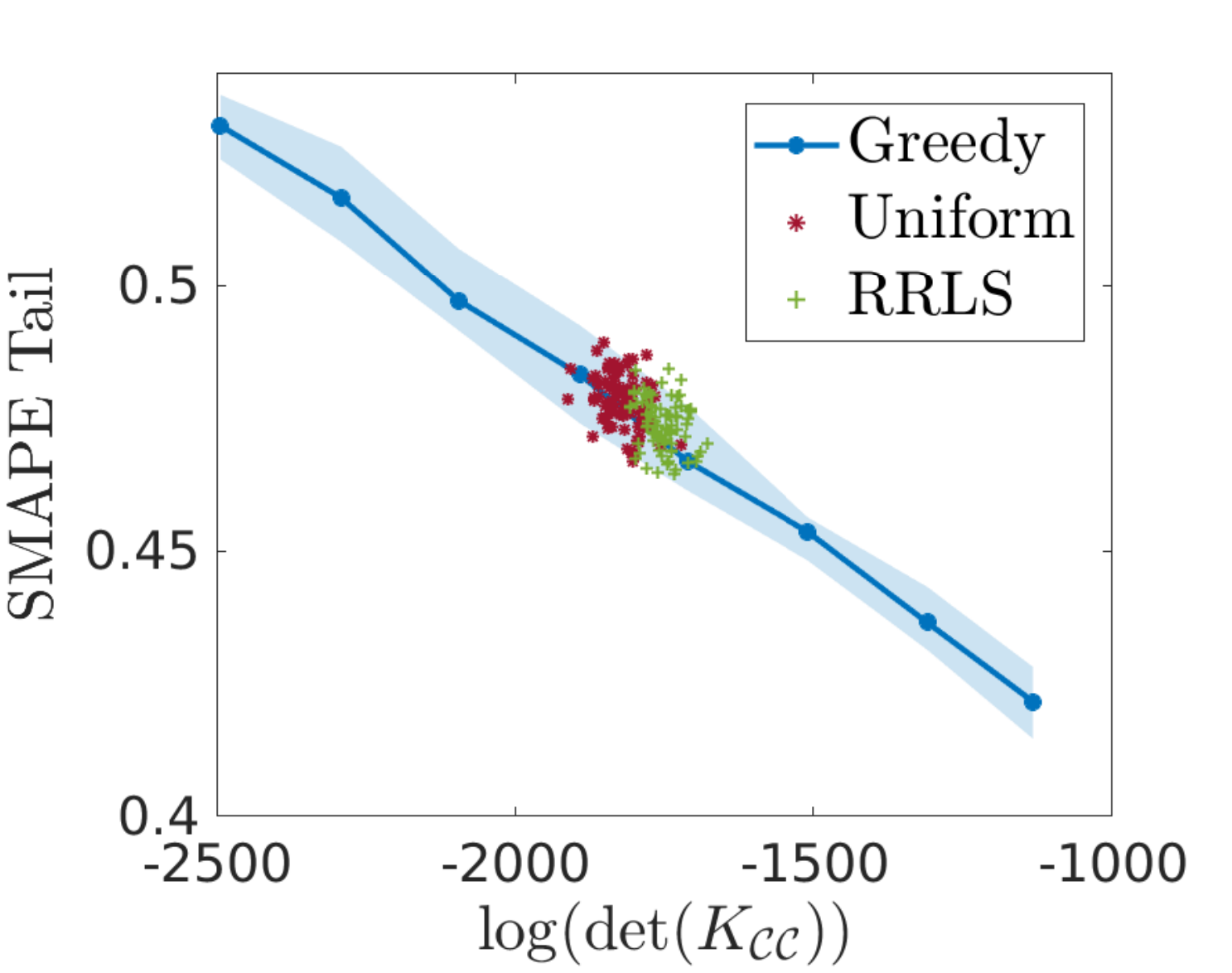}
			\caption{\texttt{Bike S.}: Tail}
		\end{subfigure}
		\begin{subfigure}[b]{0.24\textwidth}
			\includegraphics[width=\textwidth, height= 0.95\textwidth]{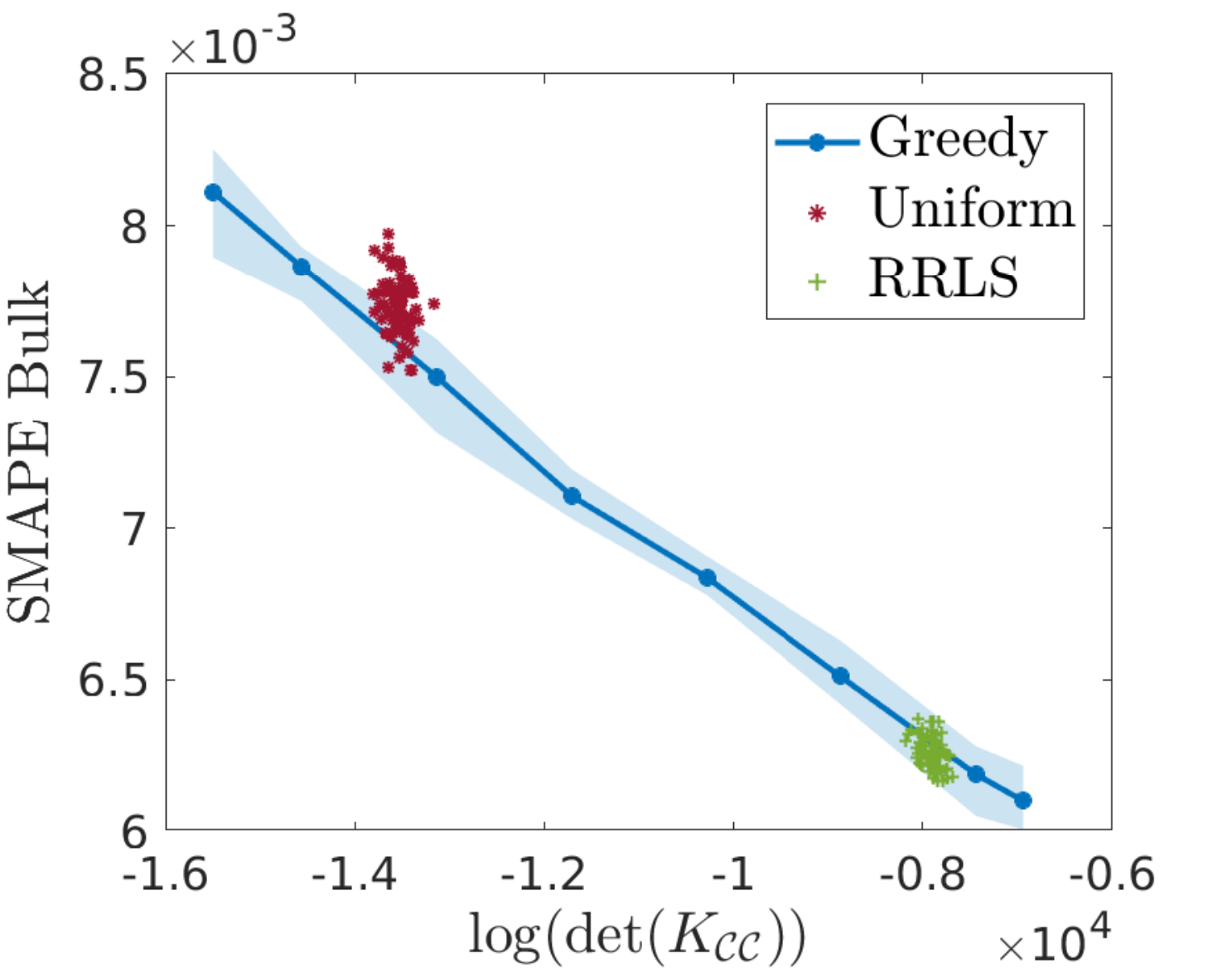}
			\caption{\texttt{Year Pred.}: Bulk}
		\end{subfigure}
		\begin{subfigure}[b]{0.24\textwidth}
			\includegraphics[width=1\textwidth, height= 0.95\textwidth]{Figures/YearPredictionMSD_Perf_Bulk_Reg_Ridge_Div-eps-converted-to.pdf}
			\caption{\texttt{Year Pred.}: Tail}
		\end{subfigure}		
		\caption{Large-scale Regression results. The SMAPE on the bulk and tail of the test set versus the sample diversity. The larger $\mathrm{det}(K_{\mathcal{C}\mathcal{C}})$, the more diverse the subset.}\label{fig:RegressionLS}
	\end{figure}

\paragraph{Preconditioner}
A natural idea is to consider iterative methods to solve the system in \eqref{eq:SmallLS} because of their simplicity and low iteration cost. The speed and accuracy of convergence of the conjugate gradient method depends on the condition number of the linear system~\cite{kershaw1978incomplete}, which makes the use of diverse samplings in combination with iterative methods particularly interesting. This is illustrated on the  \texttt{Parkinson}, and  \texttt{Pumadyn8FM} datasets\footnotemark[1].
The condition number of the preconditioned system is measured, where the preconditionner given in \eqref{eq:precondionerUniform} is used in combination with uniform sampling, RLS and k-DPP in combination with the preconditioner defined in \eqref{eq:precondionerUniform}. The ridge regularization parameter is equal to $\lambda = 10^{-10}$ to illustrate the impact of diversity on badly conditioned systems.
From the results in Figure~\ref{fig:Preconditioner}, we see empirically that sampling a more diverse subset, results in a better conditioning of the linear system. 
\begin{figure}[h]
		\centering
		\begin{subfigure}[t]{0.24\textwidth}
			\includegraphics[width=\textwidth, height= 0.85\textwidth]{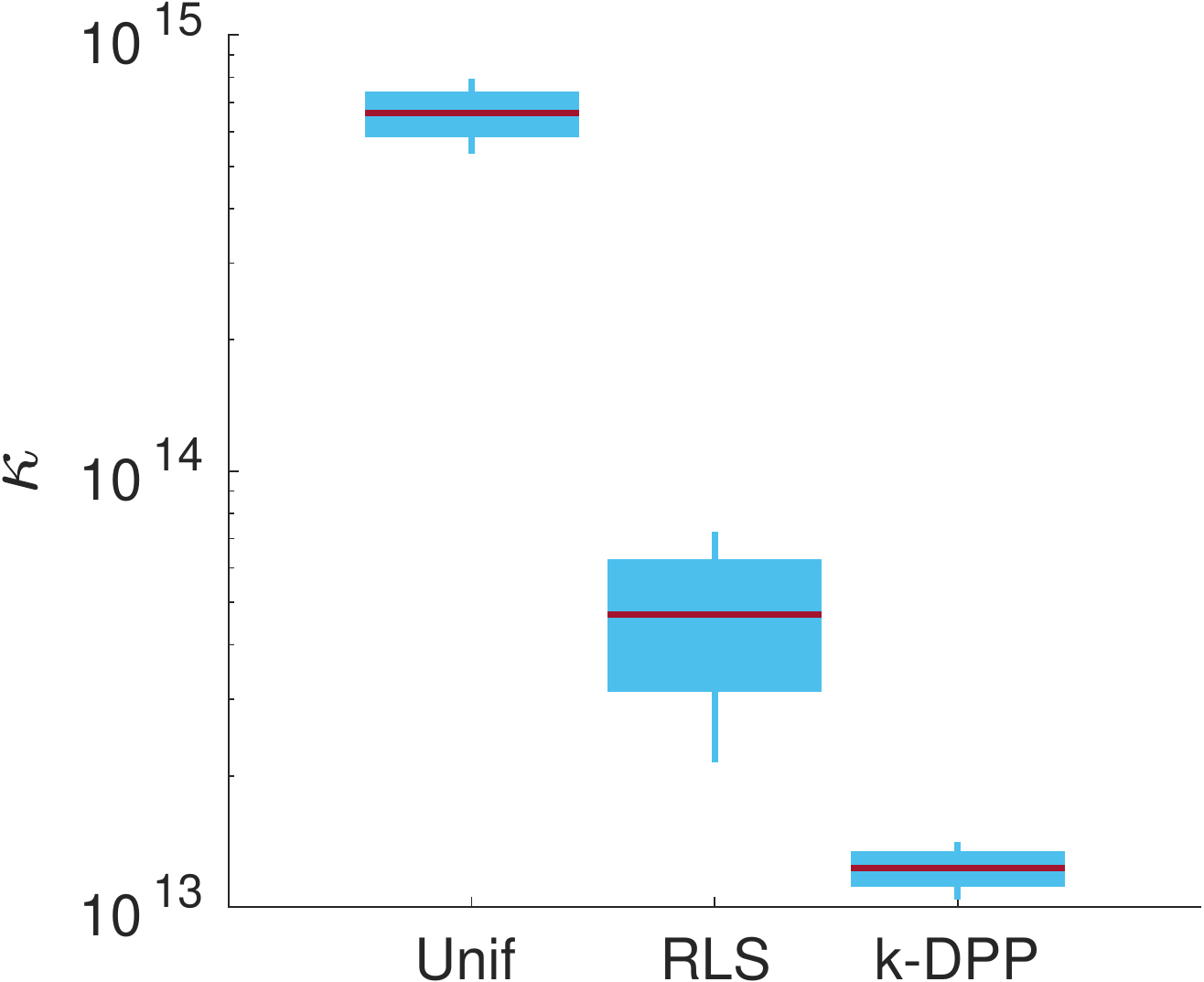}
			\caption{\texttt{Parkinson}}
		\end{subfigure}
		\begin{subfigure}[t]{0.24\textwidth}
			\includegraphics[width=\textwidth, height= 0.91\textwidth]{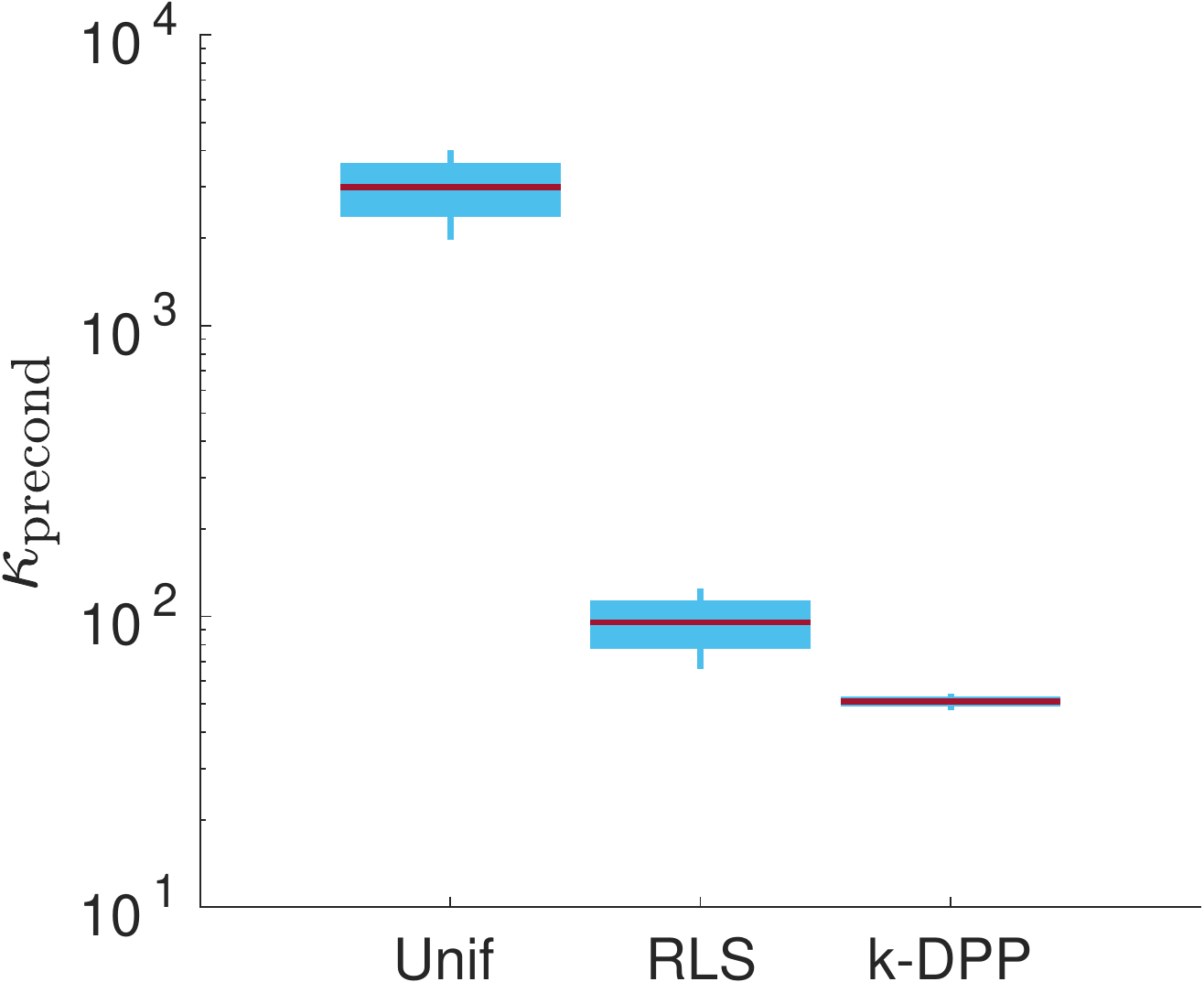}
			\caption{\texttt{Parkinson}}
		\end{subfigure}
		\begin{subfigure}[t]{0.24\textwidth}
			\includegraphics[width=\textwidth, height= 0.86\textwidth]{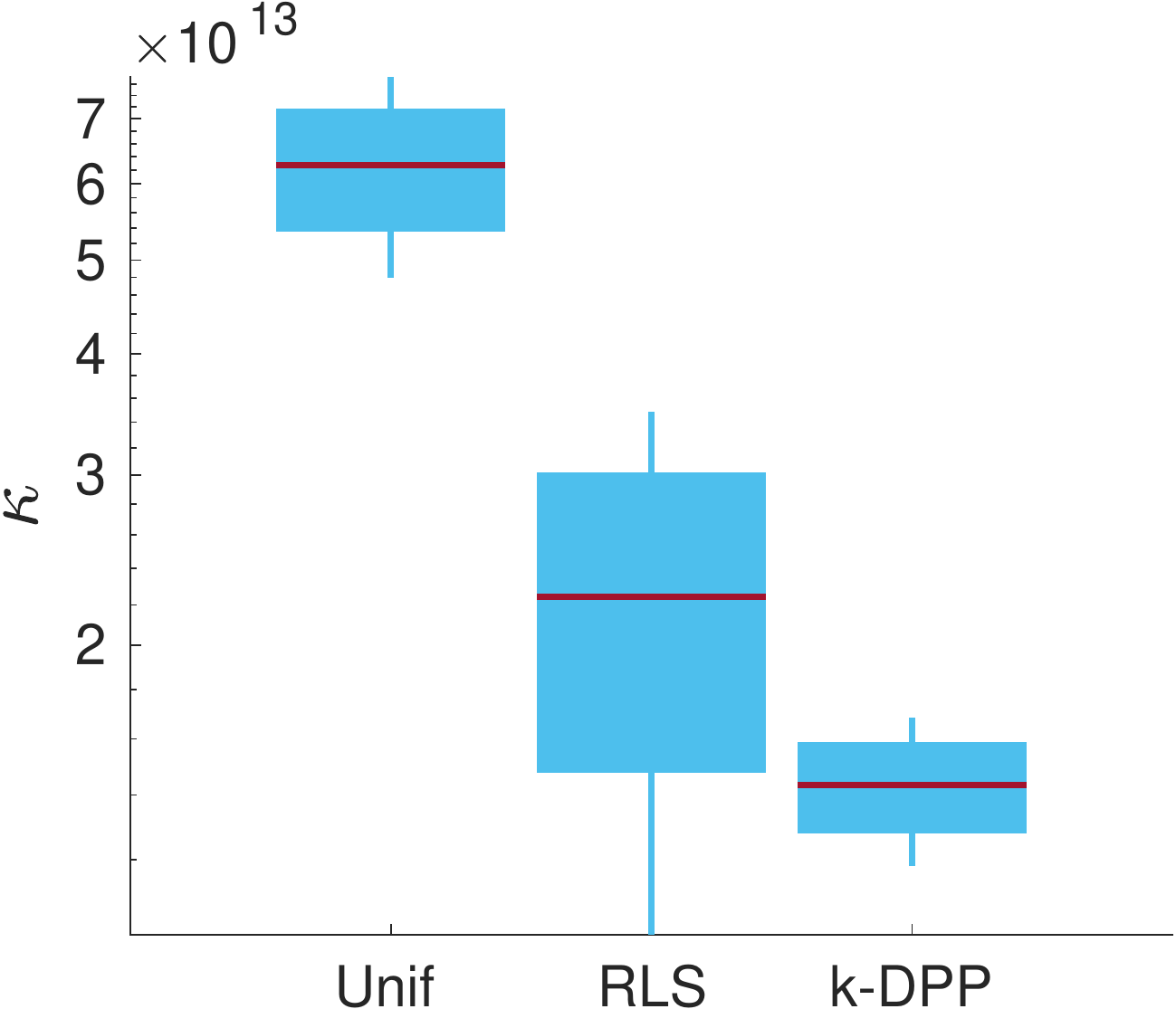}
			\caption{\texttt{Pumadyn8FM}}
		\end{subfigure}	
		\begin{subfigure}[t]{0.24\textwidth}
			\includegraphics[width=\textwidth, height= 0.91\textwidth]{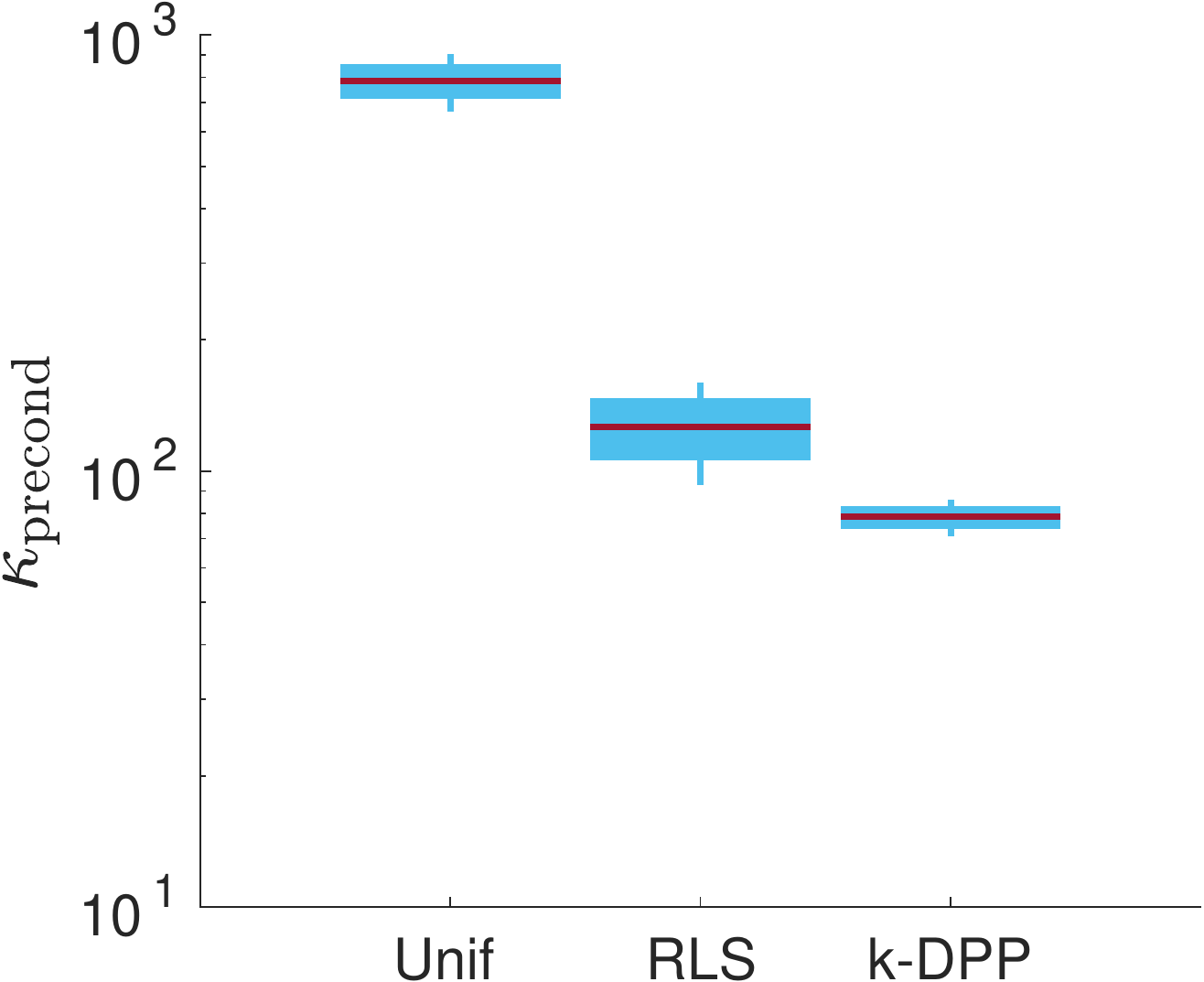}
			\caption{\texttt{Pumadyn8FM}}
		\end{subfigure}			
		\caption{Preconditioning results. The condition number of the linear system before and after the preconditioning is plotted for Uniform, RLS and k-DPP sampling. From left to right, the condition number before and after preconditioning, for \texttt{Parkinson} and \texttt{Pumadyn8FM} datasets, respectively.}\label{fig:Preconditioner}
\end{figure}

\section{Conclusions}
\label{sec:conclusions}

In this paper, the interest of sampling diverse landmarks in the context of Nystr\"om approximation was illustrated. Our empirical findings relating regularization and diversity are partly supported by theoretical results. In the context of Kernel Ridge  Regression, an extra contribution of the paper consists in proposing the use of the Ridge Leverage Score distribution in order to assess the uniformity of the performance of the regressor.
\section*{Acknowledgments}
{\footnotesize EU: The research leading to these results has received funding from the European Research Council under the European Union's Horizon 2020 research and innovation program / ERC Advanced Grant E-DUALITY (787960). This paper reflects only the authors' views and the Union is not liable for any use that may be made of the contained information.Research Council KUL: Optimization frameworks for deep kernel machines C14/18/068. Flemish Government:FWO: projects: GOA4917N (Deep Restricted Kernel Machines: Methods and Foundations), PhD/Postdoc grant.Impulsfonds AI: VR 2019 2203 DOC.0318/1QUATER Kenniscentrum Data en Maatschappij. Ford KU Leuven Research Alliance Project KUL0076 (Stability analysis and performance improvement of deep reinforcement learning algorithms). The computational resources and services used in this work were provided by the VSC (Flemish Supercomputer Center), funded by the Research Foundation - Flanders(FWO) and the Flemish Government – department EWI.}
\appendix
\section{Proofs}

\begin{proof}[Proof of Theorem \ref{Thm1}]Without loss of generality, we put $\alpha =1$.\\
(i)  Then, we first prove $\mathbb{E}_{\mathcal{C}}[C K_{\mathcal{C}\mathcal{C}}^{-1}  C^\top] = (K+\alpha \mathbb{I})^{-1}$.
The matrix inversion lemma yields
\begin{align*}
\det\begin{pmatrix}
K_{\mathcal{C}\mathcal{C}} & C^\top \bm{w}\\
\bm{u}^{\top}C & 1
\end{pmatrix} = \det(K_{\mathcal{C}\mathcal{C}})\left[1 - \bm{u}^{\top}C K_{\mathcal{C}\mathcal{C}}^{-1} C^\top\bm{w}\right] =  \det\left(K_{\mathcal{C}\mathcal{C}}- C^\top\bm{w}\bm{u}^{\top}C\right).
\end{align*}
This simplifies to
\[
\bm{u}^{\top}_{\mathcal{C}} K_{\mathcal{C}\mathcal{C}}^{-1}  \bm{w}_{\mathcal{C}} = 1-\frac{\det(K_{\mathcal{C}\mathcal{C}}-C^\top \bm{w}\bm{u}^{\top}C)}{\det(K_{\mathcal{C}\mathcal{C}})}.
\]
By taking the expectation on both sides, we find
\[
\mathbb{E}_{\mathcal{C}}\left[\bm{u}^{\top}_{\mathcal{C}} K_{\mathcal{C}\mathcal{C}}^{-1}  \bm{w}_{\mathcal{C}}\right] = 1-\frac{\det(\mathbb{I}+K- \bm{w}\bm{u}^{\top})}{\det(\mathbb{I}+K)} = \bm{u}^{\top}(\mathbb{I}+K)^{-1}\bm{w},
\]
where we used $\sum_{\mathcal{C}\subseteq [n]}A_{\mathcal{C}\mathcal{C}} = \mathbb{I} + A$ with $A$ a square matrix.\\
(ii) Secondly, we prove that  $\mathbb{E}_{\mathcal{C}}[\bm{u}^\top C C^\top K  CC^\top \bm{v}]= P_{(2)}\circ K$. We first calculate  by using the marginal kernel
\[
\mathbb{E}_{\mathcal{C}}[1_{i\in \mathcal{C}}1_{j\in \mathcal{C}}] = 1_{i=j}\times P_{ii} + 1_{i\neq j}\times\det P_{\{i,j\}\{i,j\}} = 1_{i=j}\times P_{ii} + (P_{ii}P_{jj}-P_{ij}^2).
\]
The result follows from the following formula
\[
\mathbb{E}_{\mathcal{C}}[\bm{u}^\top C C^\top K  CC^\top \bm{v}] = \sum_{i=1}^n\sum_{j=1}^n \mathbb{E}_{\mathcal{C}}[u_i 1_{i\in \mathcal{C}} K_{ij}1_{j\in \mathcal{C}} v_j],
\]
and by using the linearity of the expectation.
\end{proof}

\begin{proof}[Proof of Corollary \ref{corol2}] We follow the proof strategy of~\cite{{Fastdpp}}.

Let $f(\mathcal{C}) = \bm{w}^{\top}_{\mathcal{C}} K_{\mathcal{C}\mathcal{C}}^{-1}  \bm{w}_{\mathcal{C}}$ a function viewed as $f:\{0,1\}^n \to \mathbb{R}$.
We quote now a simplified result from~\cite{pemantle2014concentration} related to strong Rayleigh measures which generalize in particular DPPs.
\begin{theorem}[Thm 3.2 in cite]
Let $\mathbb{P}$ be strong Rayleigh. Let $f:\{0,1\}^n \to \mathbb{R}$ be $1-$Lipschitz with respect to the Hamming distance. Then, 
\[
\mathbb{P}(|f-\mathbb{E}f|>  a)\leq 5 \exp (-a^2/(48n)).
\]
\end{theorem}
Let $\delta\geq 5 \exp (-a^2/(48n))$.  This means $ a\geq (48 n\log(5/\delta))^{1/2}$. Then, with probability less than $\delta$, 
\[
|f-\mathbb{E}f|> \ell a\geq \ell \sqrt{48 n\log(5/\delta)}.
\]
We now need to calculate an upper bound on the Lipschitz constant of $f$ with respect to the Hamming distance. To do so, it is sufficient to consider two binary vectors  $C$ and $C'\in\{0,1\}^n$  differing of one digit. Say that $f(C)\geq f(C')\geq 0$.
Then, we have
\[
|f(C)-f(C')|\leq f(C)- f(C')\leq \bm{w}^{\top}_{\mathcal{C}} K_{\mathcal{C}\mathcal{C}}^{-1}  \bm{w}_{\mathcal{C}}\leq 1/\lambda_{\min}(K) = d_H(C,C')/\lambda_{\min}(K),
\]
since $\|w\|_2 = 1$ and where $d_H(C,C')$ is the Hamming distance between the binary vectors $C$ and $C'$. The final result follows by taking $\ell =1/\lambda_{\min}(K)$.
\end{proof}

\begin{proof}[Proof of Theorem \ref{thm:projection-cost}]
Since $X$ is a orthogonal projector and $L_{\mathcal{C}}\preceq K$, we have $\Tr(XL_{\mathcal{C}}X)\leq \Tr(XKX)$, which yields
\[
\Tr(K-XKX) \leq \Tr(L_{\mathcal{C}}-XL_{\mathcal{C}}X) + \Tr(K-L_{\mathcal{C}}) = \Tr(K)-\Tr(XL_{\mathcal{C}}X).
\]
By taking the expectation over $\mathcal{C}\sim DPP(K/\alpha)$ on both sides of the above inequality, and by using $\mathbb{E}_{\mathcal{C}}(L_{\mathcal{C}}) = K-\alpha P$, we obtain
\[
\Tr(K-XKX) \leq \mathbb{E}_{\mathcal{C}}[\Tr(L_{\mathcal{C}}-XL_{\mathcal{C}}X)] + \alpha d_{\rm eff}(K/\alpha) = \Tr(K-XKX)+ \alpha \Tr(XPX).
\]
Finally, by using that $X$ is a projector, it holds that $\Tr(XPX)\leq \Tr(X) = k$ since $P\preceq \mathbb{I}$. Also, we have $\Tr(XPX)\leq \Tr(P) = d_{\rm eff}(K/\alpha)$, so that the final bound is obtained.
\end{proof}
\begin{proof}[Proof of Theorem \ref{thm:RiskKRR}]
The risk  $\mathcal{R}(\hat{z}) = \frac{1}{n} \mathbb{E}_\epsilon \|\hat{z} - z\|_2^2$ is decomposed in terms of bias and variance $\mathcal{R}(\hat{z}) =  \text{bias}^2(K) + \text{var}(K)$, where the bias reads
\[
\text{bias}(K) = \sqrt{n\gamma^2 z^\top (K+n\gamma\mathbb{I})^{-2}z}
\]
and the variance is
\[
\text{var}(K) = \frac{\sigma^2}{n}\Tr[K^2(K+n\gamma\mathbb{I})^{-2}].
\]
Firstly, since $L_{\mathcal{C}}\preceq K$, it holds that $\lambda_\ell (L_{\mathcal{C}})\leq \lambda_\ell (K)$ for all $\ell$ and $\text{var}(L_{\mathcal{C}})\preceq \text{var}(K)$.
Then, we can give an upper bound for $\text{bias}(L_{\mathcal{C}})$. The $2$-norm submultiplicativity gives
\begin{align*}
\| (L_{\mathcal{C}} + n\gamma\mathbb{I})^{-1}z - (K + n\gamma\mathbb{I})^{-1}z \|_2 &\leq \| (L_{\mathcal{C}} + n\gamma\mathbb{I})^{-1}(K-L_{\mathcal{C}})\|_2  \|(K + n\gamma\mathbb{I})^{-1}z \|_2\\ &\leq \frac{\Tr(K-L_{\mathcal{C}}) }{n\gamma} \|(K + n\gamma\mathbb{I})^{-1}z \|_2.
\end{align*}
Hence, by using the triangle inequality and the bound hereabove, it holds that 
\begin{align*}
\sqrt{z^\top (L_{\mathcal{C}}+n\gamma\mathbb{I})^{-2}z}&\leq \|(K + n\gamma\mathbb{I})^{-1}z \|_2 + \| (L_{\mathcal{C}} + n\gamma\mathbb{I})^{-1}z - (K + n\gamma\mathbb{I})^{-1}z \|_2\\
&\leq (1+\frac{\Tr(K-L_{\mathcal{C}}) }{n\gamma})\sqrt{z^\top (K+n\gamma\mathbb{I})^{-2}z}.
\end{align*}
The result follows by taking the expectation over $\mathcal{C}\sim DPP_{L}(K/\alpha)$.
\end{proof}

\begin{proof}[Proof of Lemma~\ref{Lem:Estimator}]
This results is a direct consequence of the two following identities:
$
\mathbb{E}_{\mathcal{C}}[1_{i\in \mathcal{C}}] =  P_{ii},
$
and 
\[
\mathbb{E}_{\mathcal{C}}[1_{i\in \mathcal{C}}1_{j\in \mathcal{C}}] = 1_{i=j}\times P_{ii} + (P_{ii}P_{jj}-P_{ij}^2).
\]
\end{proof}
\subsection{Parameters and dataset descriptions}
The parameters and datasets used in the simulations can be found in Table~\ref{Table:data}. When a subset is sampled from $k$-DPP, the number of landmarks is fixed to $k = d_{\mathrm{eff}}(n\lambda)$, where the effective dimension corresponds to the expected subset size for $DPP_L(K/(n\lambda))$.
\begin{table}[h]
	\caption{Datasets and parameters used for the experiments on the Nystr\"om approximation.}
	\label{Table:data}
        \begin{center}
                \begin{tabular}{rccccccc}
                        \toprule
                        Dataset & Task &$n$ & $d$  & $\sigma$ & $\lambda$ & $d_{\mathrm{eff}}(n\lambda)$ & $n_{RRLS}$ \\ \midrule
                        \texttt{Housing}& Kernel approx. & $506$ & $13$  & $5$ & $10^{-6}$& $186$ & / \\ 
 \texttt{MiniBooNE}& Kernel approx. & $130065$ & $50$  & $8$ & $10^{-6}$& $462$ & $8000$ \\
 \texttt{codRNA}& Kernel approx. & $331152$ & $8$  & $8$ & $10^{-6}$& $1204$ & $8000$ \\ \midrule
			\texttt{B. Cancer}& KPCA & $569$ & $30$  & $ 10$ & $10^{-6}$& $158$ & /\\
			\texttt{A. Credit}& KPCA & $690$ & $14 $ & $5$ &$10^{-6}$& $371$  & /   \\
\texttt{Adult}& KPCA & $48842$ & $110 $ & $8$ &$10^{-6}$& $1202$  & $8000$   \\
\texttt{Covertype}& KPCA & $581012$ & $54 $ & $20$ &$10^{-6}$& $3665$  & $10000$   \\			
			 \midrule
 			\texttt{Abalone}& KRR & $4177$ &$ 8 $ & $1$ & $10^{-4}$& $294$ & / \\
 			\texttt{Wine Quality}& KRR & $6497$ & $11 $ &$2$ & $10^{-4}$& $555$ & /  \\ 
	\texttt{Bike S.}& KRR & $17389$ & $16 $ &$1$ & $10^{-4}$& $2732$ & $8000$ \\  			
 	\texttt{Year Pred.}& KRR & $515345$ & $90 $ &$10$ & $10^{-4}$& $4217$ & $10000$ \\ 			
 			\midrule
 			\texttt{Parkinson}& Prec. & $5875$ & $20 $ & $5$ & $10^{-6}$& $738$ & / \\
 			\texttt{Pumadyn8FM}& Prec. & $8192$ & $25 $ &$5$ & $10^{-6}$& $296$ & / \\ \bottomrule
 		\end{tabular}
 	\end{center}
 \end{table}

\section{Supplementary Material}
\paragraph{Setting}
In the performance plots displayed in the sequel, i.e., Figures~\ref{fig:ClusteringSupp},  \ref{fig:NystromSupp}, \ref{fig:LSSupp}, \ref{fig:KPCASupp50}, \ref{fig:RegressionSupp}, \ref{fig:RegressionLSSupp} and \ref{fig:NystromLSSupp}, the results are plotted on a logarithmic scale, averaged over 10 trials and the errobars show the $0.05$ and $0.95$ quantile. Recall that the larger the $\mathrm{log}(\mathrm{det}(K_{\mathcal{C}\mathcal{C}}))$, the more diverse the subset. The computer used for the small-scale simulations has 8 processors 3.40GHz and 15.5 GB of RAM. Large scale experiments on \texttt{Covertype} and \texttt{Year PredictionMSD} were done on the Vlaams Super Computer (VSC). The implementation of the algorithms is done with matlabR2018b.

\subsection{Additional case study: Clustering}
The performance of diverse kernel approximation methods is evaluated for a clustering task using kernel k-means with Nystr\"om approximation~\cite{wang2019scalable}. Samples with different diversities are sampled, afterwards \emph{Algorithm 2} of Wang et al.~\cite{wang2019scalable} is used to cluster the dataset, where the target dimension $s = k$ is always equal to the number of desired clusters.
The clustering performance is evaluated by the normalized mutual information (NMI) \cite{strehl2002cluster}, the NMI gives a value between 0 and 1, where 1 represents perfect correlation between the ground truth and the clustering outcome. We first illustrate the effect of using a diverse sampling on a toy example. In Figure \ref{fig:ToyExampleClusteringSupp}, we show a very imbalanced dataset consisting of 5 Gaussian bumps with a different number of points and different variances. Uniform sampling often only selects landmarks from 3 out of the 5 clusters, whereas a $k$-DPP samples landmarks out of every cluster. Consequently, kernel k-means algorithm with Nystr\"om approximation can be improved by using diverse sampling. It is important to note that the superior performance is due to the histogram of ridge-leverage scores with $\gamma = 10^{-4}$ (measure of outlyingness) having a long tail (cfr. Figure~\ref{fig:EIG_Regression}) together with the different clusters being heavily imbalanced. 
Next, we demonstrate the effect of diverse sampling on the \texttt{Glass}, \texttt{Breast Cancer} and \texttt{Australian Credit} datasets\footnotemark[1] of Table~\ref{Table:ClusteringSupp}.
\begin{table}[h]				
\centering
\caption{Information on the datasets and parameters used in the clustering experiments.}
				\label{Table:ClusteringSupp}
				\begin{tabular}{rcccccc}
					\toprule
					Dataset & $n$ & $d$  & $\sigma$ & $\lambda$ & $d_{\mathrm{eff}}(n\lambda)$ & \# Clusters \\ \midrule
					\texttt{Glass} & $214$ & $9$ & $9$ & $10^{-4}$& $24$ & $7$ \\				
					\texttt{Cancer} & $569$ & $30$  &$ 3$ & $10^{-4}$& $363$  & $2$ \\
					\texttt{Credit} & $690$ & $14 $ & $2$ &$10^{-4}$& $456$  & $2$  \\ \bottomrule
				\end{tabular}
\end{table}		
The condition number, smallest/largest eigenvalues of $K_{\mathcal{C}\mathcal{C}}$ and NMI is plotted as a function of the determinant. The averaged results are visualized in Figure~\ref{fig:ClusteringSupp}. Information on the datasets and hyperparameters used for the experiments is given in Table \ref{Table:ClusteringSupp}. Sampling a diverse subset gives a more accurate clustering. 
Similar as for the kernel approximation experiments, in the presence of outliers, taking samples with an extremely large $\mathrm{det}(K_{\mathcal{C}\mathcal{C}})$ thanks to the Greedy Swapping Algorithm might decrease the accuracy as for the \texttt{Glass} dataset.
\begin{figure}[b]
	\centering
	\begin{subfigure}[t]{0.45\textwidth}
		\includegraphics[width=\textwidth, height= 0.8\textwidth]{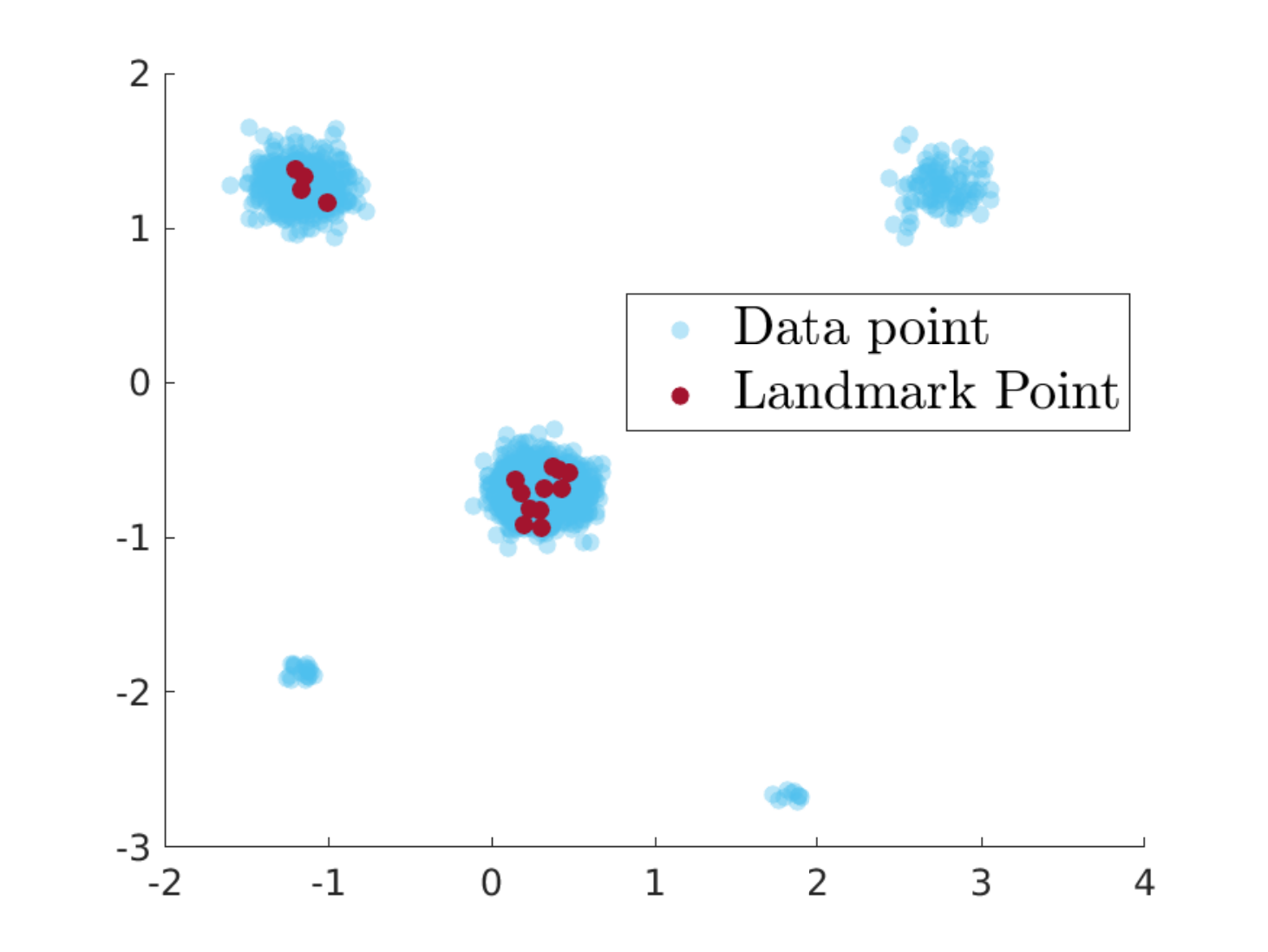}
		\caption{Uniform sampling}
	\end{subfigure}
	\begin{subfigure}[t]{0.45\textwidth}
		\includegraphics[width=\textwidth, height= 0.8\textwidth]{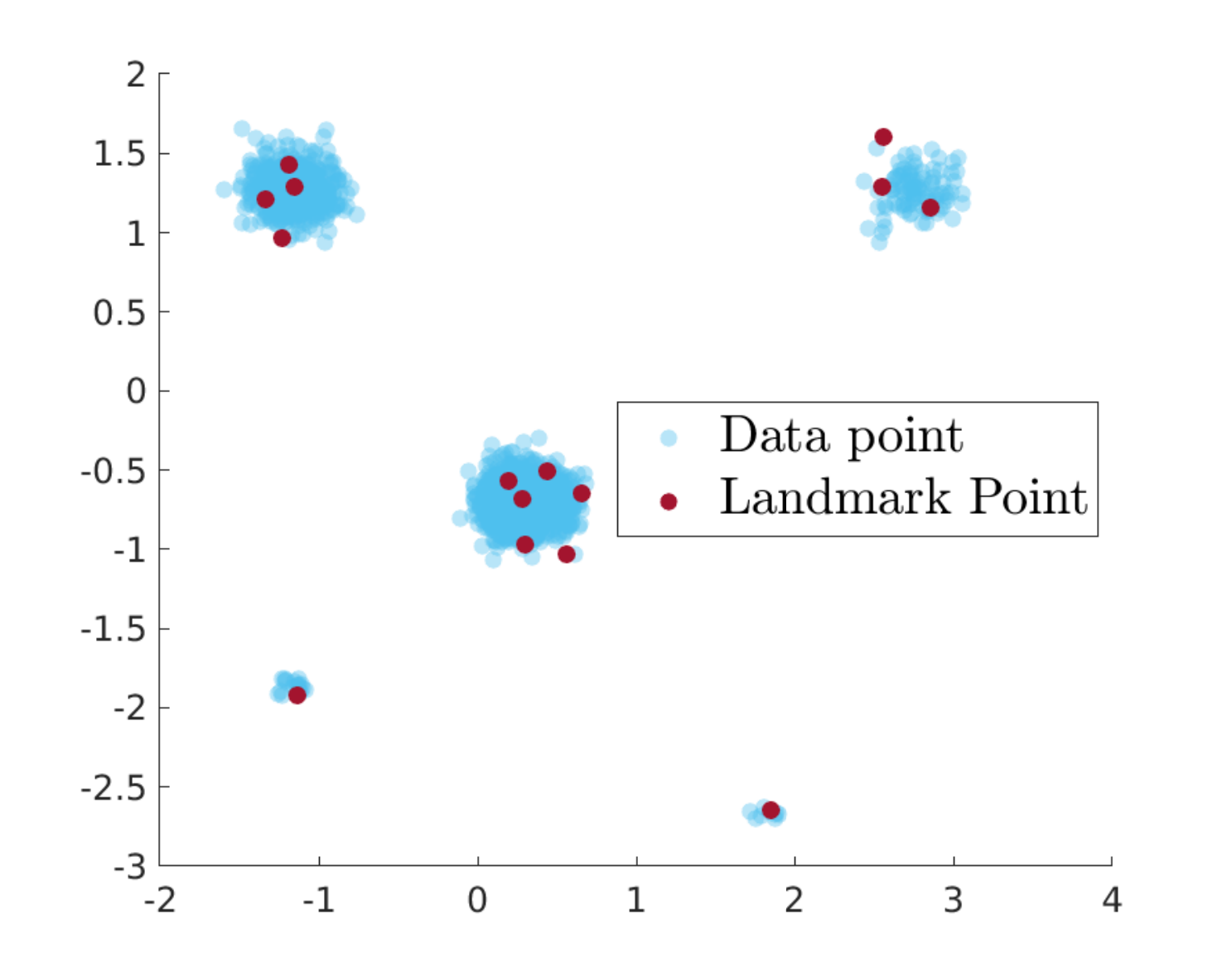}
		\caption{k-DPP sampling}
	\end{subfigure}
	\begin{subfigure}[t]{0.45\textwidth}
		\includegraphics[width=\textwidth, height= 0.8\textwidth]{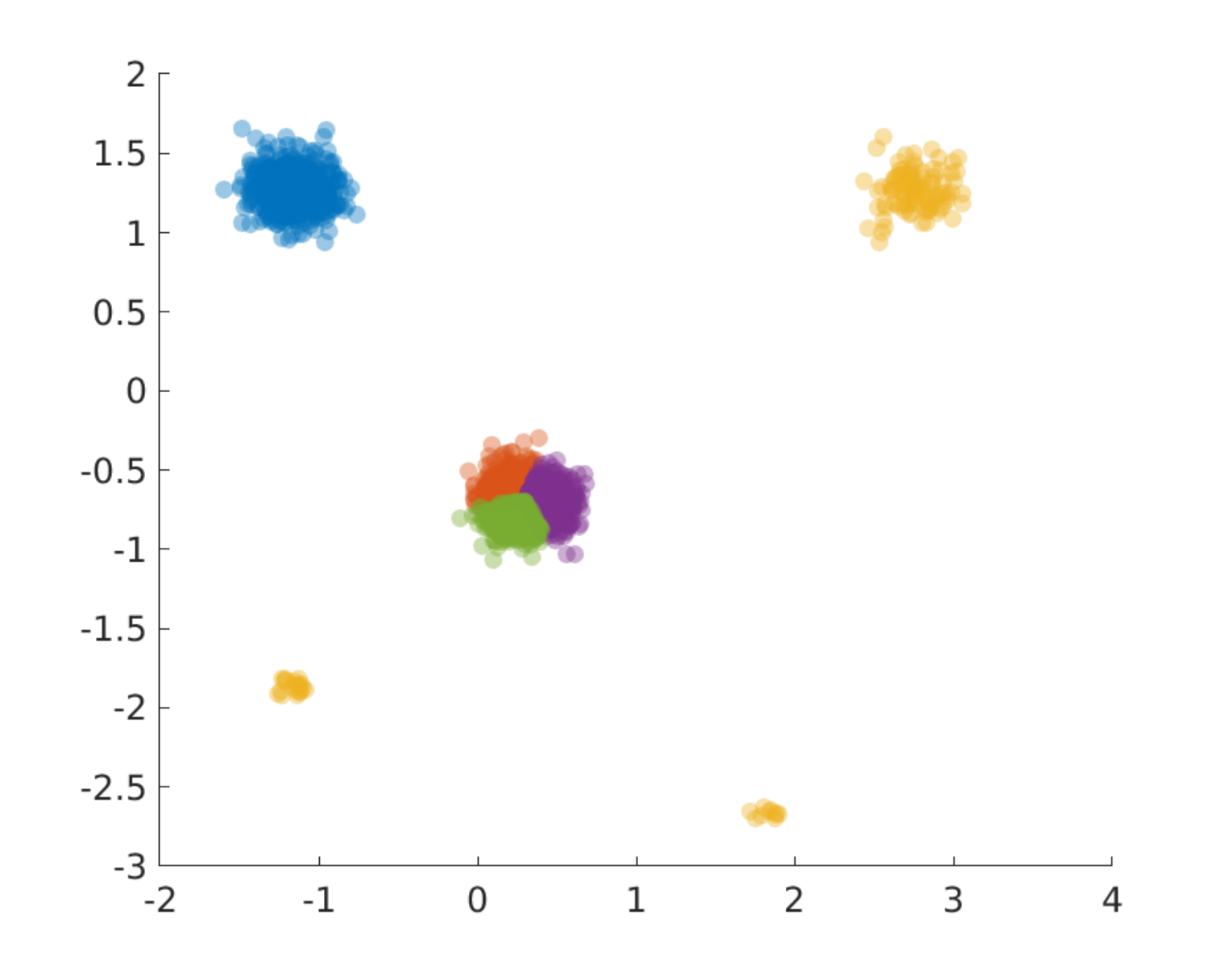}
		\caption{Clustering: Uniform}
	\end{subfigure}
	\begin{subfigure}[t]{0.45\textwidth}
		\includegraphics[width=\textwidth, height= 0.8\textwidth]{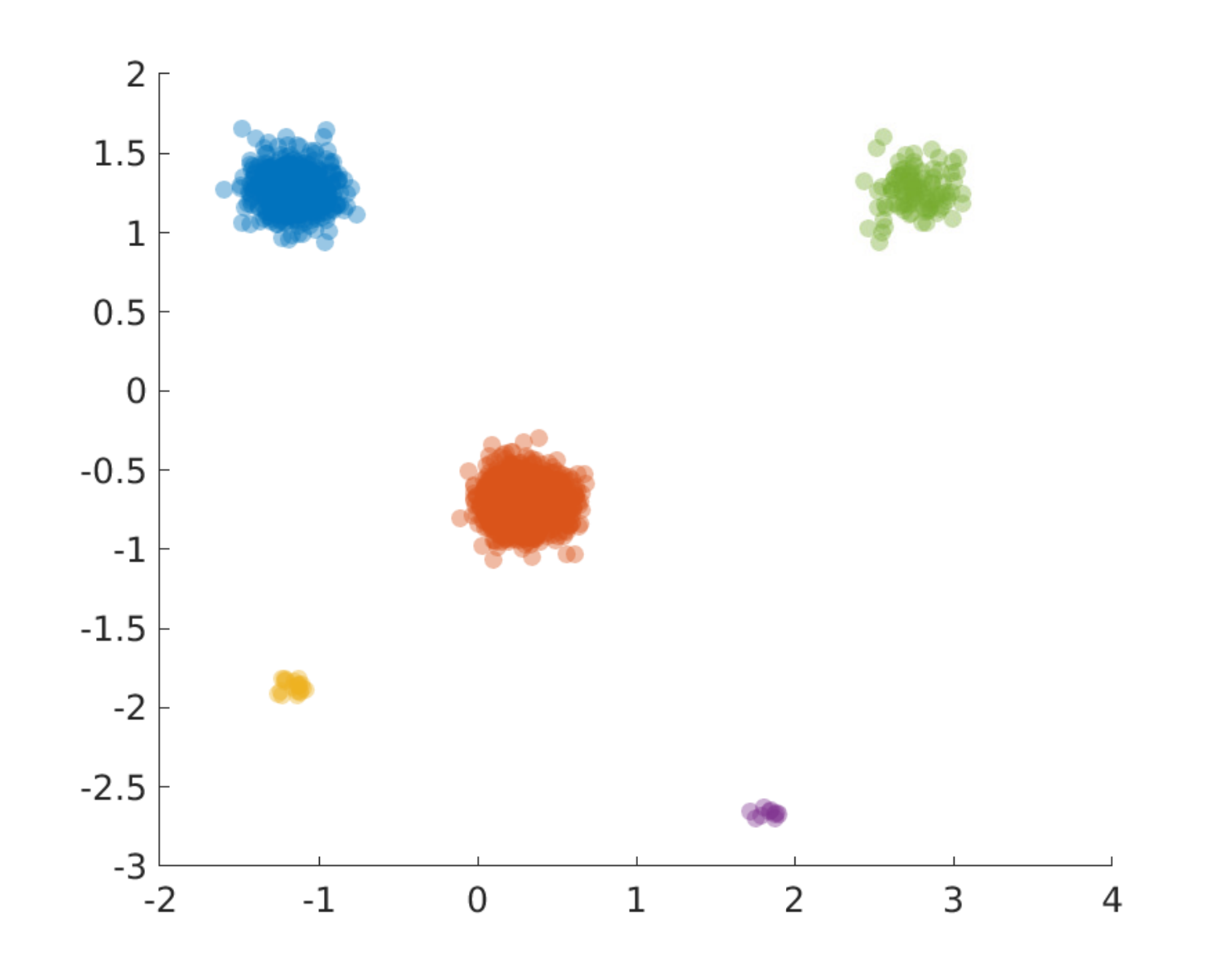}
		\caption{Clustering: k-DPP}
	\end{subfigure}		
	\caption{Illustration of sampling methods on an artificial clustering problem. Uniform sampling oversamples dense parts, and does not select landmark points in every cluster. $k$-DPP sampling overcomes this limitation, resulting in better clustering performance.}\label{fig:ToyExampleClusteringSupp}
\end{figure}
\begin{figure}[h]
	\centering
	\begin{subfigure}[t]{0.31\textwidth}
		\includegraphics[width=\textwidth, height= 0.8\textwidth]{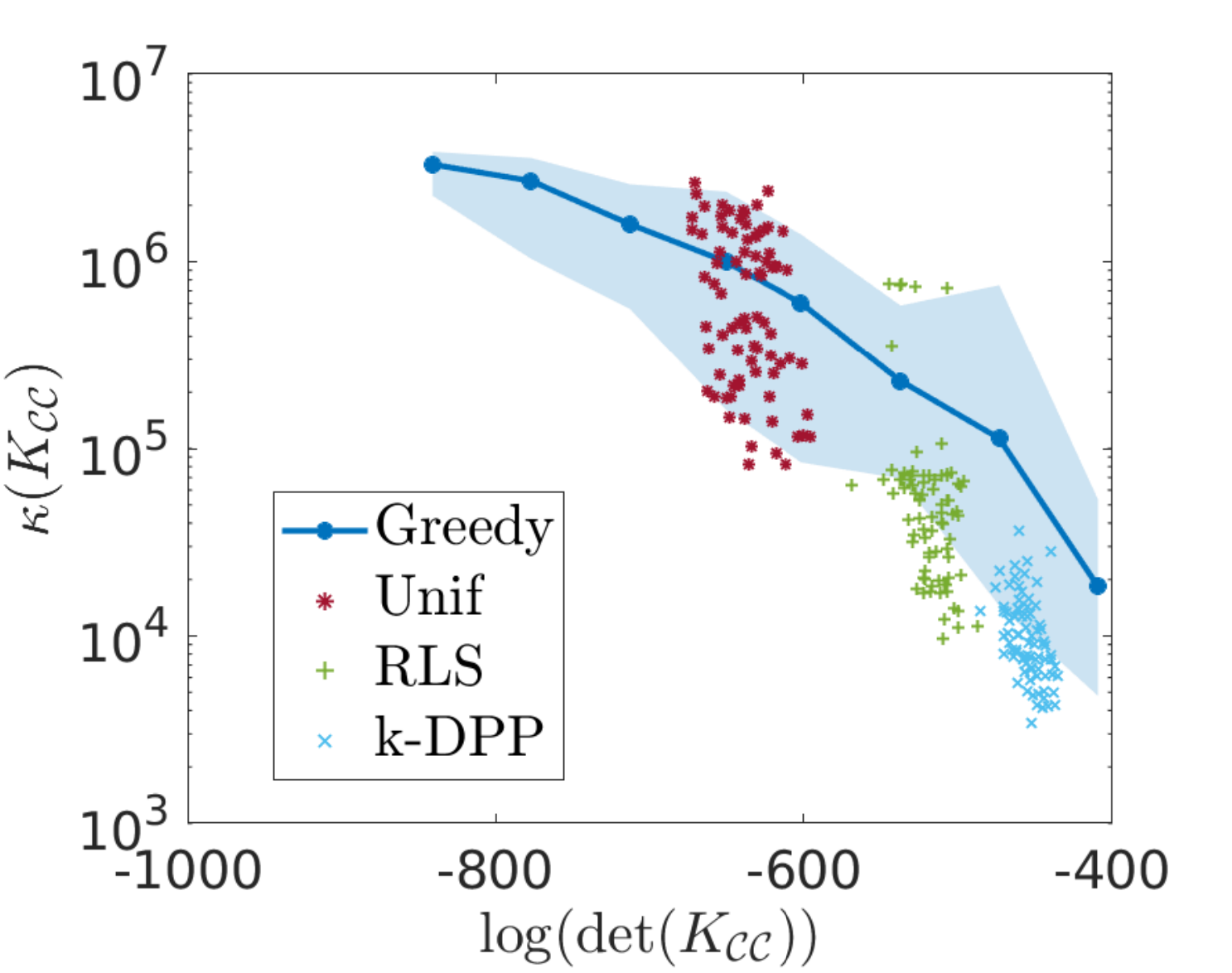}
		\caption{\texttt{A. Credit}: $\kappa(K_{\mathcal{C}\mathcal{C}})$}
	\end{subfigure}
	\begin{subfigure}[t]{0.31\textwidth}
		\includegraphics[width=\textwidth, height= 0.8\textwidth]{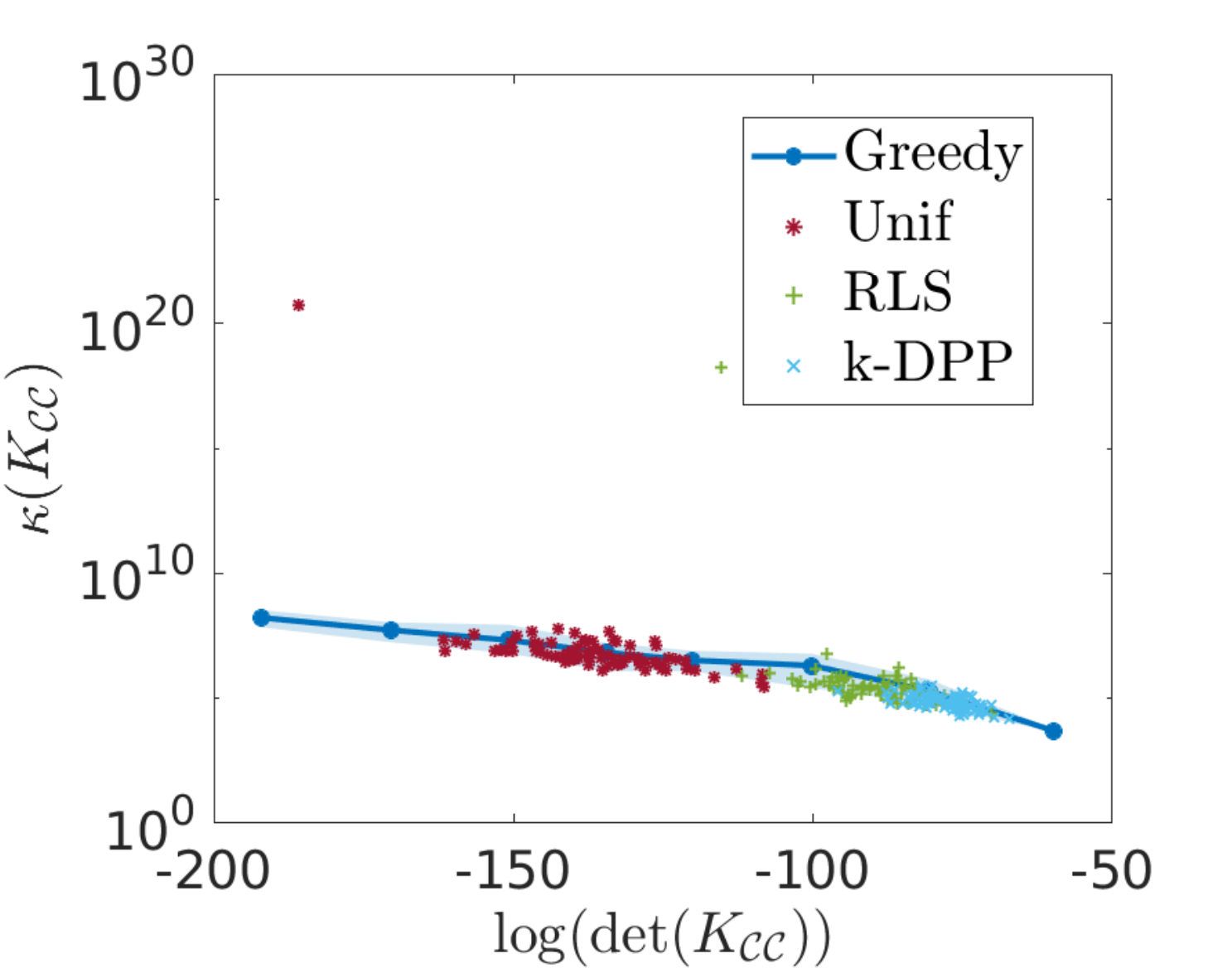}
		\caption{\texttt{Glass}: $\kappa(K_{\mathcal{C}\mathcal{C}})$}
	\end{subfigure}
	\begin{subfigure}[t]{0.31\textwidth}
		\includegraphics[width=\textwidth, height= 0.8\textwidth]{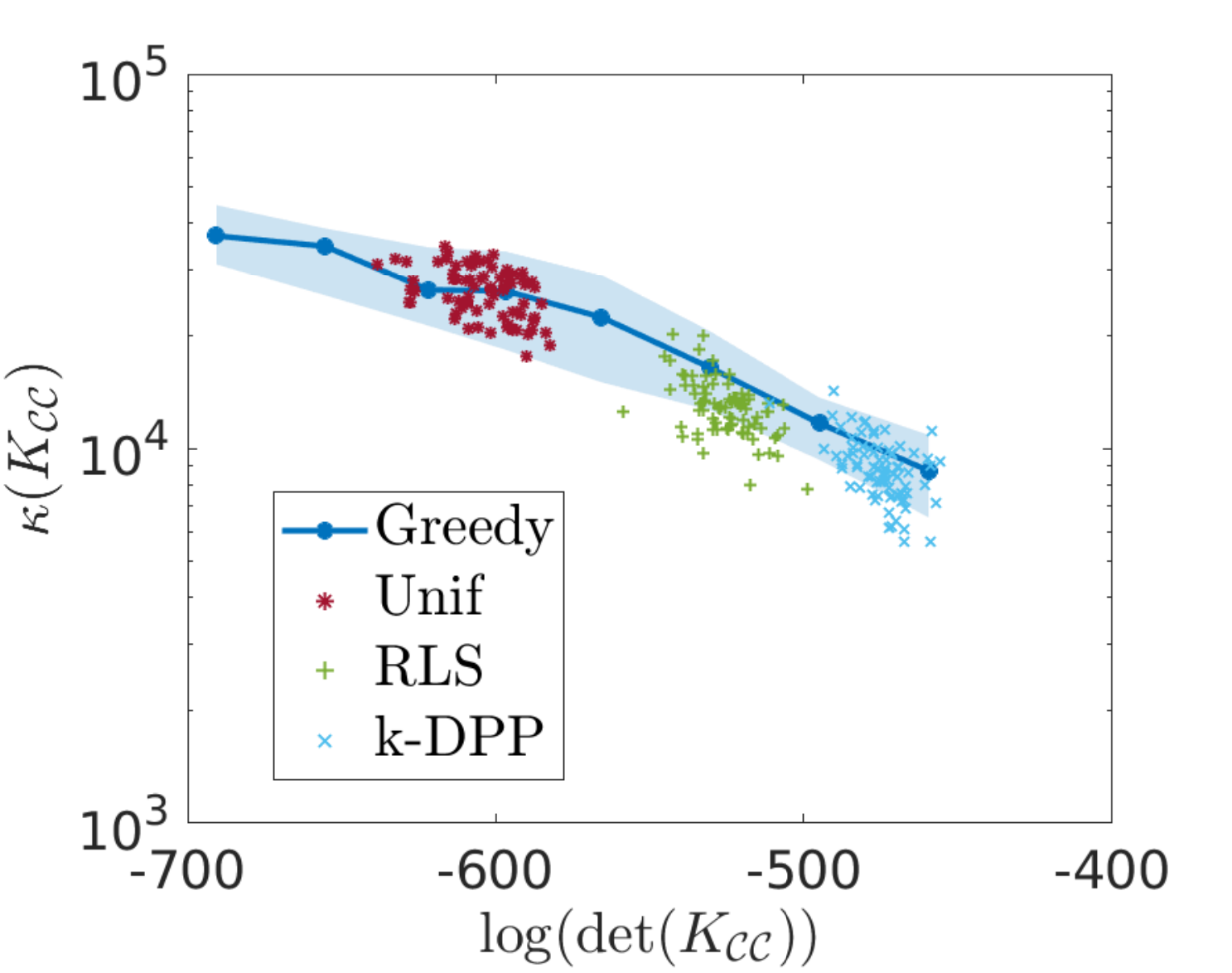}
		\caption{\texttt{B. Cancer}: $\kappa(K_{\mathcal{C}\mathcal{C}})$}
	\end{subfigure}
		\begin{subfigure}[t]{0.31\textwidth}
			\includegraphics[width=\textwidth, height= 0.8\textwidth]{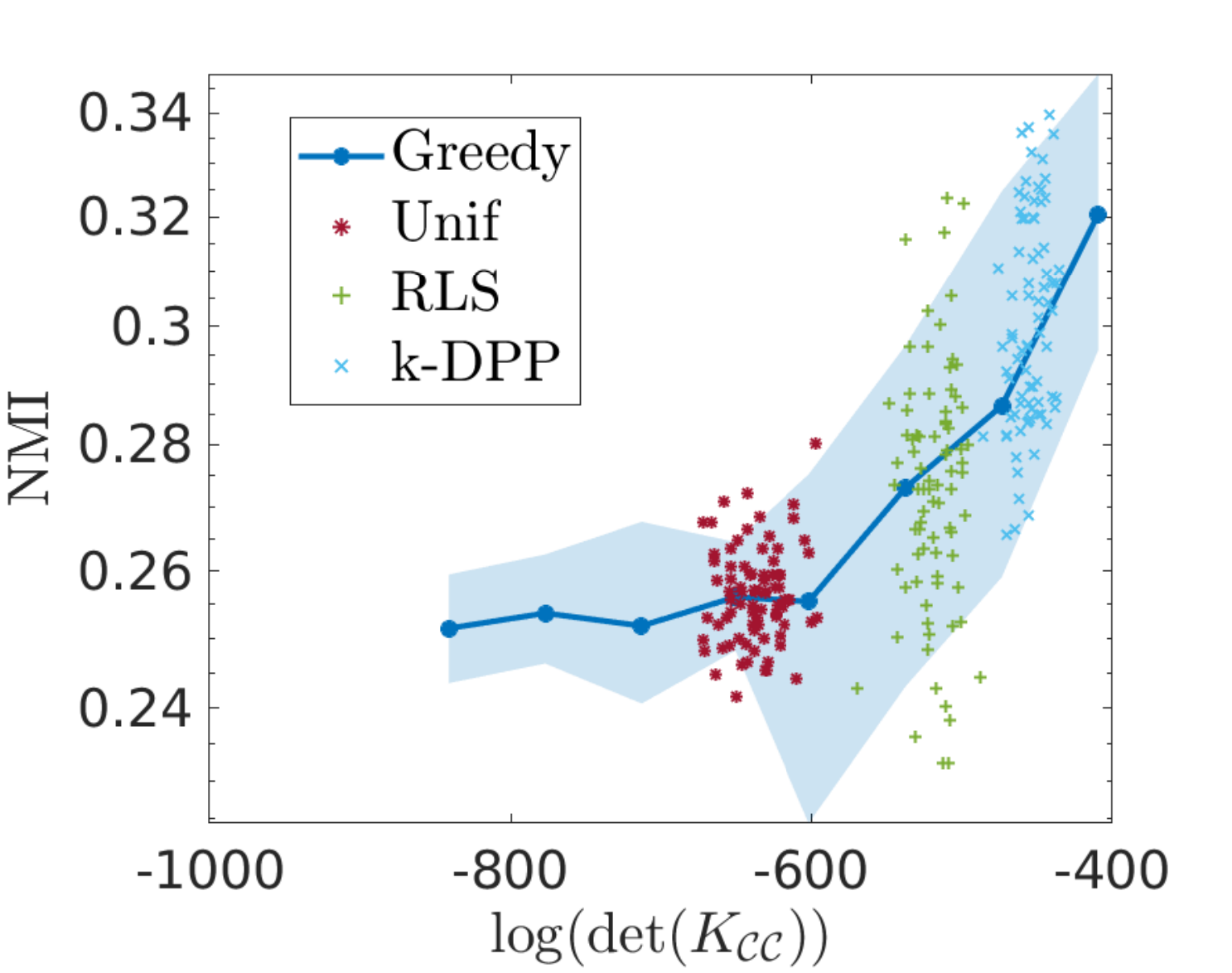}
			\caption{\texttt{A. Credit}: NMI}
		\end{subfigure}
			\begin{subfigure}[t]{0.31\textwidth}
				\includegraphics[width=\textwidth, height= 0.8\textwidth]{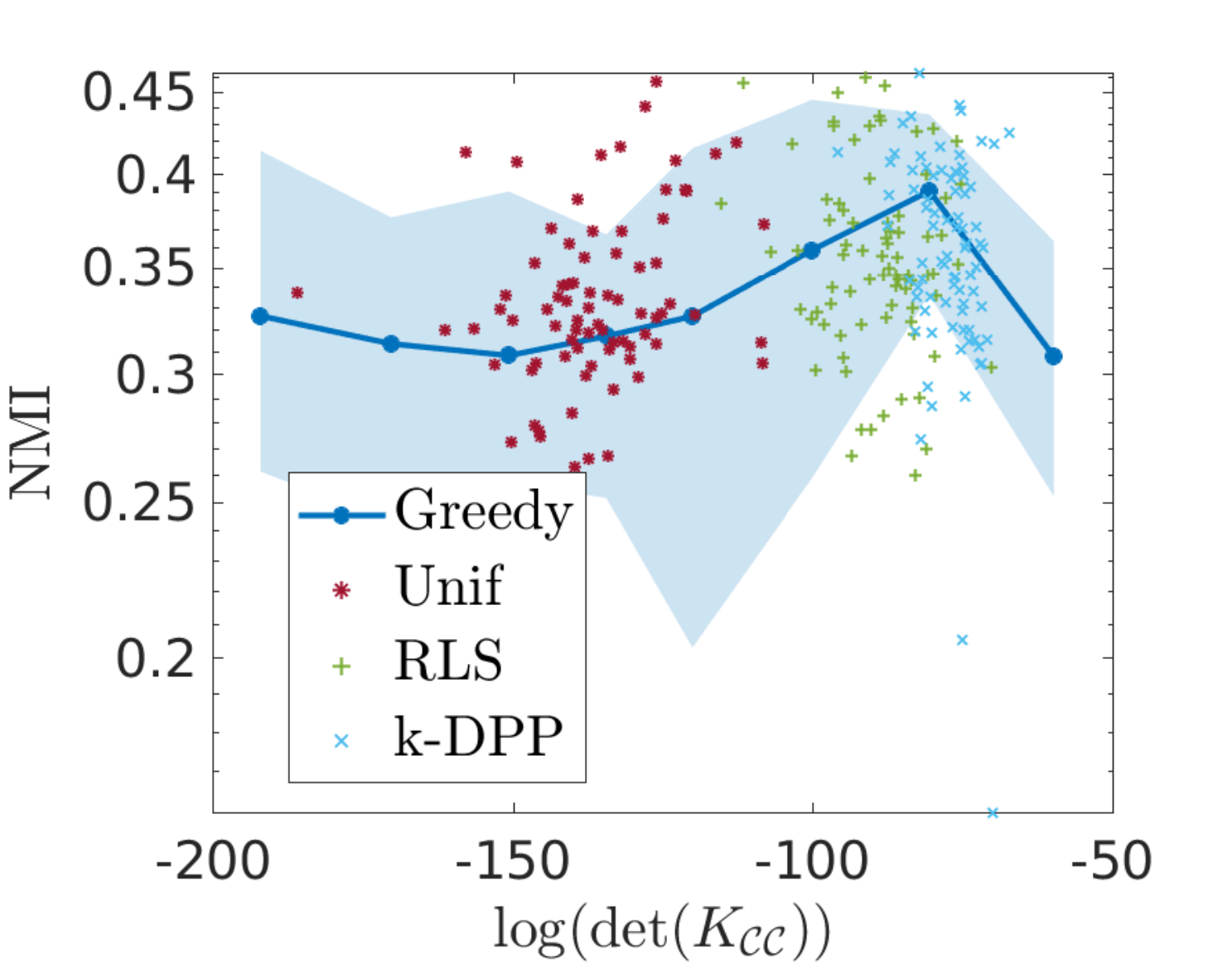}
				\caption{\texttt{Glass}: NMI}
			\end{subfigure}	
	\begin{subfigure}[t]{0.31\textwidth}
		\includegraphics[width=\textwidth, height= 0.8\textwidth]{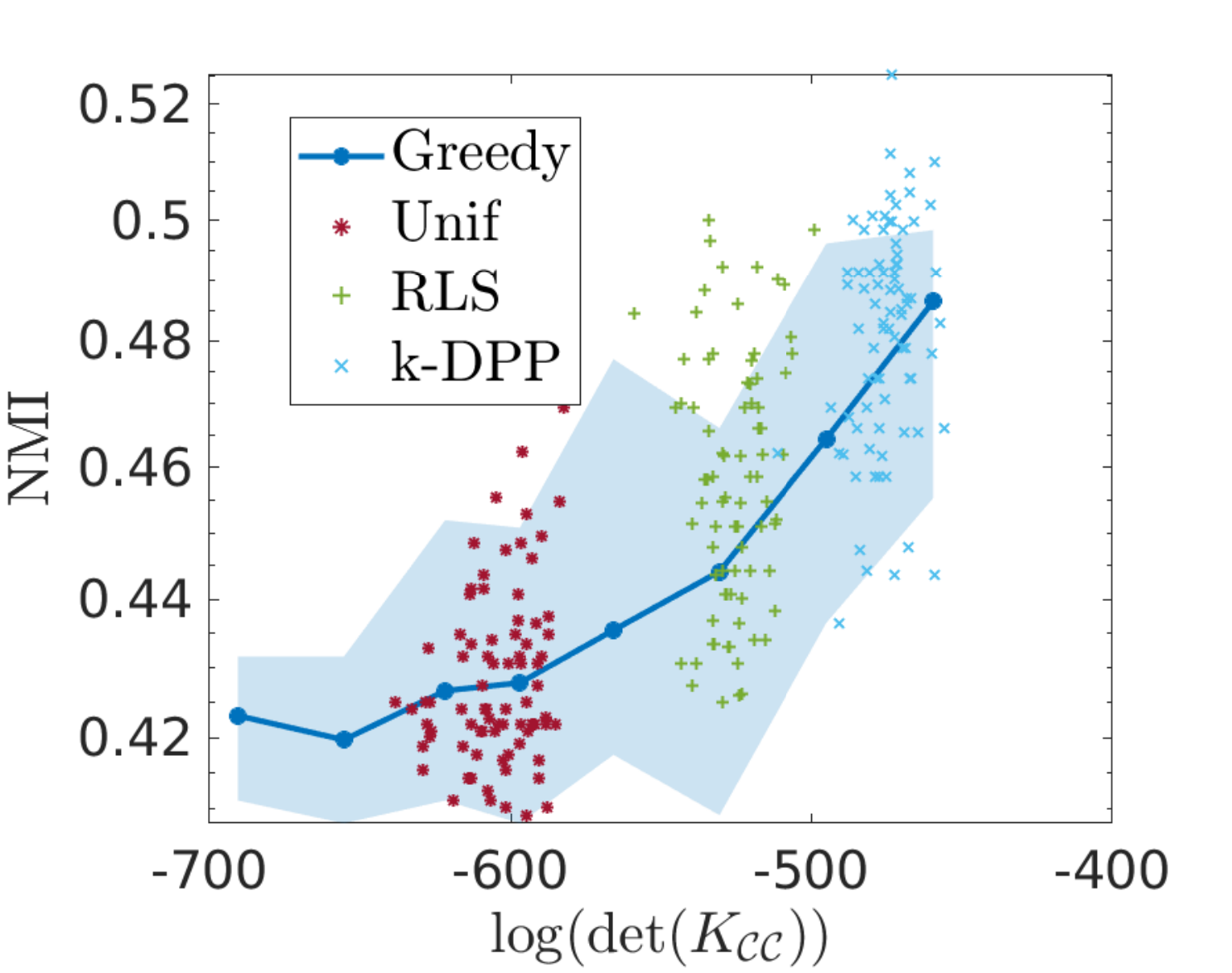}
		\caption{\texttt{B. Cancer}: NMI}
	\end{subfigure}	
	\caption{Clustering results. The condition number and NMI are plotted as a function of $\mathrm{log}(\mathrm{det}(K_{\mathcal{C}\mathcal{C}}))$.  A large NMI corresponds to a good accuracy.}\label{fig:ClusteringSupp}
\end{figure}

\subsection{Supplementary numerical experiments}
Several additional illustrations, obtained with the main methodology as in the manuscript, on the datasets given in Table~\ref{Table:dataSupp} are given in the sequel. Nystr\"om approximation error in Figures~\ref{fig:NystromSupp} and \ref{fig:NystromLSSupp}, Kernel PCA in Figure~\ref{fig:KPCASupp50}, and Kernel Ridge Regression in Figures~\ref{fig:RegressionSupp} and \ref{fig:RegressionLSSupp}. In most of the tasks illustrated in those figures, a larger diversity yields an improved performance. 
Let us discuss some particular cases.
\paragraph{Kernel approximation} As it was mentioned already hereabove, in the presence of outliers, a very diverse subsample can produced a poor kernel approximation. This can be viewed in Figure~\ref{fig:NystromBankAccSupp}, where the greedy algorithm is able to select subsets with a very large diversity. The randomized sampling methods that we studied empirically here did not suffer from this issue.
\paragraph{KRR} By using the same methodology as in the manuscript, the Mean Absolute Percentage Error is calculated both in the bulk and in the tail of the leverage score distribution of the test set. Notice that uniform sampling can often reduce the MAPE in the bulk of the data, while diverse sampling yields a larger improvement in the tail of the distribution (cfr. Figure~\ref{fig:RegressionParkinsonSupp}).
\begin{table}[h]
	\caption{Datasets and parameters used for the experiments on the Nystr\"om approximation.}
	\label{Table:dataSupp}
        \begin{center}
                \begin{tabular}{rcccccc}
                        \toprule
                        Dataset & Task & $n$ & $d$  & $\sigma$ & $\lambda$ & $d_{\mathrm{eff}}(n\lambda)$  \\ \midrule                    
                       \texttt{Stock}& Kernel approx. & $950$ & $10$ & $5$ & $10^{-6}$& $119$  \\                      
                        \texttt{Abalone}& Kernel approx. & $4177$ &$ 8 $ & $10$ & $10^{-6}$& $37$ \\
                        \texttt{Bank 8FM}& Kernel approx. & $8192$ &$ 8$  & $10 $& $10^{-6}$& $95$ \\ 
\midrule
			\texttt{Parkinson}& KPCA & $5875$ & $20 $ & $10$ & $10^{-6}$& $219$  \\
			\texttt{Wine Quality}& KPCA & $6497$ & $11 $ &$ 10$ & $10^{-6}$& $133$  \\  \midrule
 			\texttt{Housing}& KRR & $506$ & $13$  & $3$ & $10^{-4}$& $110$ \\
 			\texttt{Parkinson}& KRR & $5875$ & $20 $ & $3$ & $10^{-4}$& $343$  \\ \bottomrule
 		\end{tabular}
 	\end{center}
 \end{table}
	\begin{figure}[t]
		\centering
		\begin{subfigure}[b]{0.31\textwidth}
			\includegraphics[width=\textwidth, height= 0.8\textwidth]{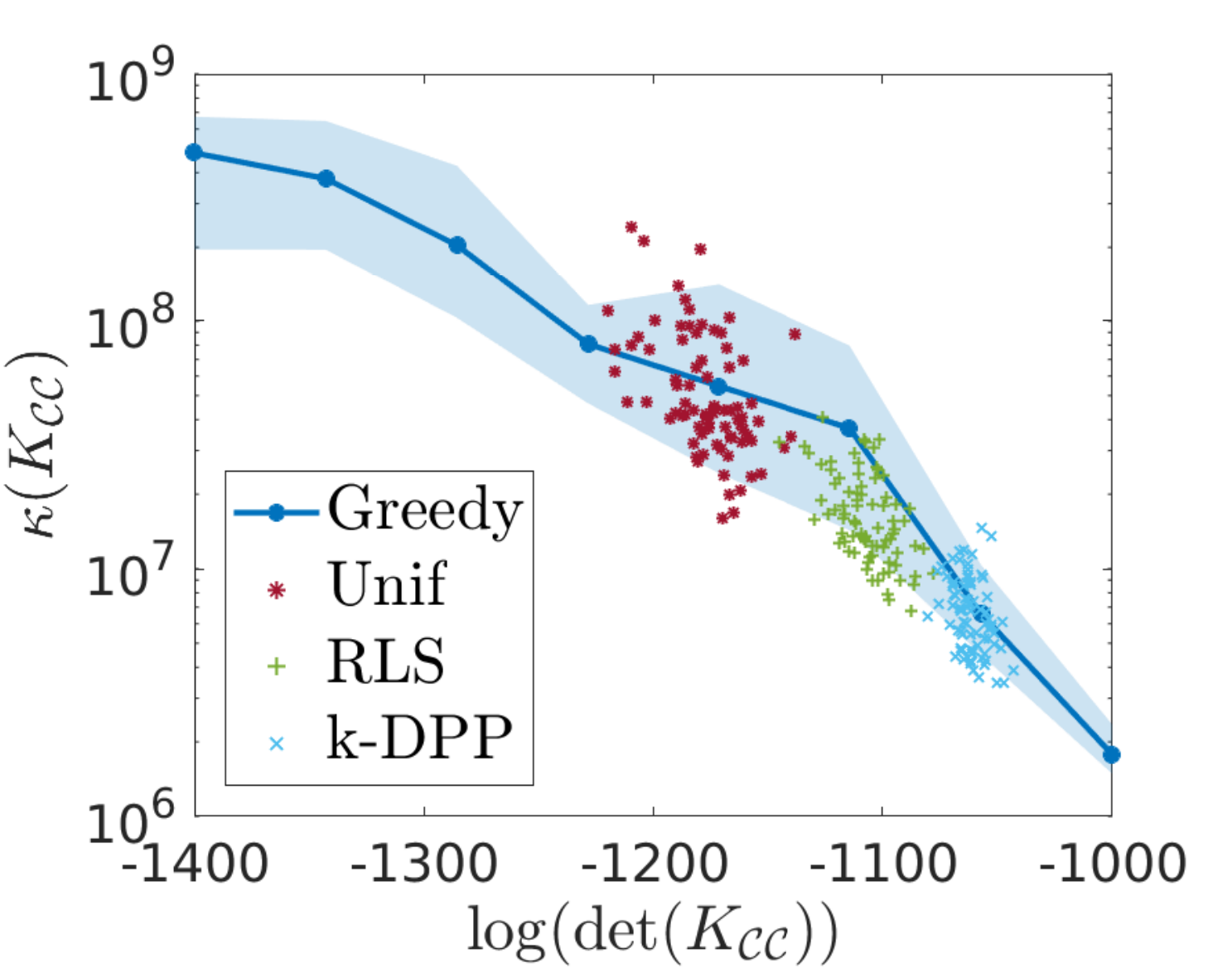}
			\caption{\texttt{Stock}: $\kappa(K_{\mathcal{C}\mathcal{C}})$}
		\end{subfigure}
		\begin{subfigure}[b]{0.31\textwidth}
			\includegraphics[width=\textwidth, height= 0.8\textwidth]{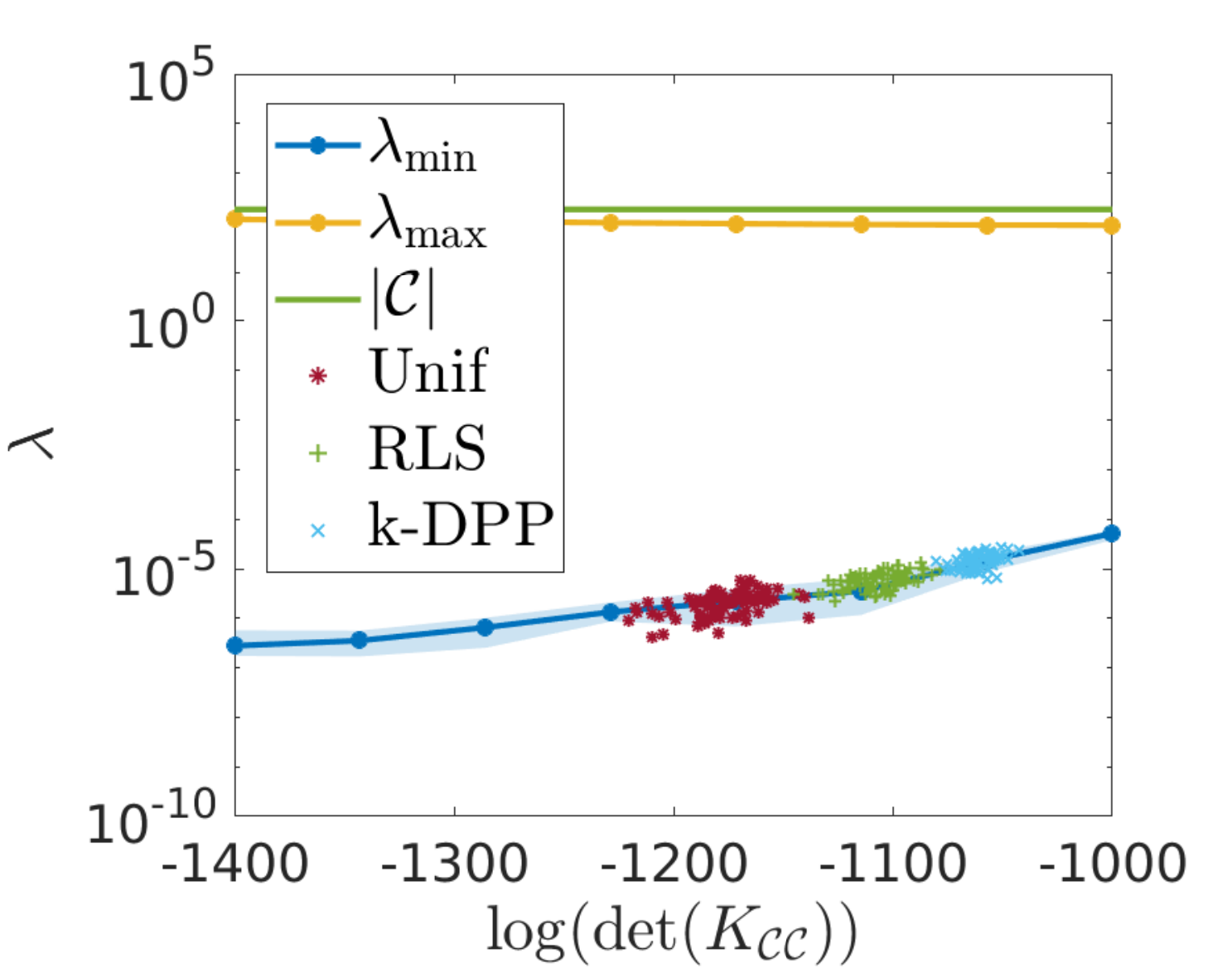}
			\caption{\texttt{Stock}: $\lambda_{\min}(K_{\mathcal{C}\mathcal{C}})$}
		\end{subfigure}
		\begin{subfigure}[b]{0.31\textwidth}
			\includegraphics[width=1\textwidth, height= 0.8\textwidth]{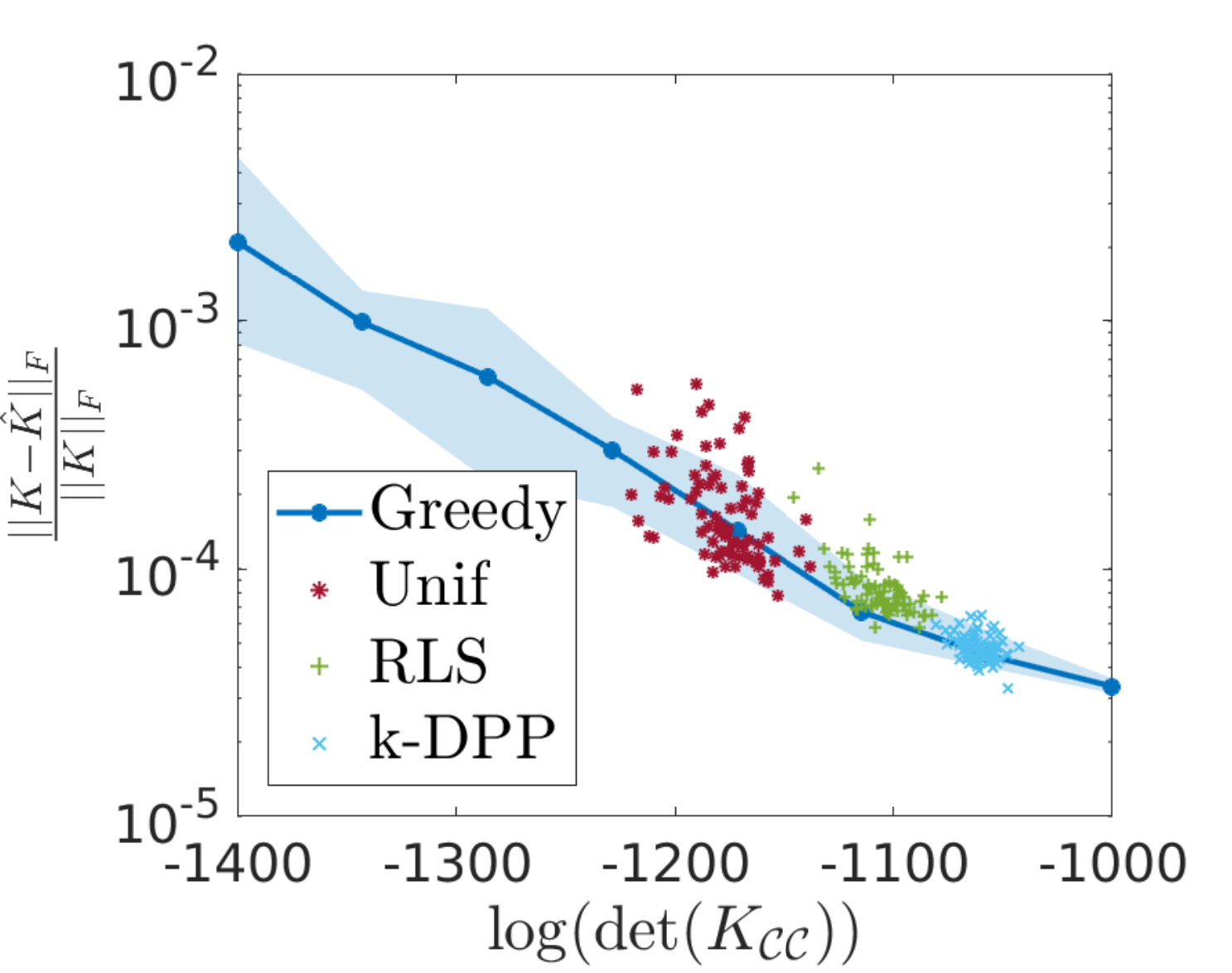}
			\caption{\texttt{Stock}: accuracy}
		\end{subfigure}
		\begin{subfigure}[b]{0.31\textwidth}
			\includegraphics[width=\textwidth, height= 0.8\textwidth]{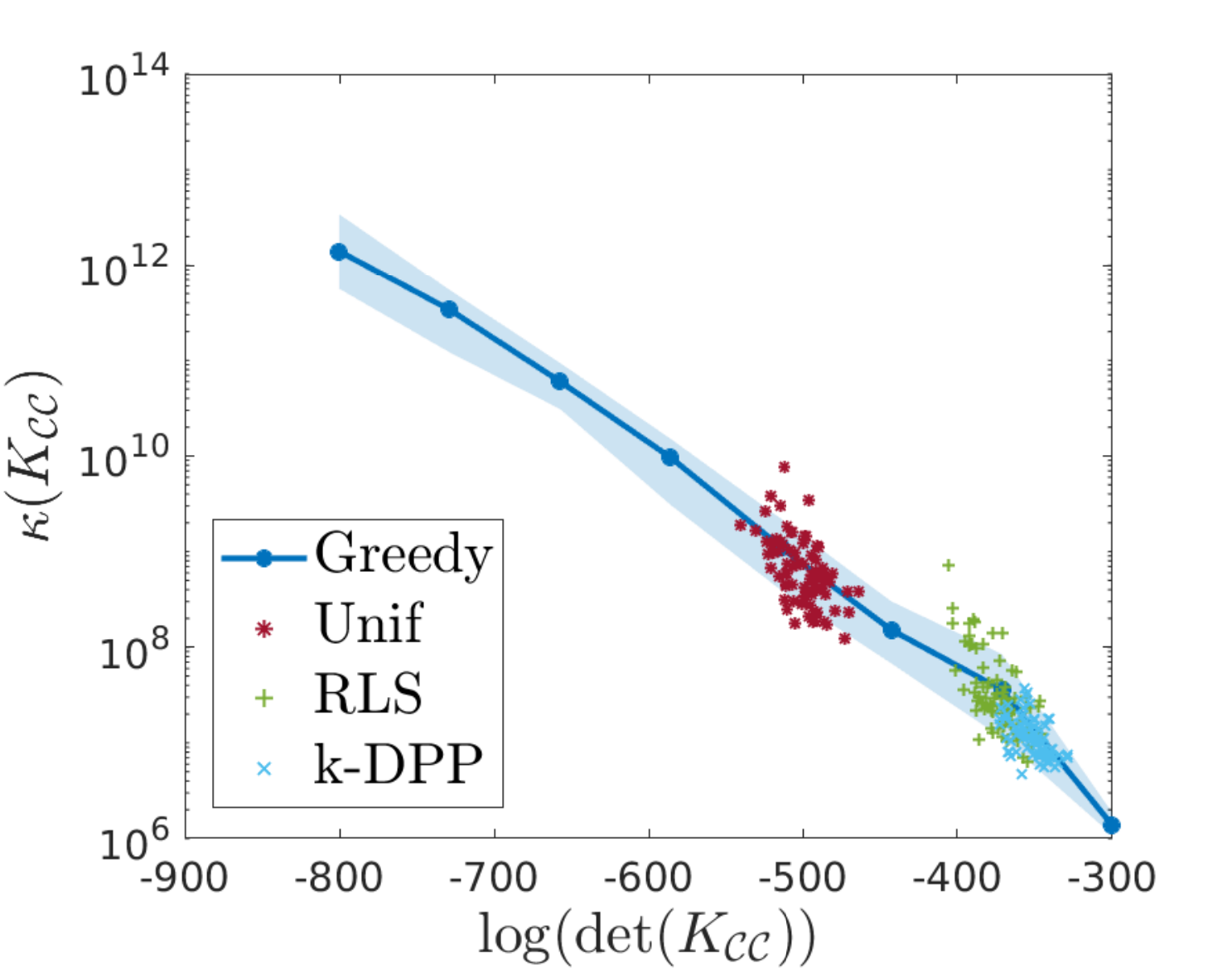}
			\caption{\texttt{Abalone}: $\kappa(K_{\mathcal{C}\mathcal{C}})$}
		\end{subfigure}
		\begin{subfigure}[b]{0.31\textwidth}
			\includegraphics[width=\textwidth, height= 0.8\textwidth]{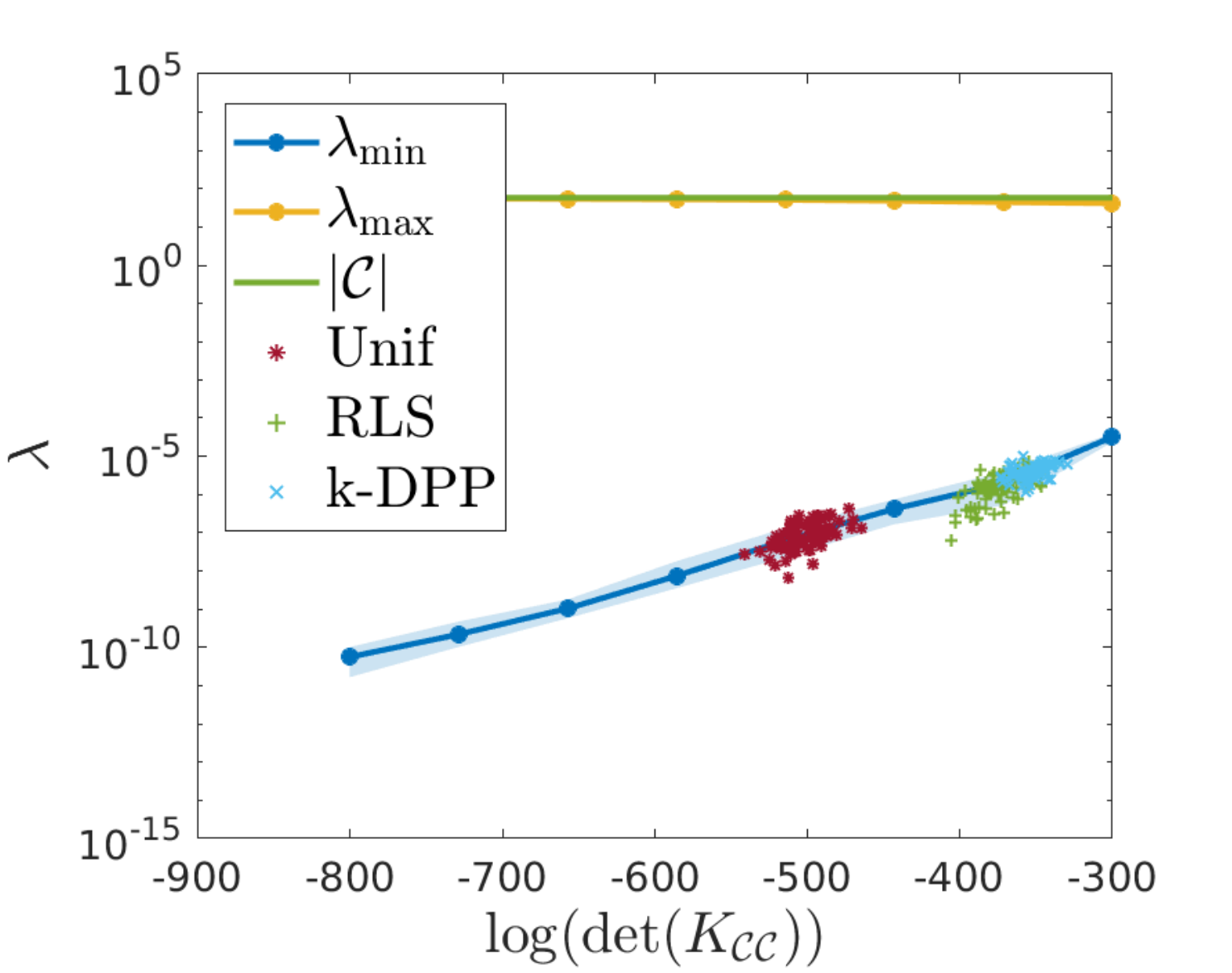}
			\caption{\texttt{Abalone}: $\lambda_{\min}(K_{\mathcal{C}\mathcal{C}})$}
		\end{subfigure}
		\begin{subfigure}[b]{0.31\textwidth}
			\includegraphics[width=\textwidth, height= 0.8\textwidth]{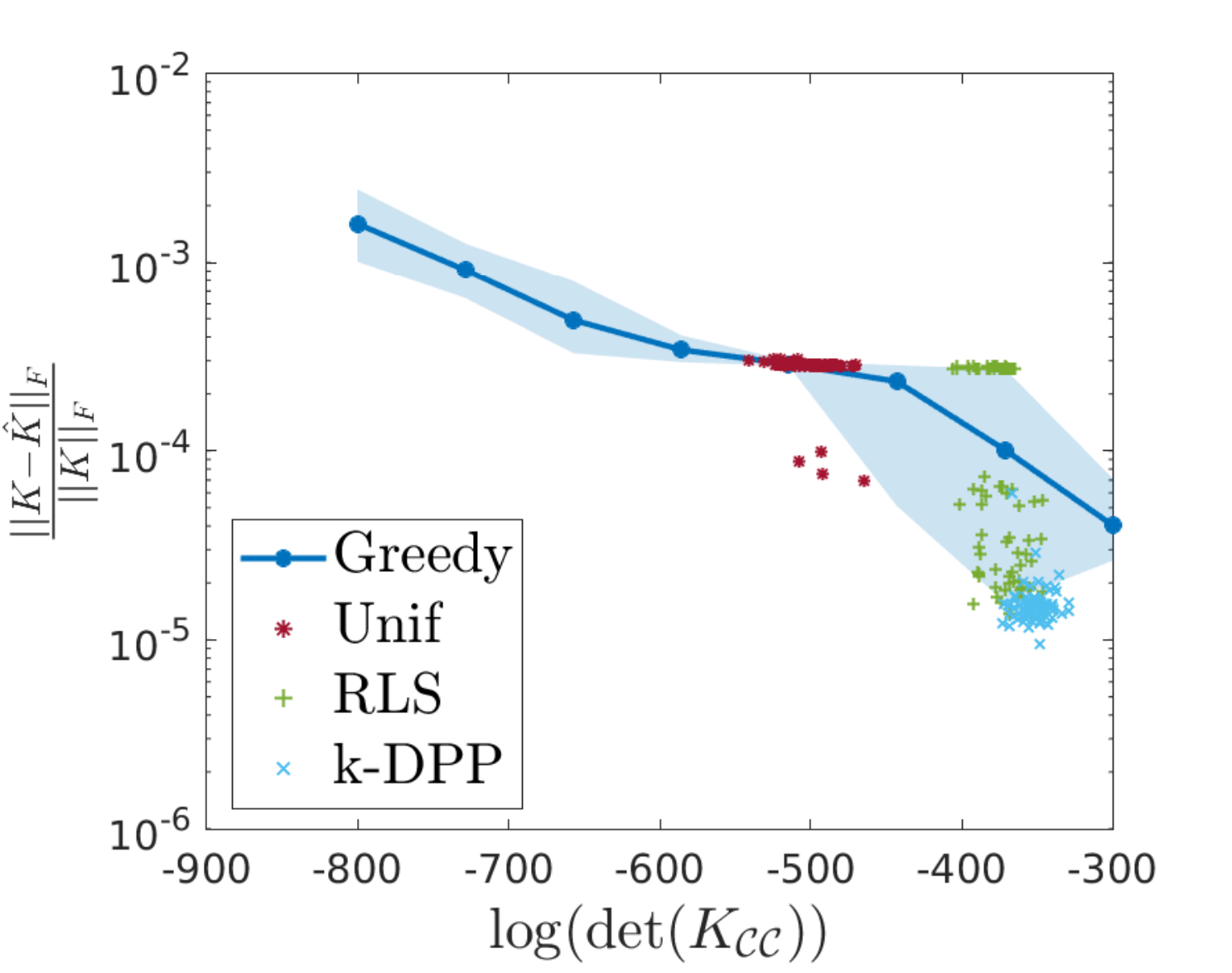}
			\caption{\texttt{Abalone}: accuracy}
		\end{subfigure}
		\begin{subfigure}[b]{0.31\textwidth}
			\includegraphics[width=\textwidth, height= 0.8\textwidth]{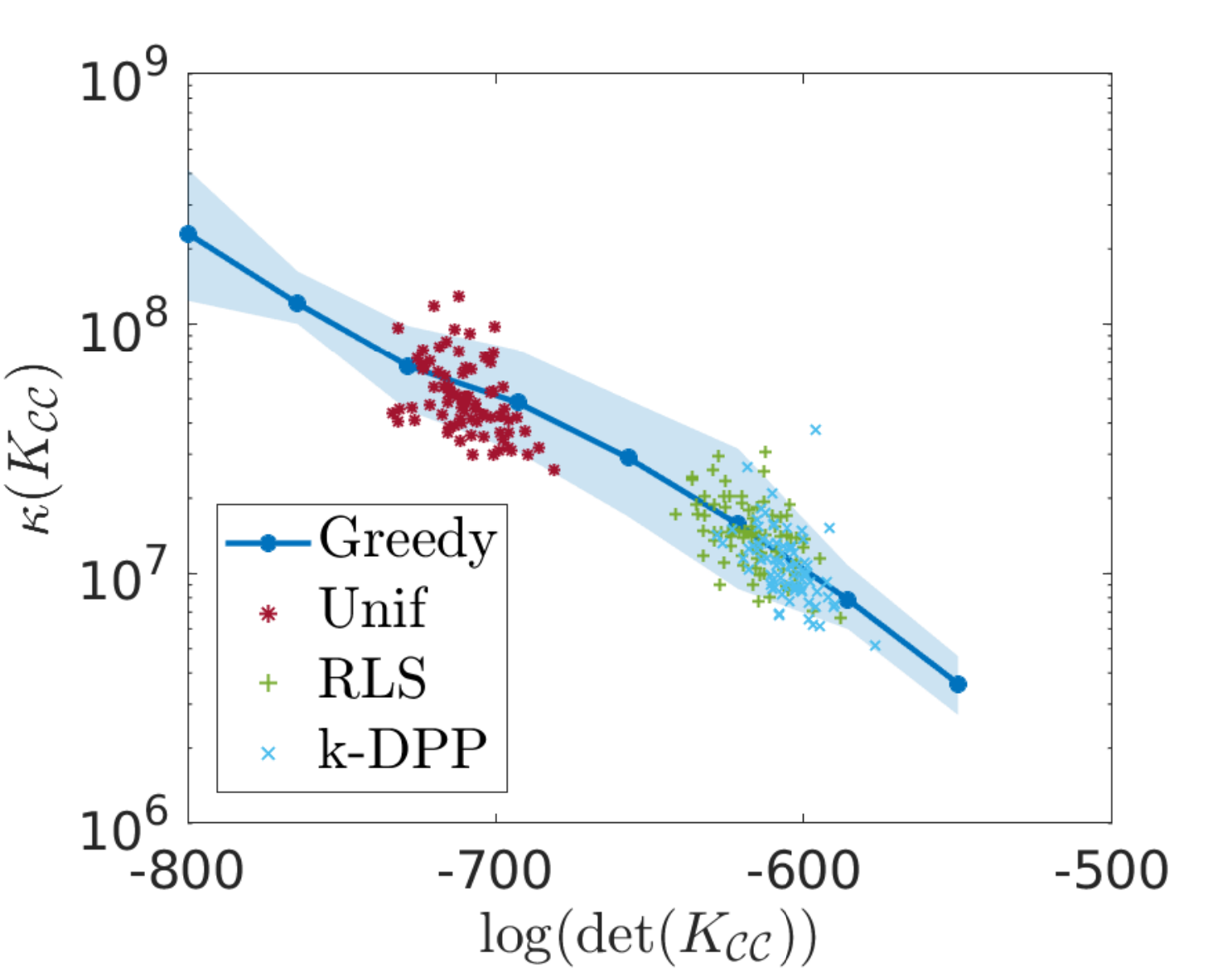}
			\caption{\texttt{Bank 8FM}: $\kappa(K_{\mathcal{C}\mathcal{C}})$}
		\end{subfigure}
		\begin{subfigure}[b]{0.31\textwidth}
			\includegraphics[width=\textwidth, height= 0.8\textwidth]{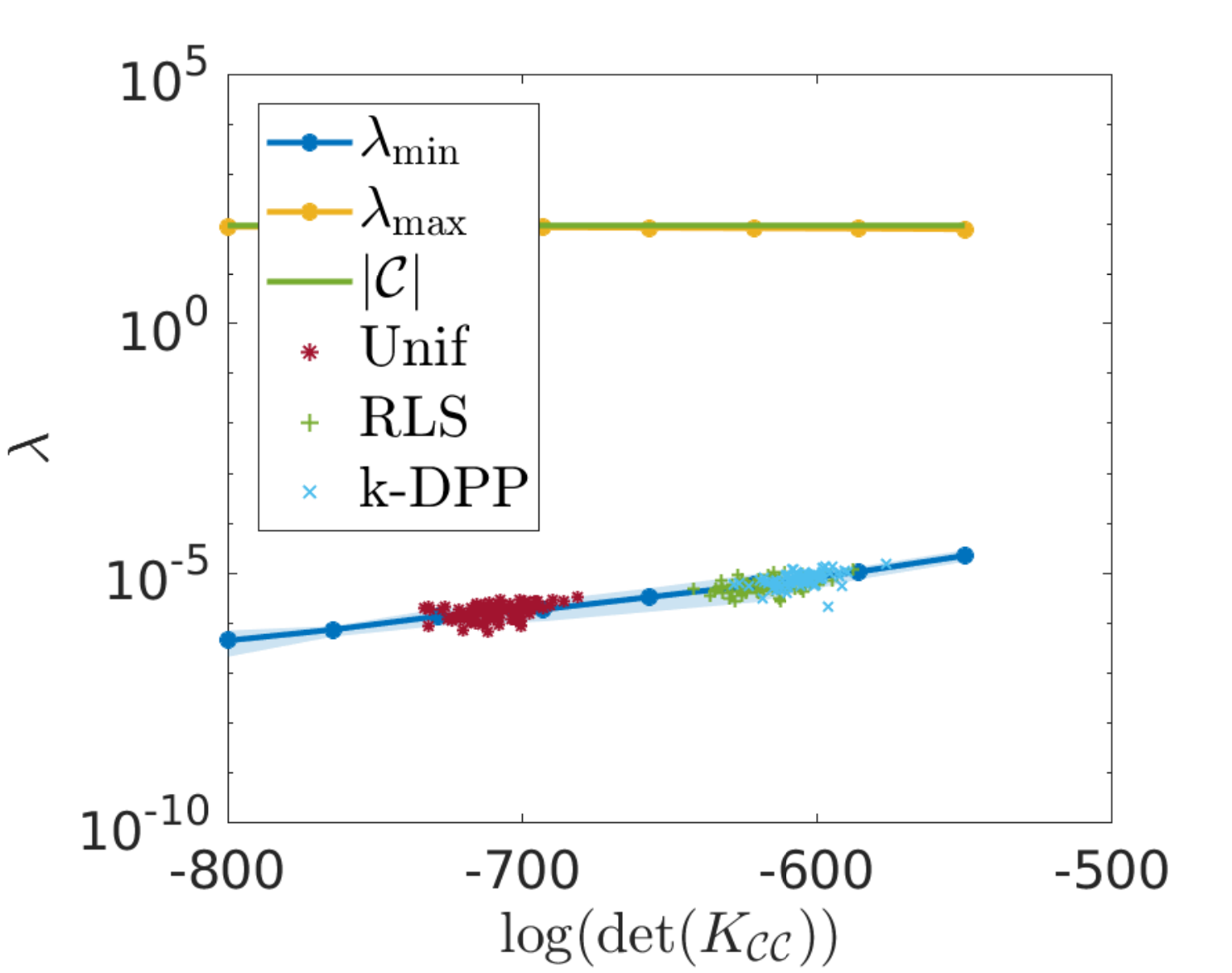}
			\caption{\texttt{Bank 8FM}: $\lambda_{\min}(K_{\mathcal{C}\mathcal{C}})$}
		\end{subfigure}
		\begin{subfigure}[b]{0.31\textwidth}
			\includegraphics[width=1\textwidth, height= 0.817\textwidth]{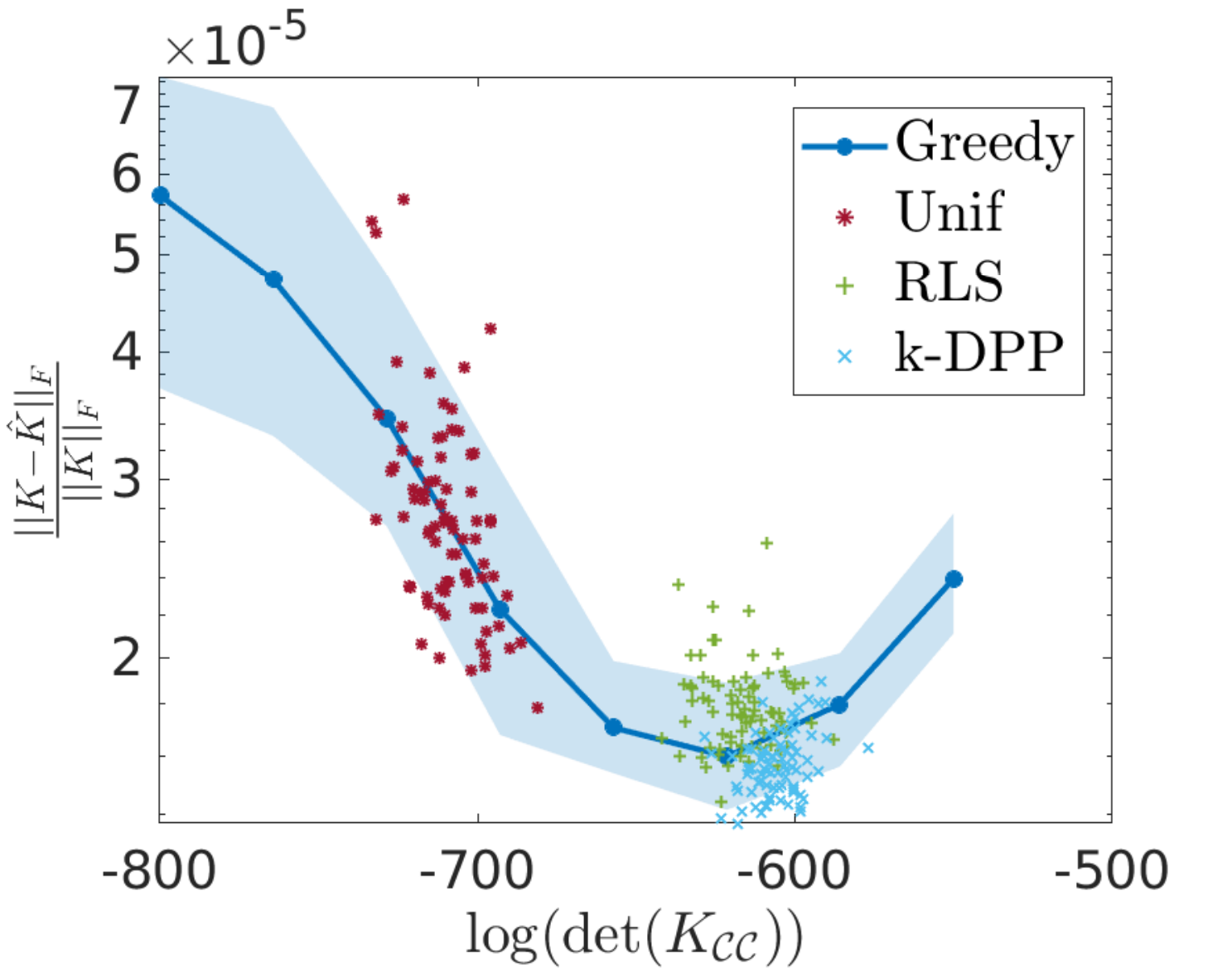}
			\caption{\texttt{Bank 8FM}: accuracy \label{fig:NystromBankAccSupp}}
		\end{subfigure}

	\caption{Kernel approximation results. The condition number, smallest and largest eigenvalues of $K_{\mathcal{C}\mathcal{C}}$, relative Frobenius norm of the Nystr\"om approximation error versus $\mathrm{log}(\mathrm{det}(K_{\mathcal{C}\mathcal{C}}))$.}\label{fig:NystromSupp}
\end{figure}

\begin{figure}[b]
	\centering
	\begin{subfigure}[t]{0.24\textwidth}
		\includegraphics[width=\textwidth, height= 0.9\textwidth]{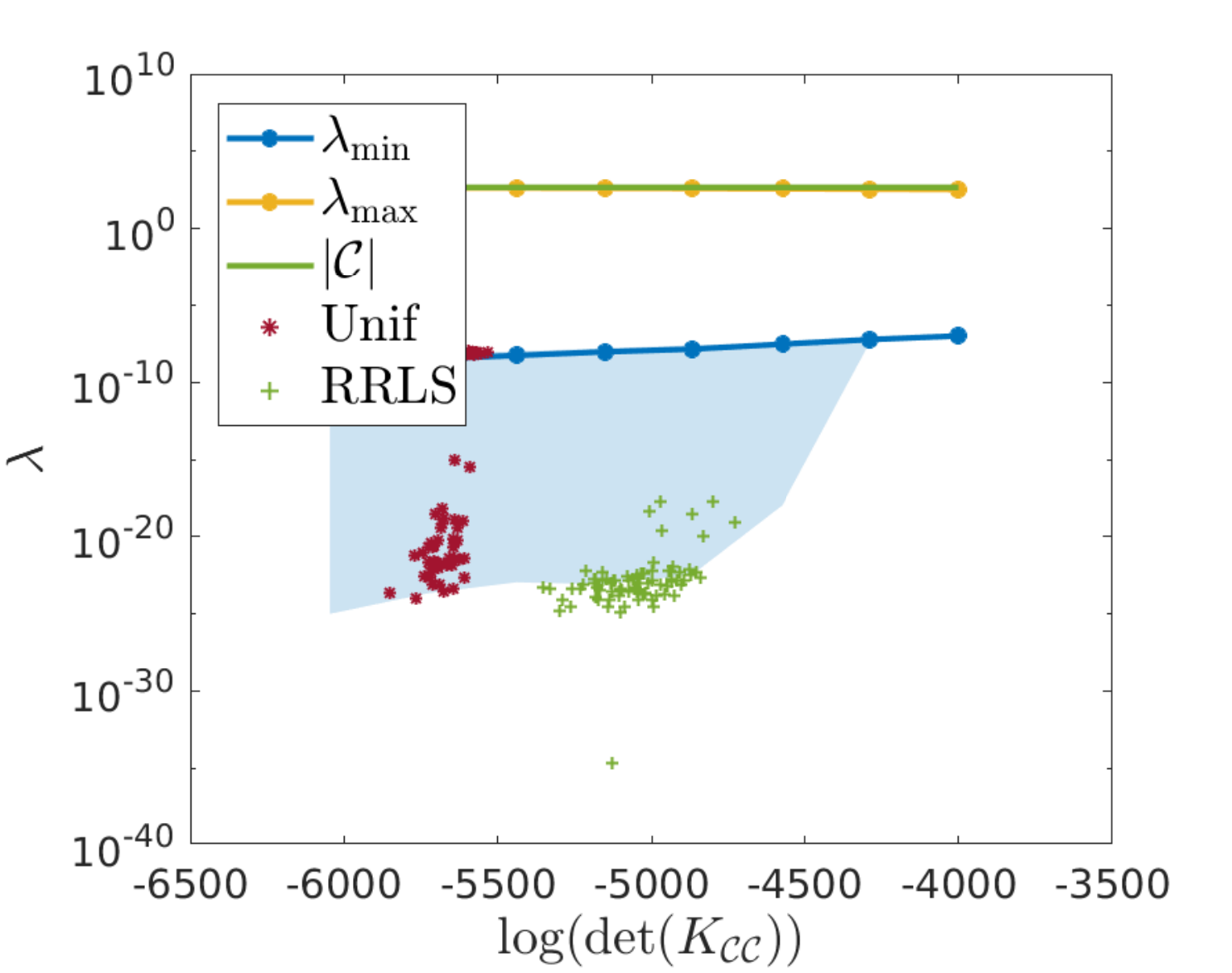}
		\caption{\texttt{MiniBooNE}}
	\end{subfigure}
	\begin{subfigure}[t]{0.24\textwidth}
		\includegraphics[width=\textwidth, height= 0.9\textwidth]{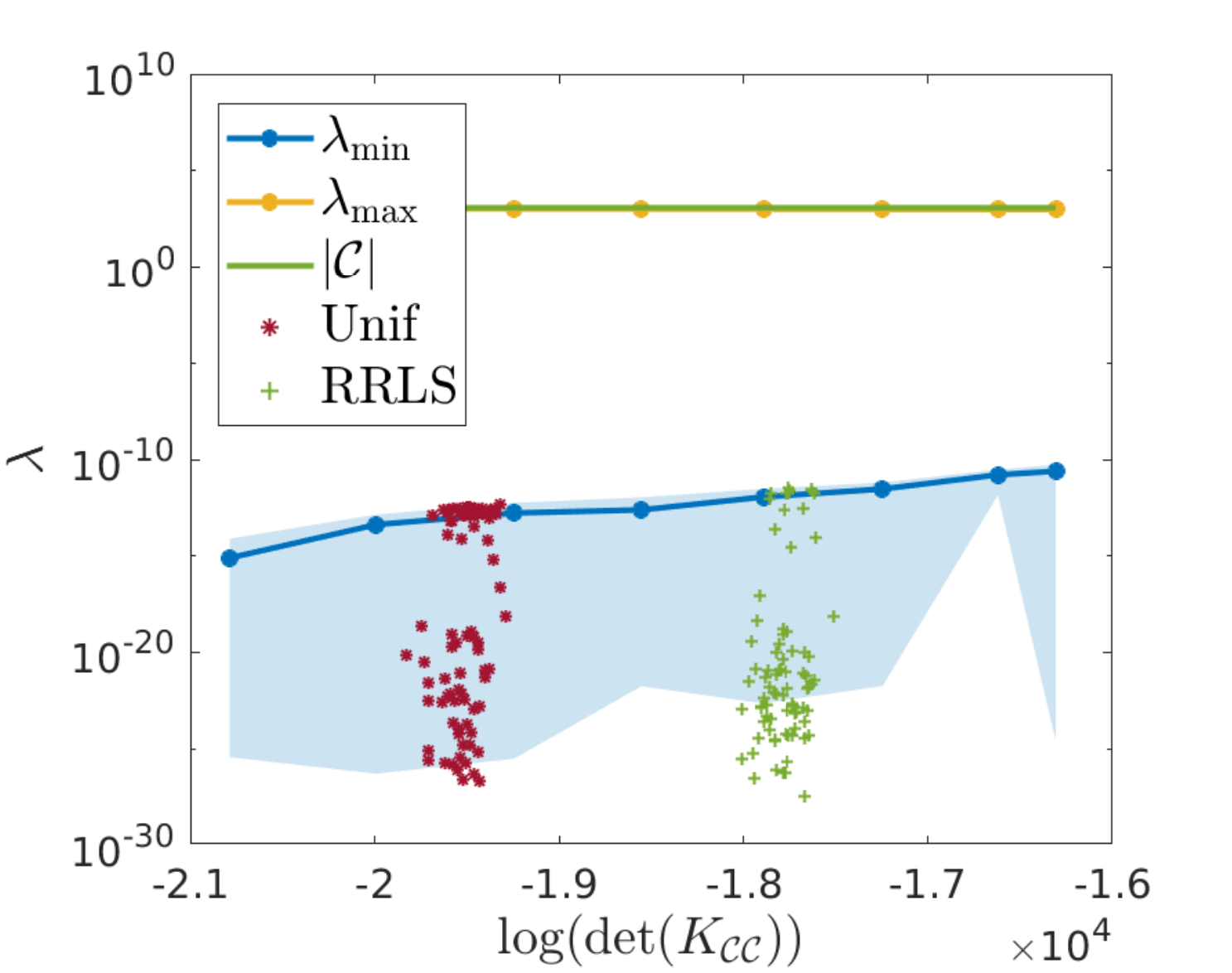}
		\caption{\texttt{codRNA}}
	\end{subfigure}		
	\begin{subfigure}[t]{0.24\textwidth}
		\includegraphics[width=\textwidth, height= 0.9\textwidth]{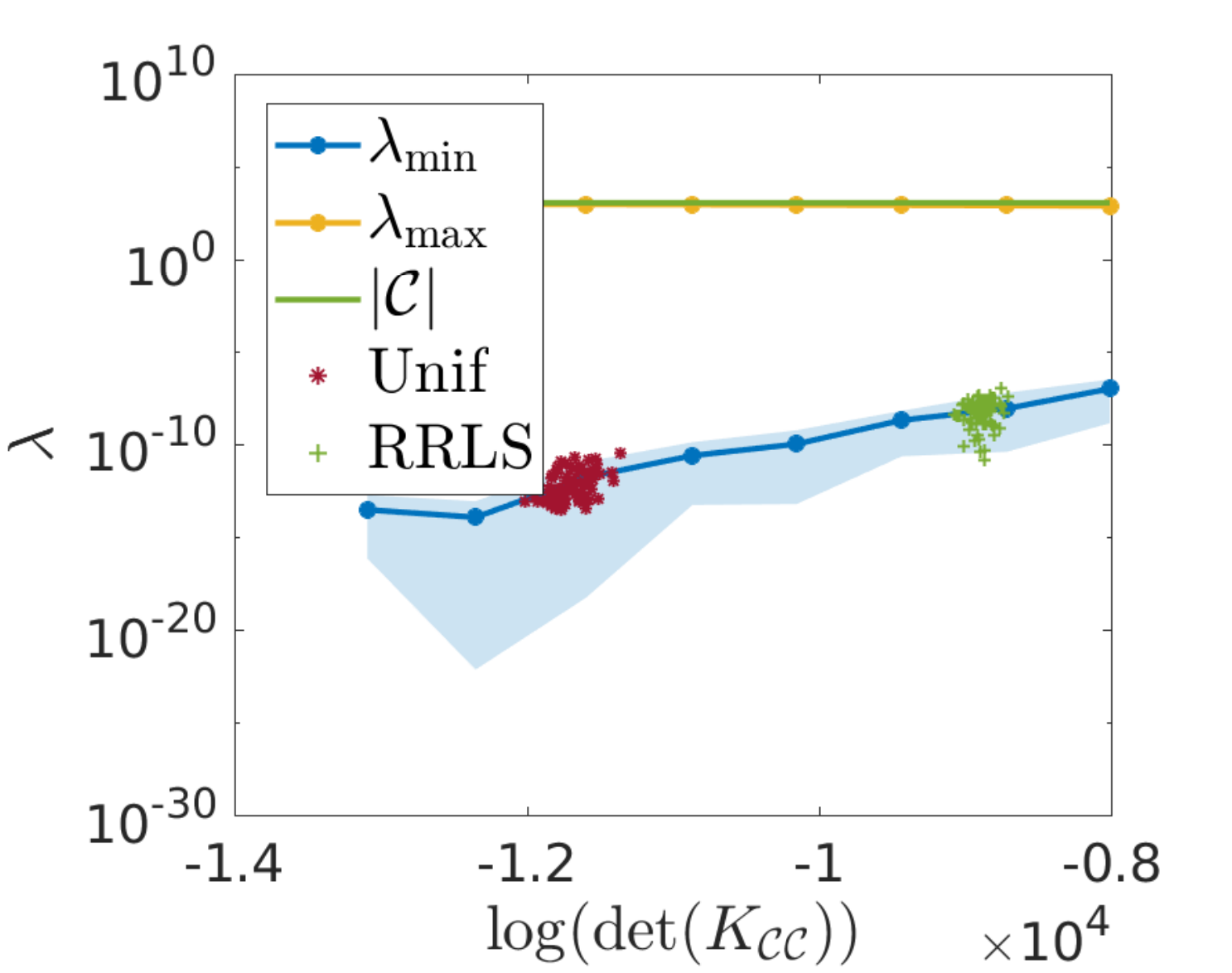}
		\caption{\texttt{Adult}}
	\end{subfigure}
	\begin{subfigure}[t]{0.24\textwidth}
		\includegraphics[width=\textwidth, height= 0.9\textwidth]{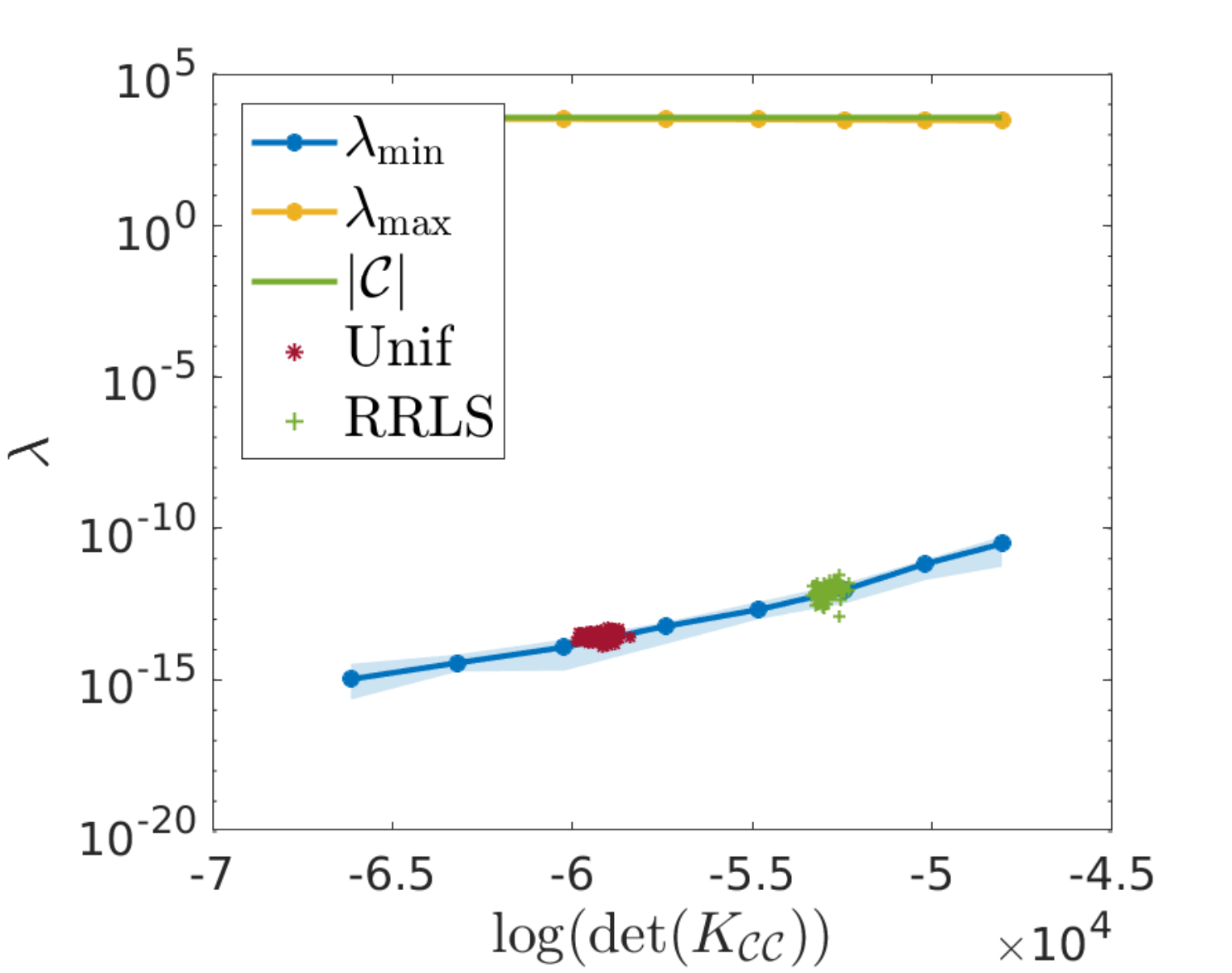}
		\caption{\texttt{Covertype}}
	\end{subfigure}		
	\caption{Results accompanying Figures \ref{fig:NystromLS} and \ref{fig:LS_KPCA}. The smallest and largest eigenvalues of $K_{\mathcal{C}\mathcal{C}}$ are plotted as a function of $\mathrm{log}(\mathrm{det}(K_{\mathcal{C}\mathcal{C}}))$.}\label{fig:LSSupp}
\end{figure}

\begin{figure}[b]
		\centering
		\begin{subfigure}[t]{0.24\textwidth}
			\includegraphics[width=\textwidth, height= 0.85\textwidth]{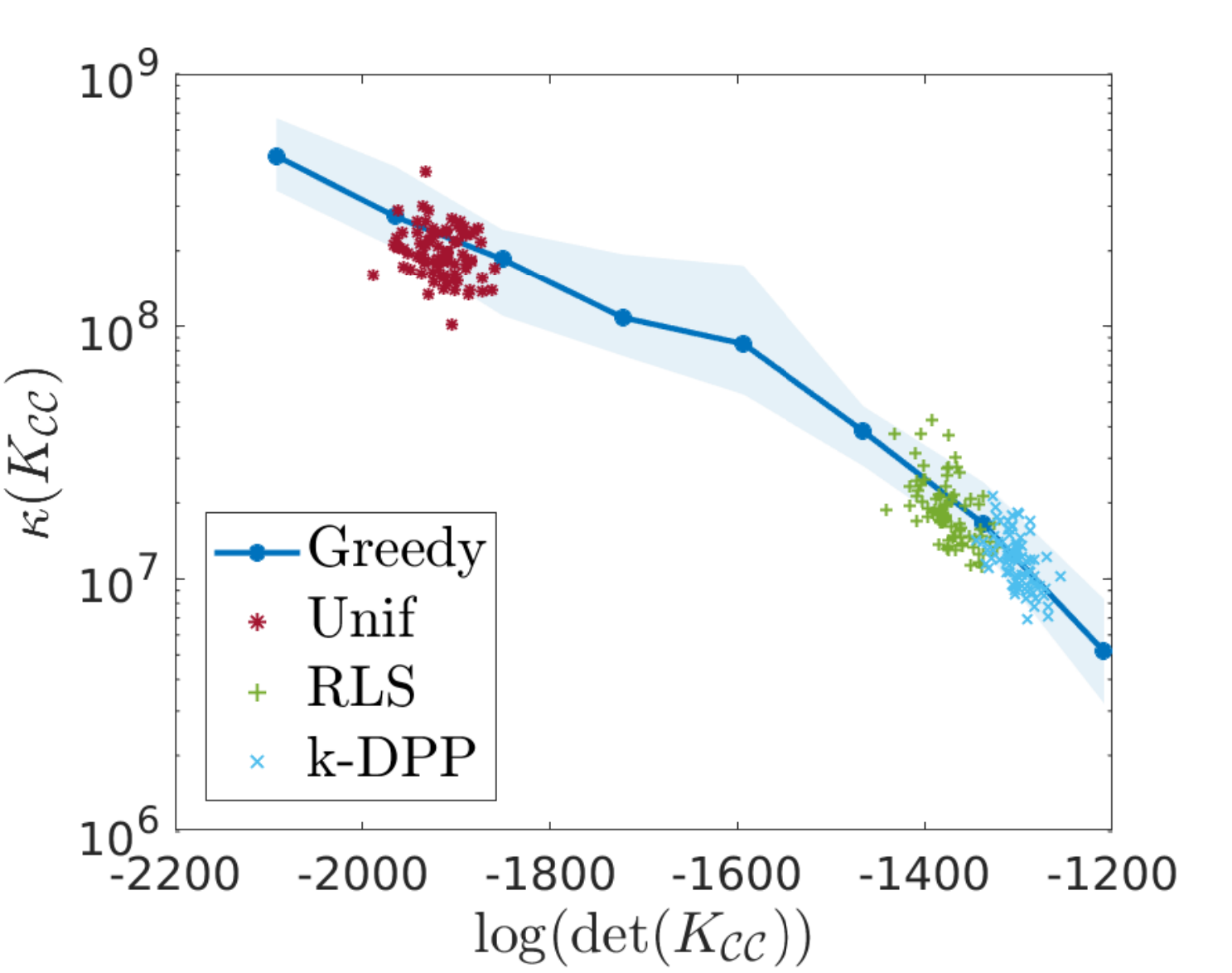}
			\caption{\texttt{Park.}: $\kappa(K_{\mathcal{C}\mathcal{C}})$}
		\end{subfigure}
		\begin{subfigure}[t]{0.24\textwidth}
			\includegraphics[width=\textwidth, height= 0.85\textwidth]{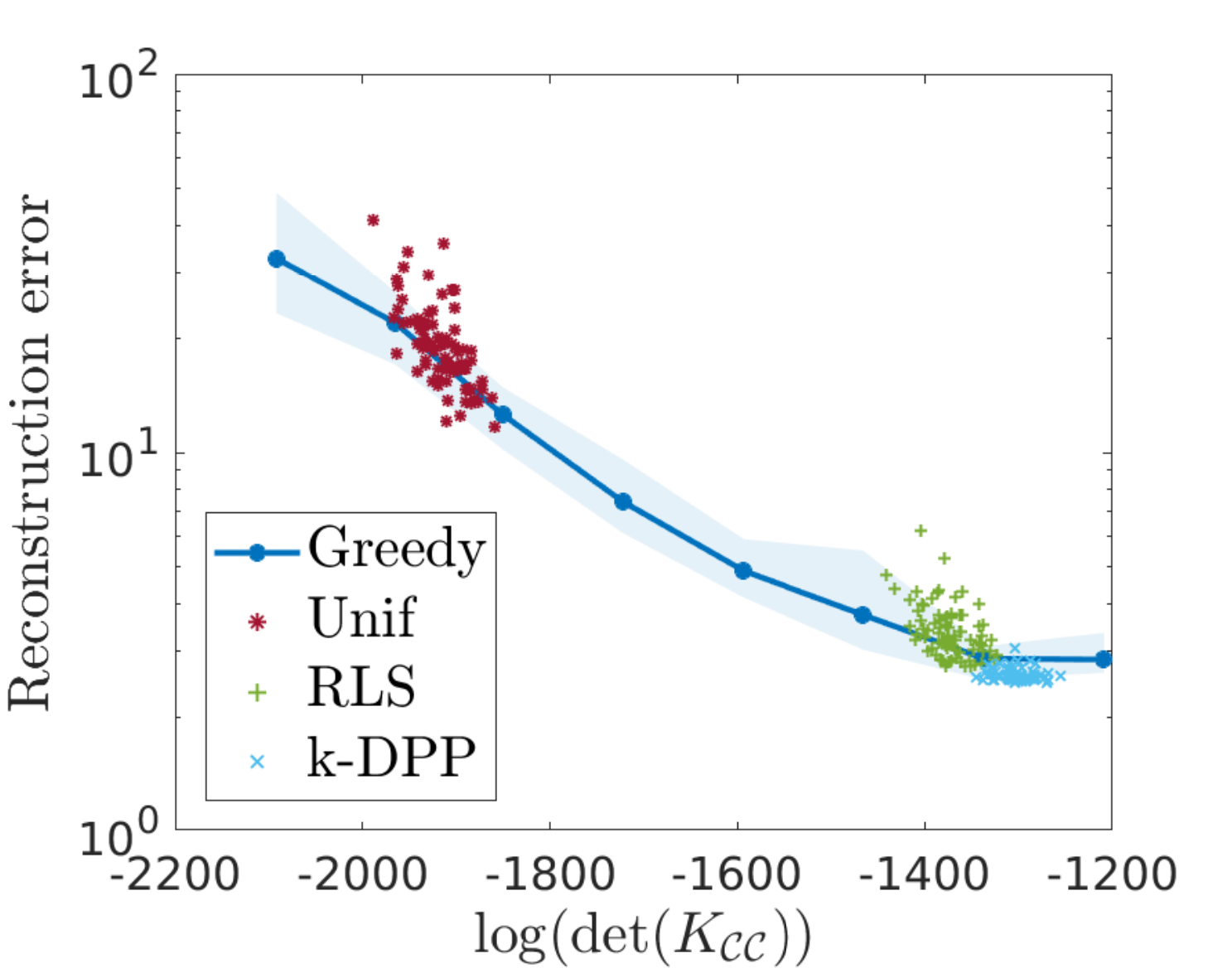}
			\caption{\texttt{Park.}: error}
		\end{subfigure}
		\begin{subfigure}[t]{0.24\textwidth}
			\includegraphics[width=\textwidth, height= 0.85\textwidth]{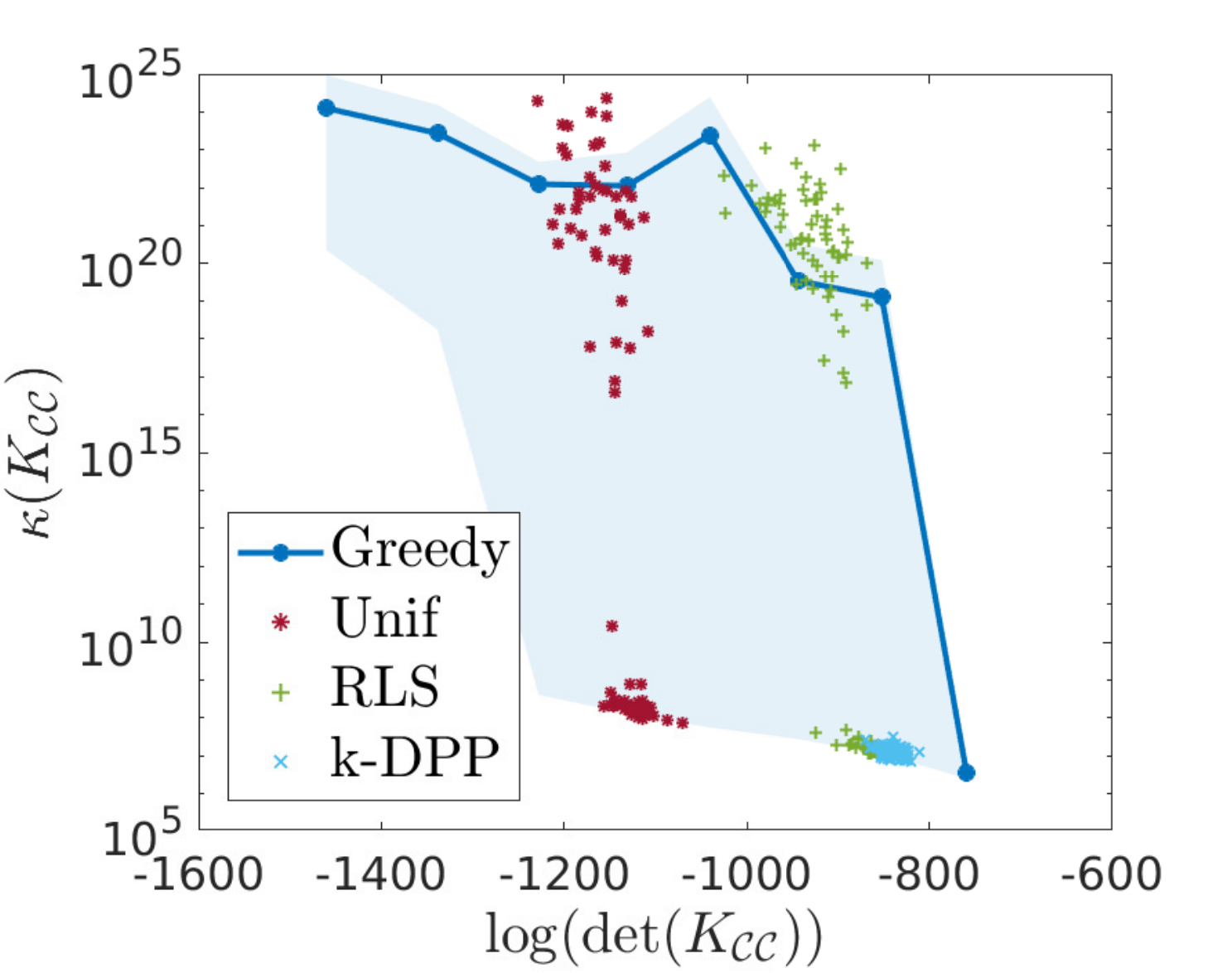}
			\caption{\texttt{Wine Q.}: $\kappa(K_{\mathcal{C}\mathcal{C}})$}
		\end{subfigure}
		\begin{subfigure}[t]{0.24\textwidth}
			\includegraphics[width=\textwidth, height= 0.85\textwidth]{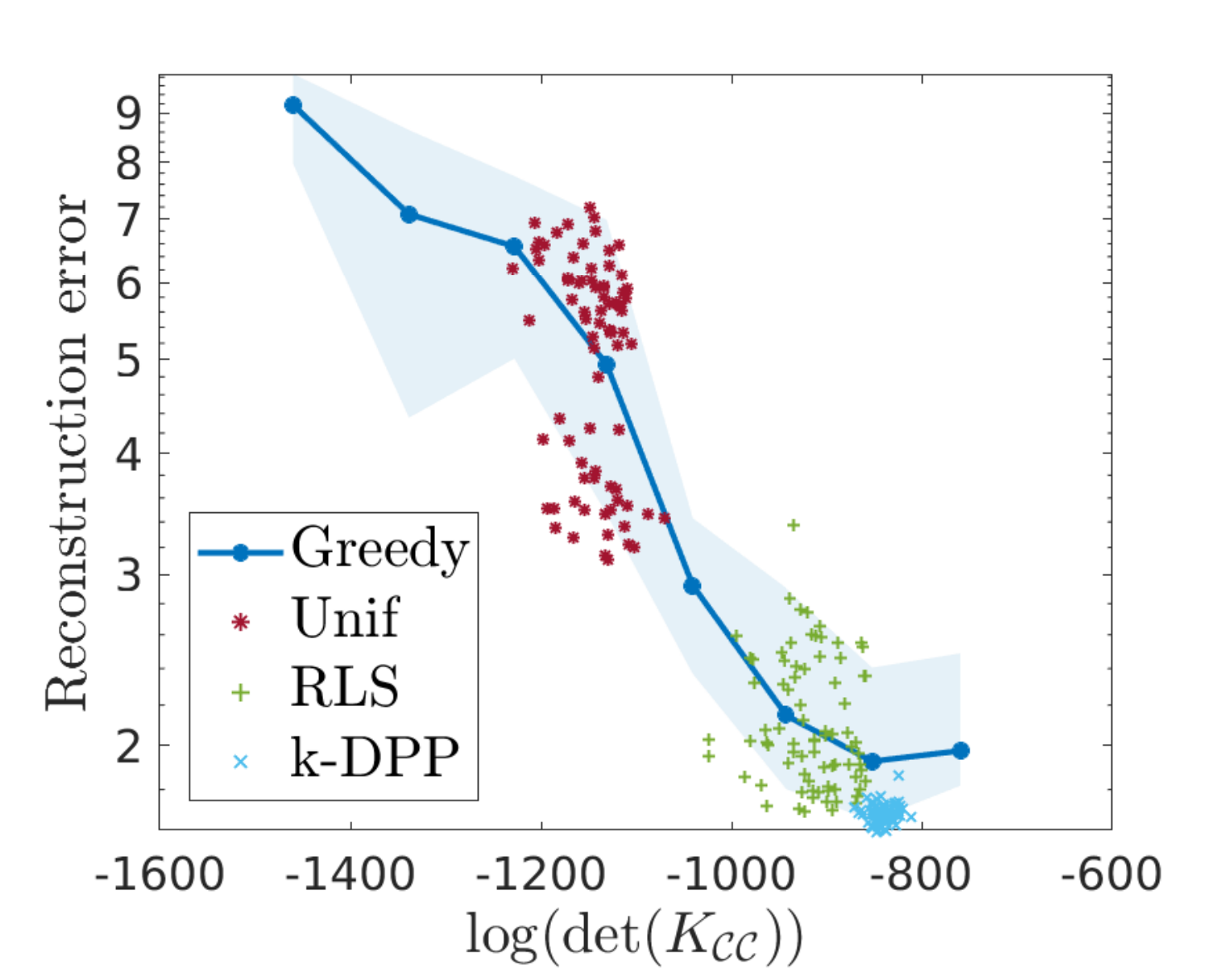}
			\caption{\texttt{Wine Q.}: error}
		\end{subfigure}				
		\caption{KPCA results. The condition number and reconstruction error using half of the components are plotted as a function of $\mathrm{log}(\mathrm{det}(K_{\mathcal{C}\mathcal{C}}))$.}\label{fig:KPCASupp50}
	\end{figure}


 \begin{figure}[t]
		\centering
		\begin{subfigure}[b]{0.31\textwidth}
			\includegraphics[width=\textwidth, height= 0.8\textwidth]{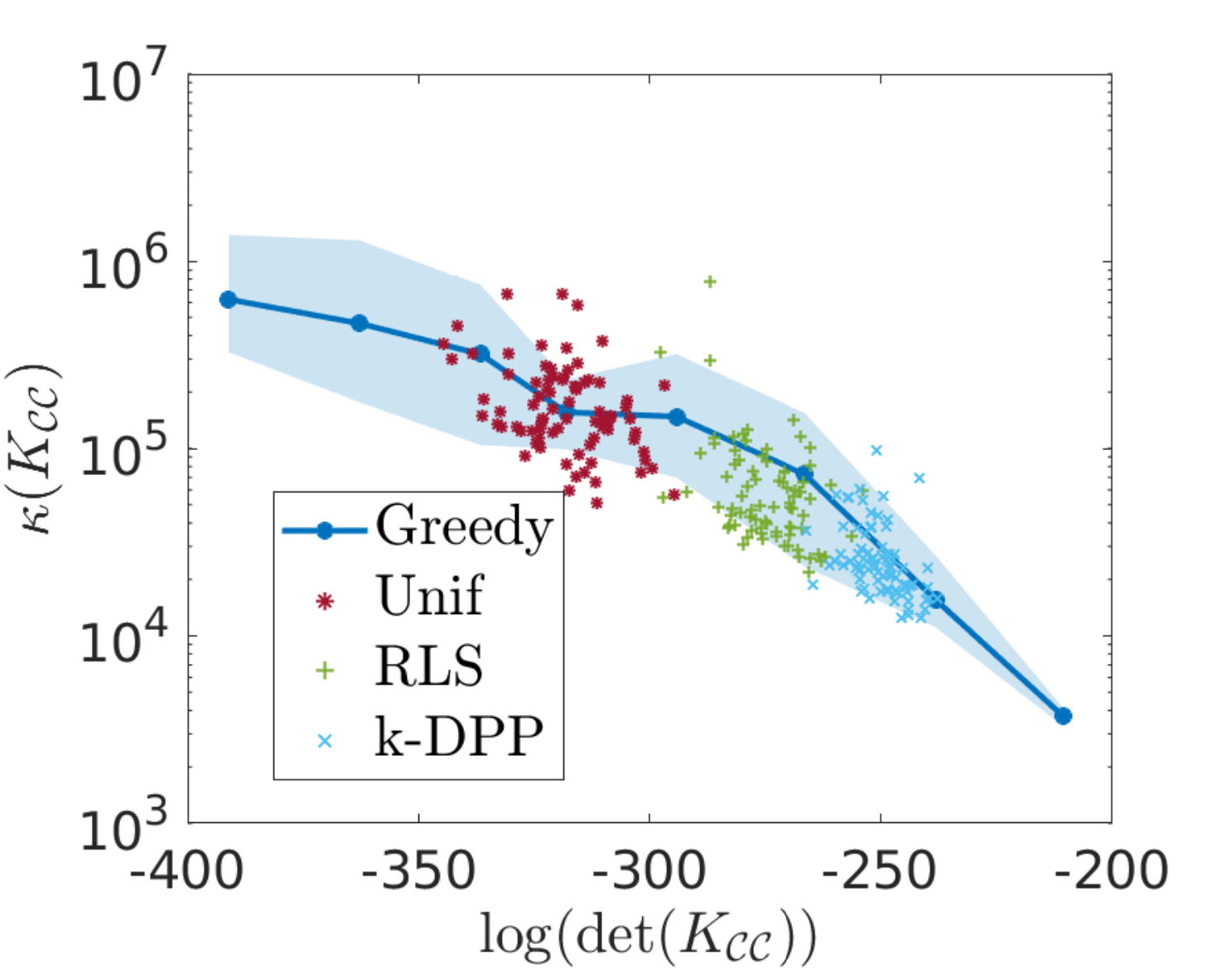}
			\caption{\texttt{Housing}: $\kappa(K_{\mathcal{C}\mathcal{C}})$}
		\end{subfigure}
		\begin{subfigure}[b]{0.31\textwidth}
			\includegraphics[width=\textwidth, height= 0.8\textwidth]{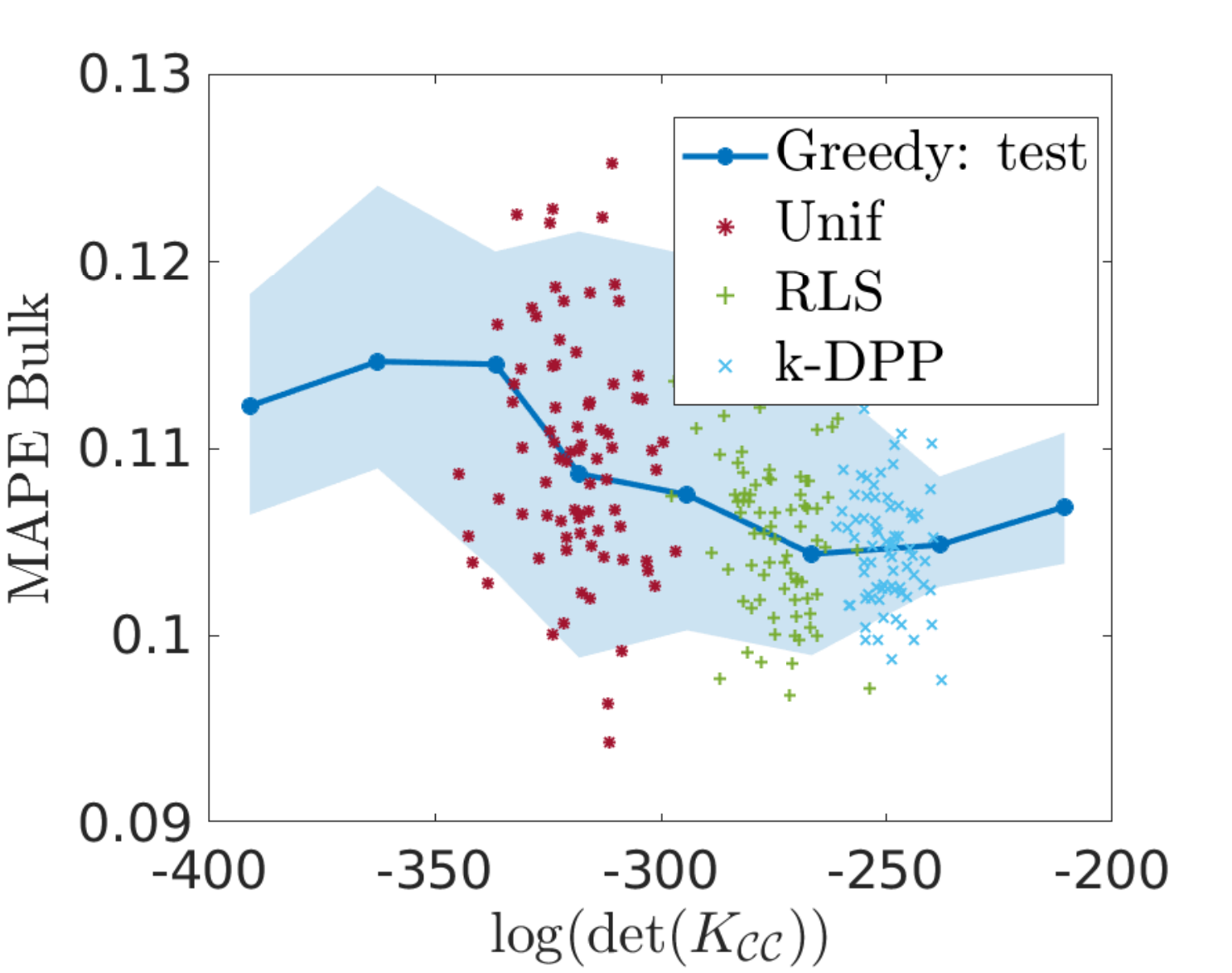}
			\caption{\texttt{Housing}: MAPE Bulk}
		\end{subfigure}
		\begin{subfigure}[b]{0.31\textwidth}
			\includegraphics[width=\textwidth, height= 0.8\textwidth]{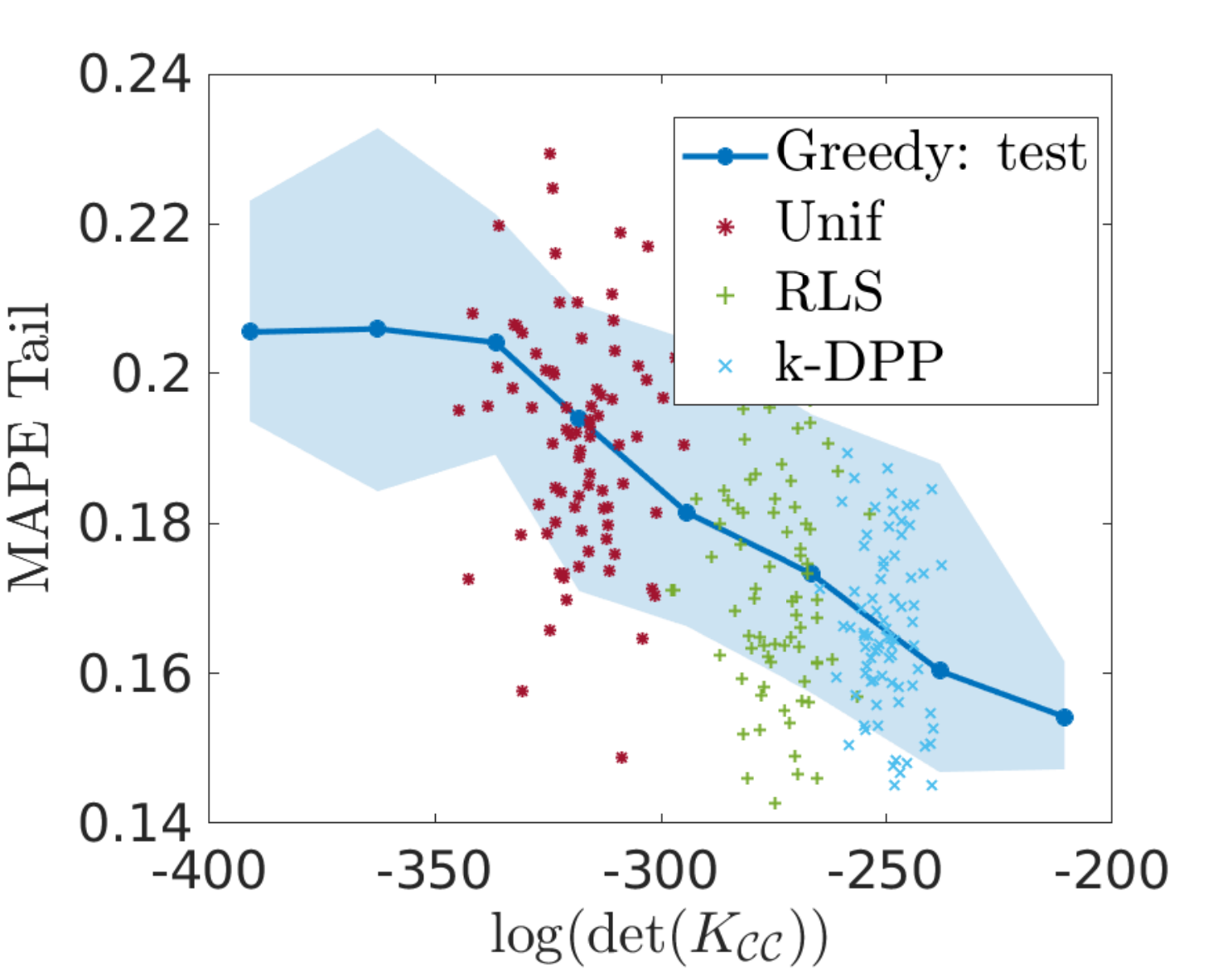}
			\caption{\texttt{Housing}: MAPE Tail}
		\end{subfigure}
		\begin{subfigure}[b]{0.31\textwidth}
			\includegraphics[width=\textwidth, height= 0.8\textwidth]{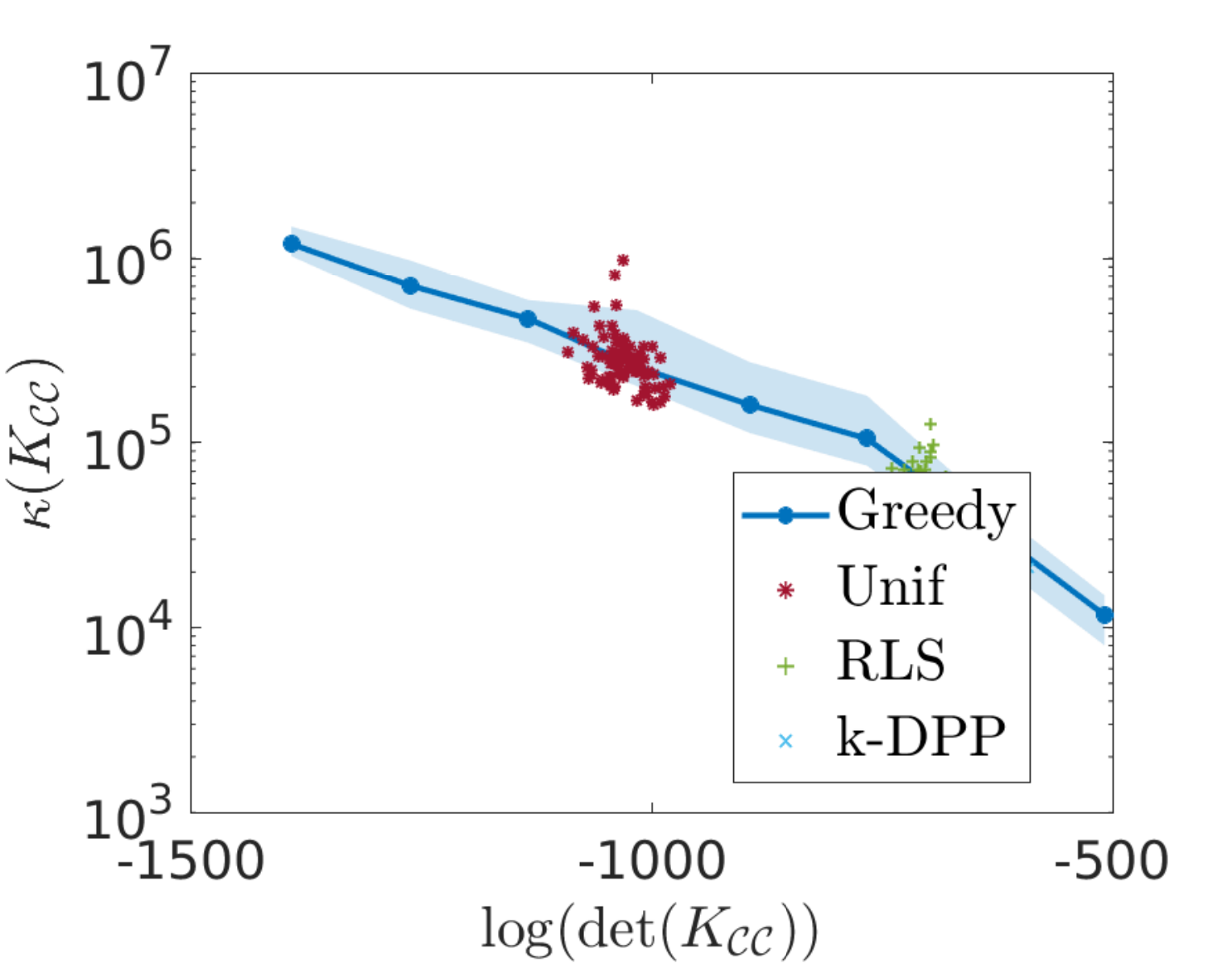}
			\caption{\texttt{Park.}: $\kappa(K_{\mathcal{C}\mathcal{C}})$}
		\end{subfigure}
		\begin{subfigure}[b]{0.31\textwidth}
			\includegraphics[width=\textwidth, height= 0.8\textwidth]{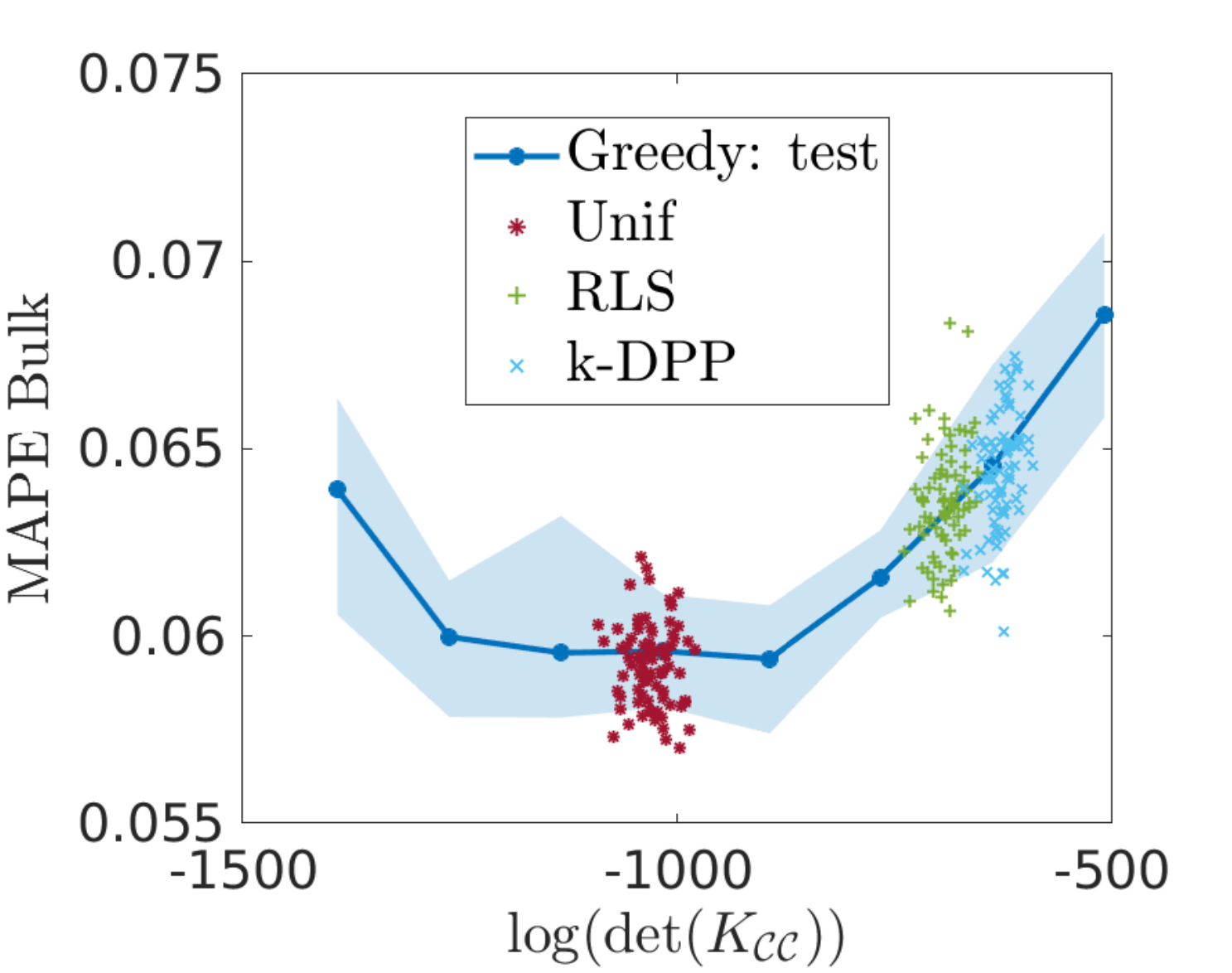}
			\caption{\texttt{Park.}: MAPE Bulk \label{fig:RegressionParkinsonSupp}}
		\end{subfigure}
		\begin{subfigure}[b]{0.31\textwidth}
			\includegraphics[width=\textwidth, height= 0.8\textwidth]{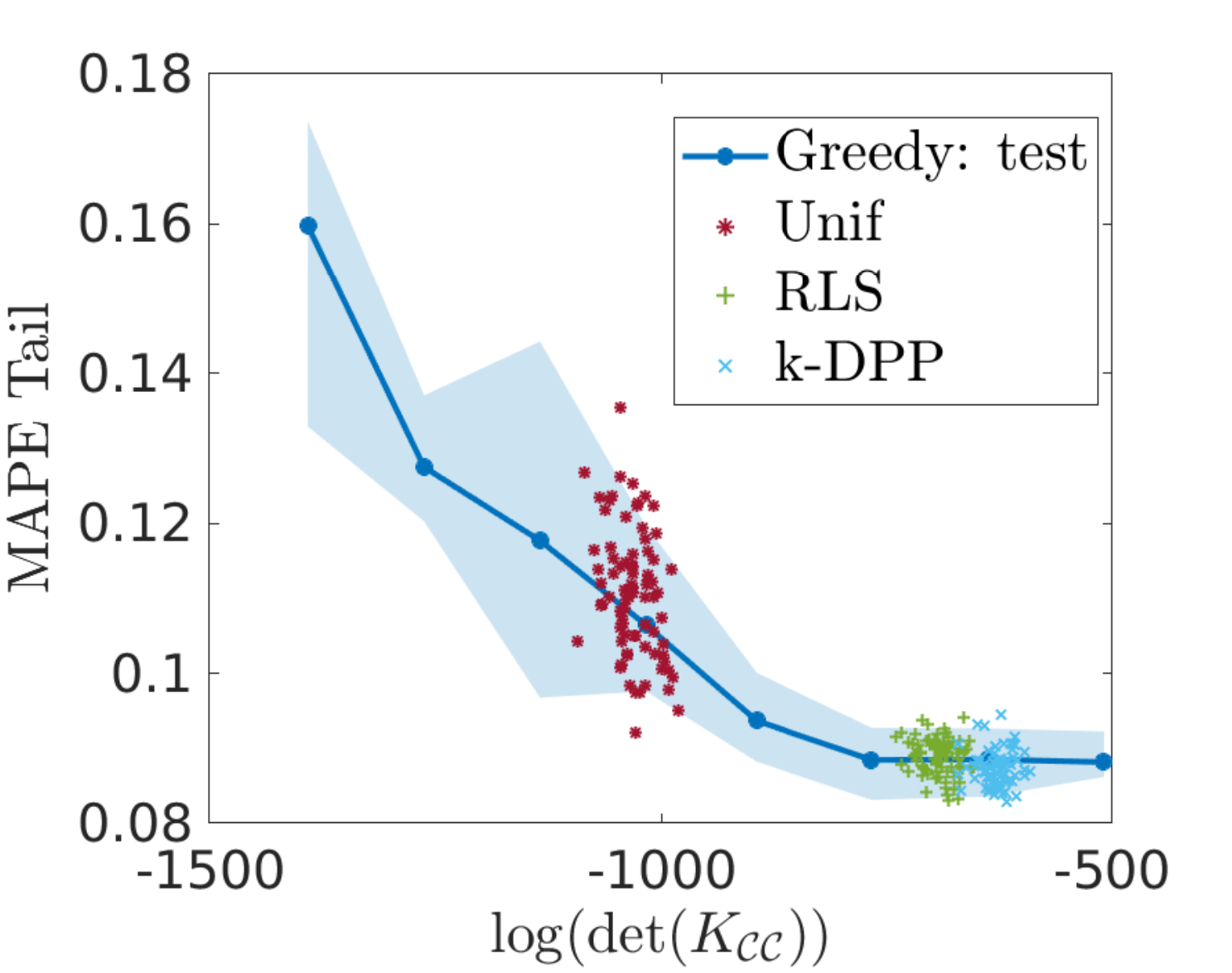}
			\caption{\texttt{Park.}: MAPE Tail}
		\end{subfigure}
		\caption{Regression results. The condition number and MAPE on the test set are plotted as a function of $\mathrm{log}(\mathrm{det}(K_{\mathcal{C}\mathcal{C}}))$.}\label{fig:RegressionSupp}
	\end{figure}

 \begin{figure}[t]
 	\centering
 	\begin{subfigure}[t]{0.24\textwidth}
 		\includegraphics[width=\textwidth, height= 0.8\textwidth]{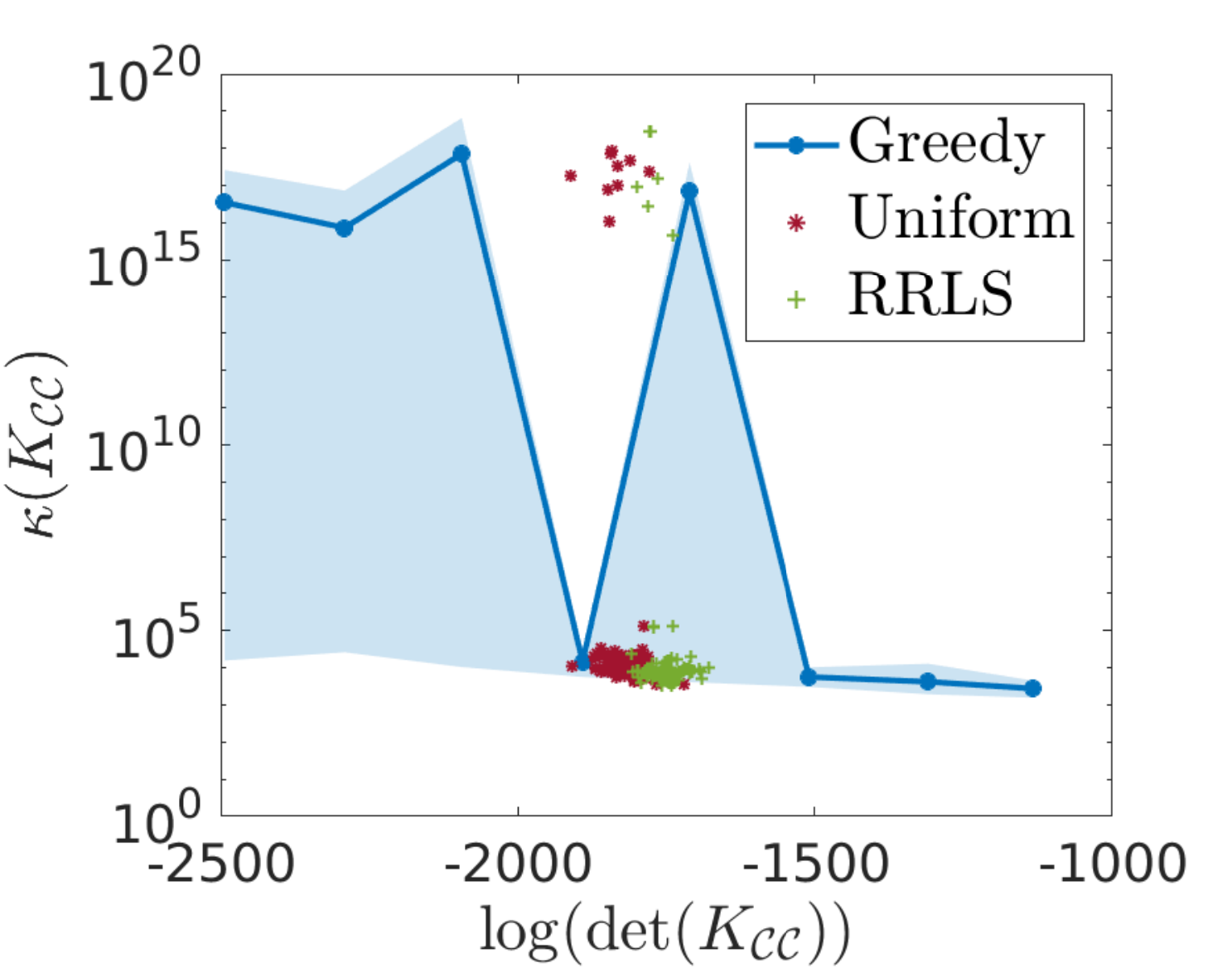}
 		\caption{\texttt{Bike S.}: $\kappa(K_{\mathcal{C}\mathcal{C}})$}
 	\end{subfigure}
  	\begin{subfigure}[t]{0.24\textwidth}
  		\includegraphics[width=\textwidth, height= 0.8\textwidth]{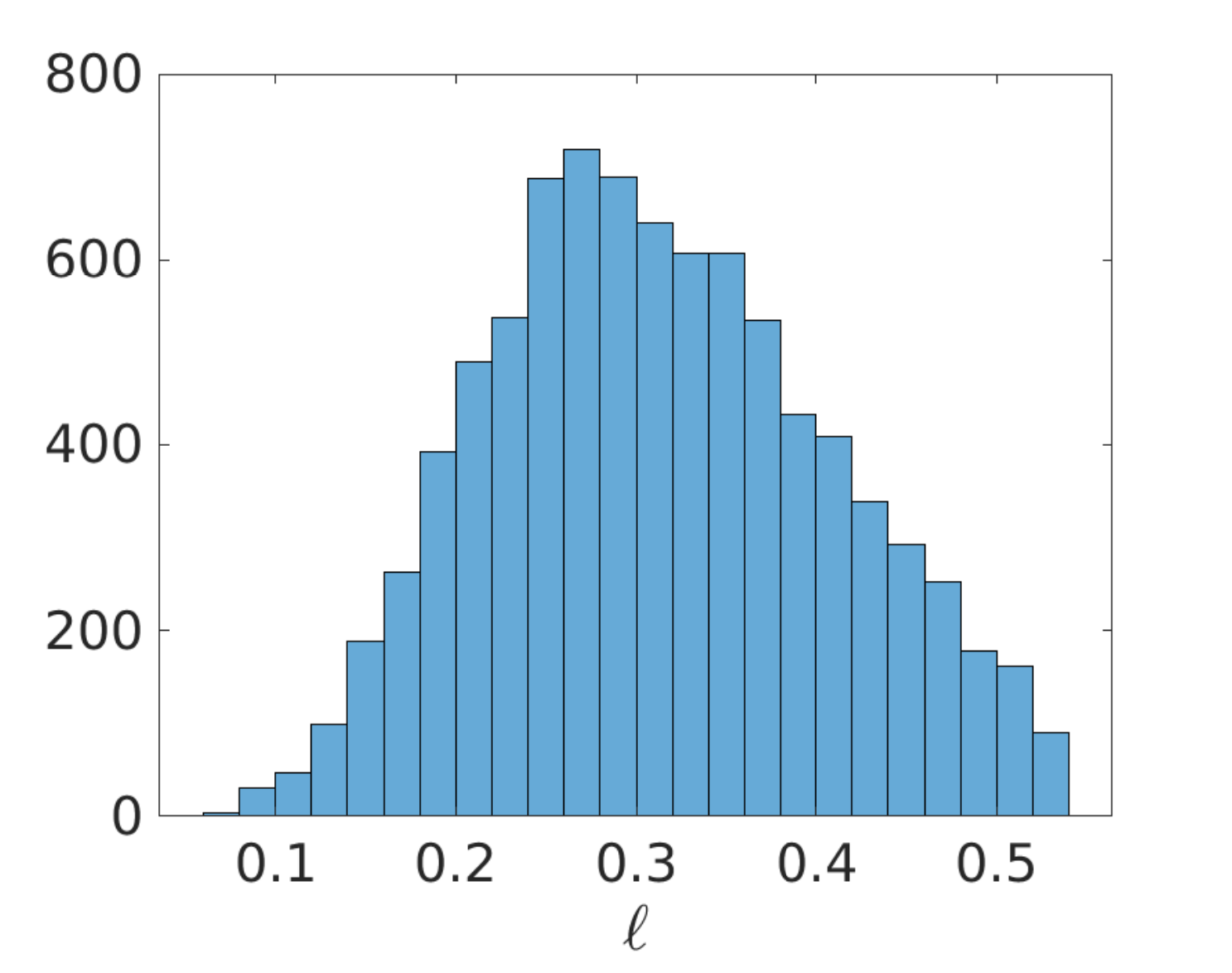}
  		\caption{\texttt{Bike S.}: RLS}
  	\end{subfigure}
 	\begin{subfigure}[t]{0.24\textwidth}
 		\includegraphics[width=\textwidth, height= 0.8\textwidth]{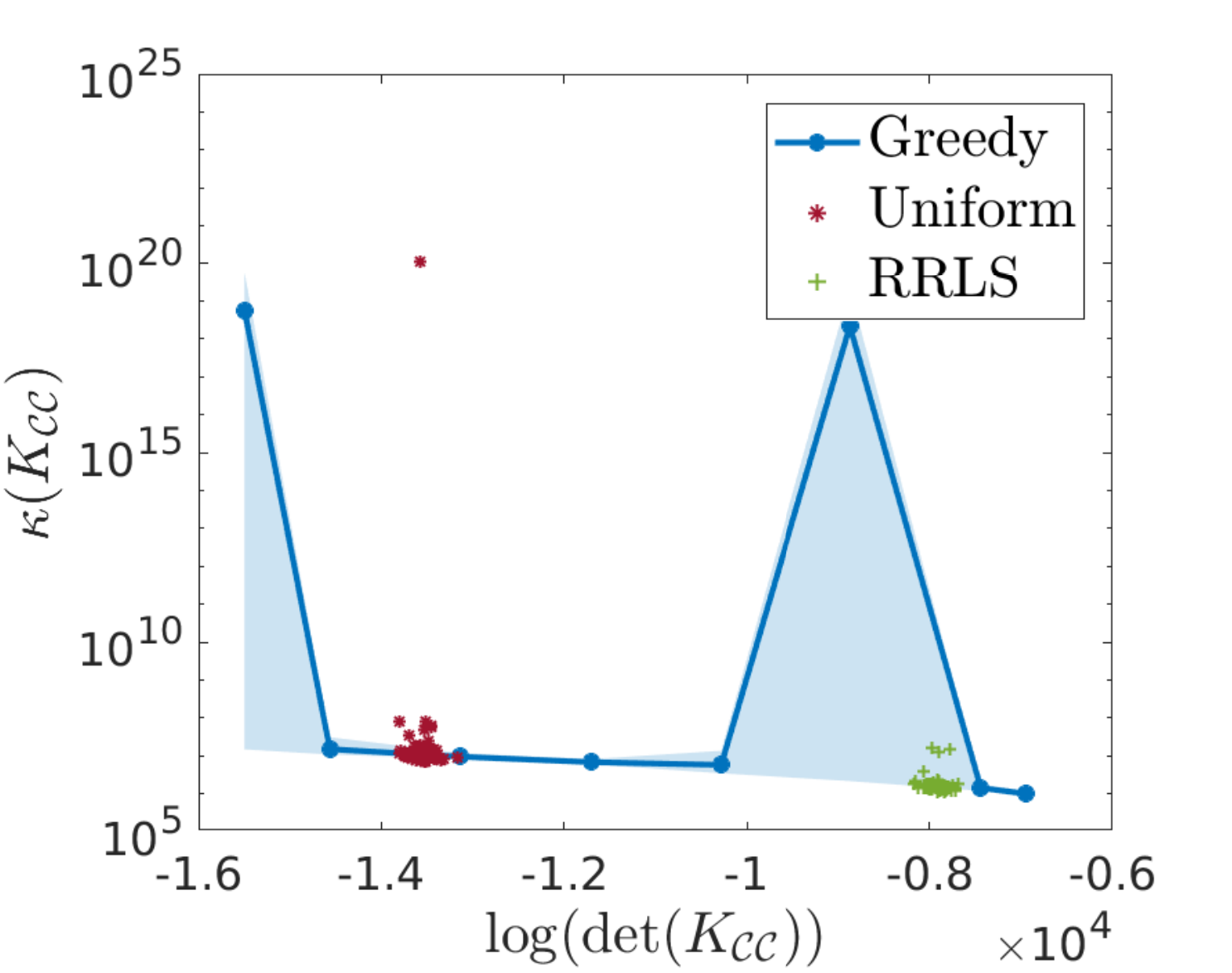}
 		\caption{\texttt{Year P.}: $\kappa(K_{\mathcal{C}\mathcal{C}})$}
 	\end{subfigure}
 	\begin{subfigure}[t]{0.24\textwidth}
 		\includegraphics[width=\textwidth, height= 0.8\textwidth]{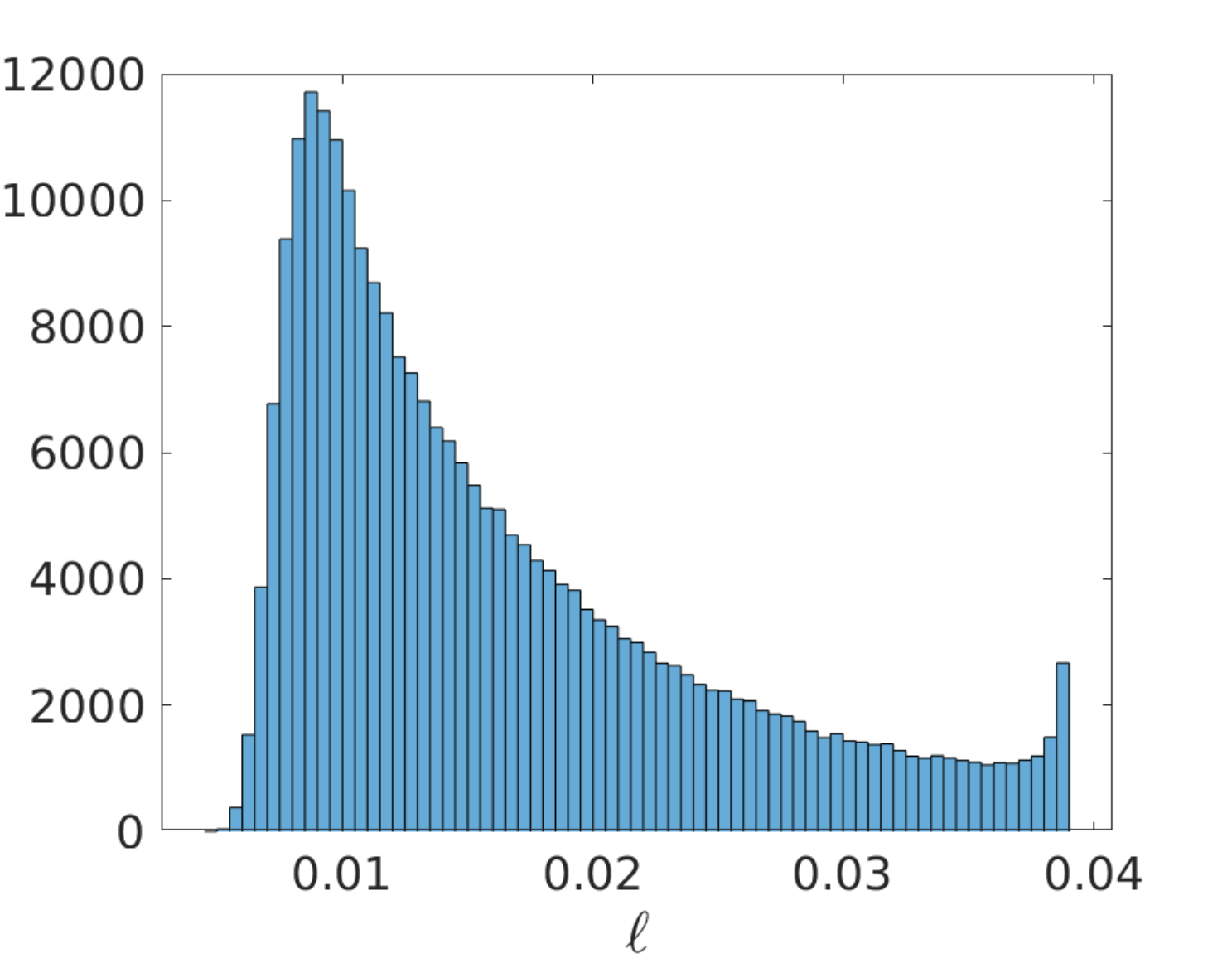}
 		\caption{\texttt{Year P.}: RLS}
 	\end{subfigure}  	
 	\caption{Results accompanying Figure \ref{fig:RegressionLS}. The condition number as a function of $\mathrm{log}(\mathrm{det}(K_{\mathcal{C}\mathcal{C}}))$ and the approximate RLS distribution are visualized.}\label{fig:RegressionLSSupp}
 \end{figure}

 	\begin{figure}[h]
		\centering
		\begin{subfigure}[b]{0.35\textwidth}
			\includegraphics[width=\textwidth, height= 0.95\textwidth]{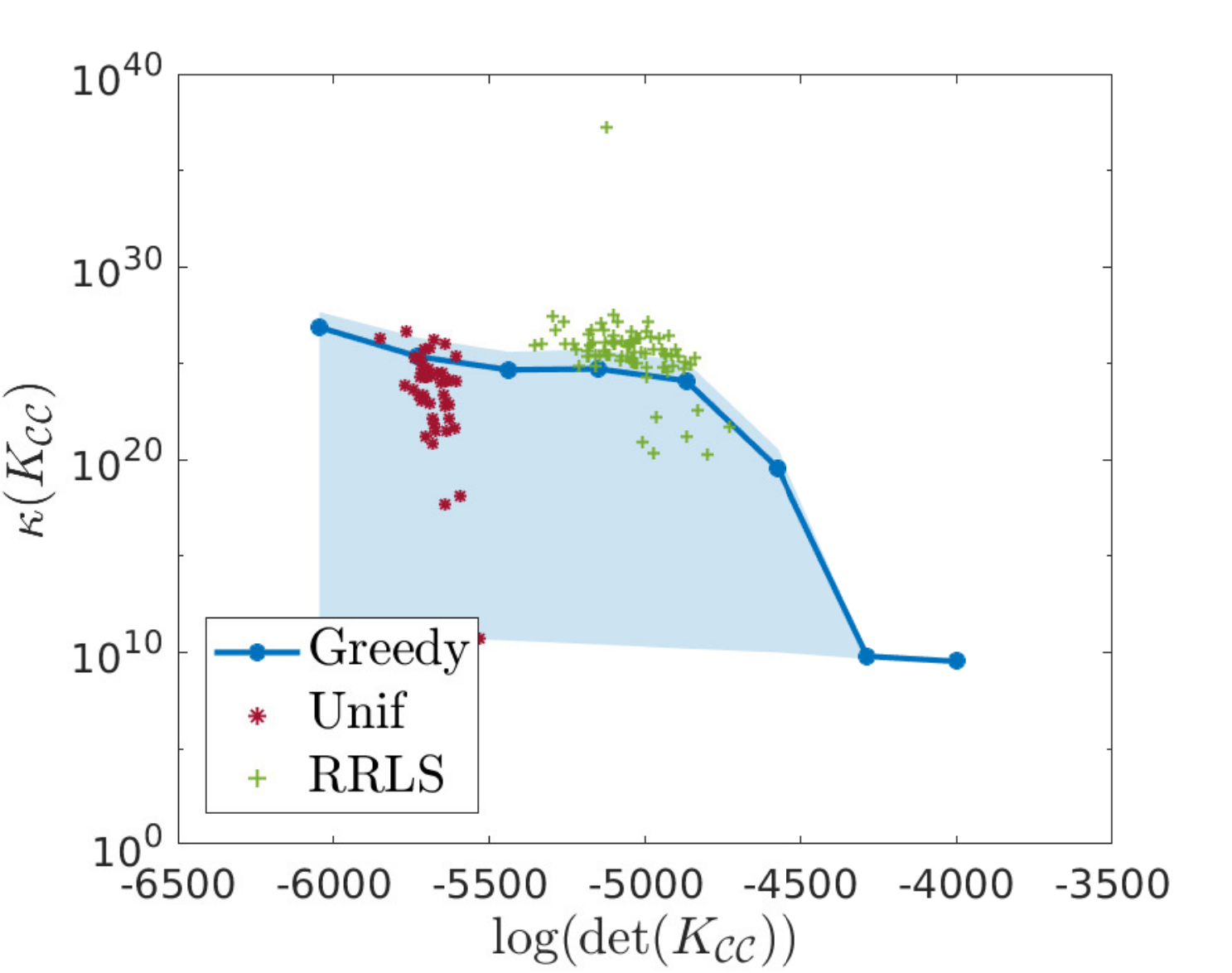}
			\caption{\texttt{MiniBooNE}:$\kappa(K_{\mathcal{C}\mathcal{C}})$}
		\end{subfigure}
		\begin{subfigure}[b]{0.35\textwidth}
			\includegraphics[width=\textwidth, height= 0.95\textwidth]{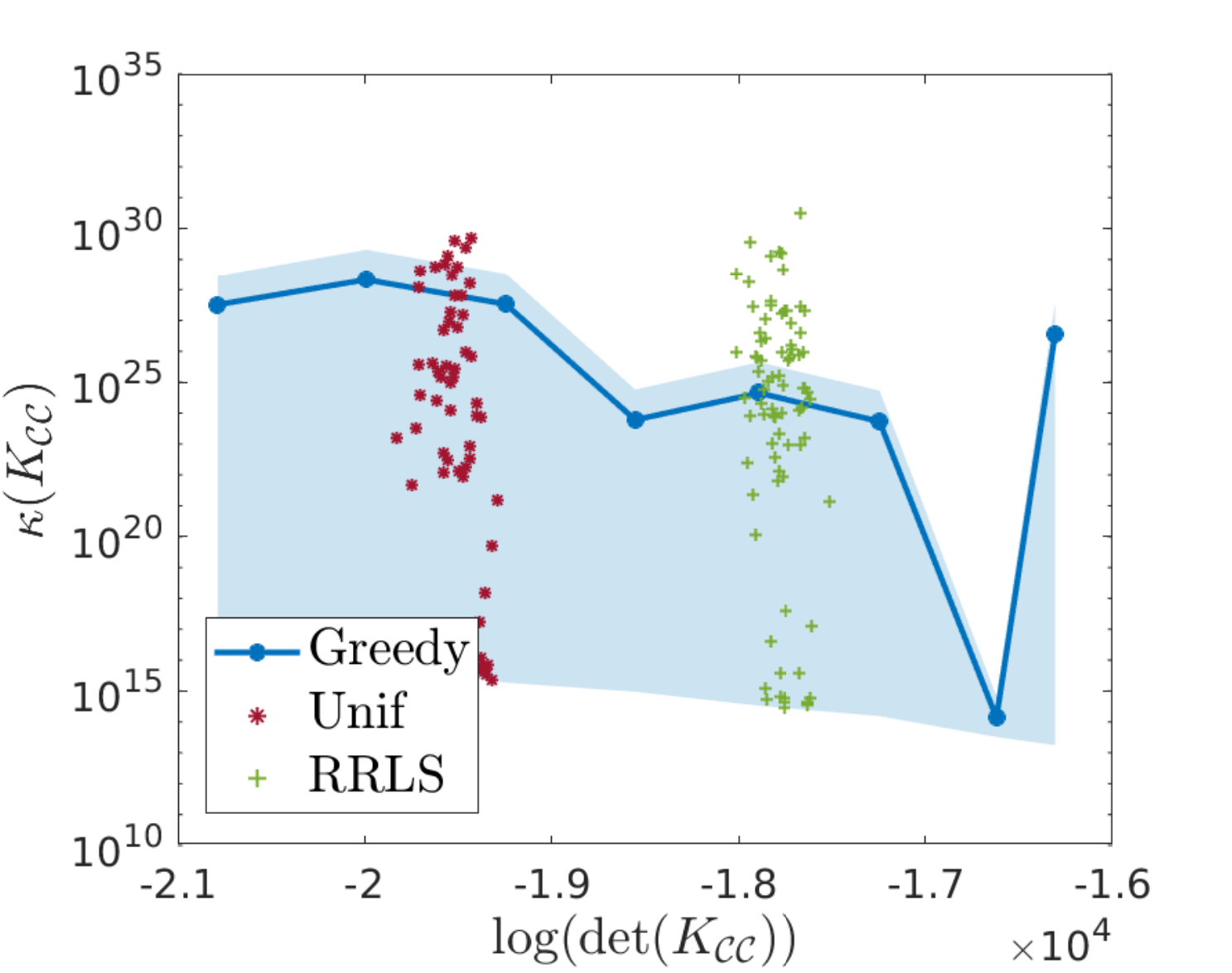}
			\caption{\texttt{codRNA}: $\kappa(K_{\mathcal{C}\mathcal{C}})$}
		\end{subfigure}		
		\caption{Results accompanying Figure \ref{fig:NystromLS}. The condition number of $K_{\mathcal{C}\mathcal{C}}$ versus the sample diversity.}\label{fig:NystromLSSupp}
	\end{figure}


\FloatBarrier
\bibliographystyle{siamplain}
\bibliography{References}
\end{document}